\documentclass[10pt, journal]{IEEEtran}

\ifCLASSOPTIONcompsoc
  \usepackage[nocompress]{cite}
\else
  \usepackage{cite}
\fi

\usepackage{array}
\usepackage{amsfonts}
\usepackage{amsthm}
\usepackage{times}
\usepackage{epsfig}
\usepackage{graphicx}
\usepackage{amsmath}
\usepackage{amssymb}
\usepackage{booktabs}
\usepackage{multirow}
\usepackage{makecell}
\usepackage{bbding}
\usepackage{upgreek}
\usepackage{bm}
\usepackage{enumitem}
\usepackage{paralist}
\usepackage{float}
\usepackage{subcaption}
\usepackage{threeparttable}
\usepackage{fontawesome}
\usepackage{bbm}
\usepackage{booktabs}
\usepackage[citecolor=blue, colorlinks]{hyperref}
\usepackage{soul}
\usepackage{tabularx}
\usepackage{tikz}
\usepackage{wrapfig}

\usepackage[ruled,linesnumbered]{algorithm2e}
\newtheorem{theorem}{Theorem}

\newcommand{\revise}[1]{\textcolor{black}{#1}}

\begin{document}

\title{
DDM:  A Metric for Comparing 3D Shapes Using Directional Distance Fields
}

\author{Siyu Ren, Junhui Hou,~\IEEEmembership{Senior Member,~IEEE,} Xiaodong Chen, Hongkai Xiong,~\IEEEmembership{Fellow,~IEEE,} \\and Wenping Wang,~\IEEEmembership{Fellow,~IEEE}                       
\IEEEcompsocitemizethanks{
\IEEEcompsocthanksitem This work was supported in part by the NSFC Excellent Young Scientists Fund 62422118, and in part by the Hong Kong
Research Grants Council under Grants 11219324 and 11219422. (\textit{Corresponding author: Junhui Hou})
\IEEEcompsocthanksitem S. Ren and J. Hou are with the Department
of Computer Science, City University of Hong Kong, Hong Kong SAR. Email: siyuren2-c@my.cityu.edu.hk, jh.hou@cityu.edu.hk\protect 
\IEEEcompsocthanksitem X. Chen is with the School of Precision Instruments and Opto-Electronic Engineering, Tianjin University, Tianjin 300072, China (e-mail: xdchen@tju.edu.cn)
\IEEEcompsocthanksitem Hongkai Xiong is with the Department of Electronic Engineering, Shanghai Jiao Tong University, Shanghai 200240, China (e-mail: xionghongkai@sjtu.edu.cn)
\IEEEcompsocthanksitem W. Wang is with the Department of Computer Science \& Engineering, Texas A\& M University, USA. Email: wenping@tamu.edu.

}

}


\maketitle 
\begin{abstract}
Qualifying the discrepancy between 3D geometric models, which could be represented with either point clouds or triangle meshes, is a pivotal issue with board applications. Existing methods mainly focus on directly establishing the correspondence between two models and then aggregating point-wise distance between corresponding points, resulting in them being either inefficient or ineffective. In this paper, we propose DDM, an efficient, effective, robust, and differentiable distance metric for 3D geometry data. Specifically, we construct DDM based on the proposed implicit representation of 3D models, namely directional distance field (DDF), which defines the directional distances of 3D points to a model to capture its local surface geometry.  
We then transfer the discrepancy between two 3D geometric models as the discrepancy between their DDFs defined on an identical domain, 
naturally establishing model correspondence. To demonstrate the advantage of our DDM, we explore various distance metric-driven 3D geometric modeling tasks, including template surface fitting, rigid registration, non-rigid registration, scene flow estimation and human pose optimization. Extensive experiments show that our DDM achieves significantly higher accuracy under all tasks. As a generic distance metric, DDM has the potential to advance the field of 3D geometric modeling. The source code is available at  \url{https://github.com/rsy6318/DDM}.
\end{abstract}

\begin{IEEEkeywords}
3D point clouds, 3D mesh, distance metric, geometric modeling,  shape registration, scene flow estimation
\end{IEEEkeywords}

\section{Introduction}\label{sec:introduction}
\IEEEPARstart{T}{hree}-dimensional (3D) geometric models, which could be represented with either 3D point clouds or triangle meshes, have found extensive applications in various fields, including computer vision/graphics and robotics. Quantifying the discrepancy between 3D geometry data is critical in these applications. For instance, in tasks such as self-supervised surface registration \cite{ICP,ICPPPP,ARL,AMM}, reconstruction \cite{FOLDINGNET, AE1, AE2, AE3}, generation \cite{TEXT2MESH, CLIPMESH}, and scene flow estimation \cite{POINTPWC,NSFP,SCOOP}, a typical distance metric needs to be employed to drive the optimization/learning process.  
Unlike 2D images, where the discrepancy between them (e.g., one being the ground-truth/reference and the other one reconstructed through a typical task) can be easily computed pixel-by-pixel, owing to the canonical pixel coordinate system, quantifying the discrepancy between two 3D geometric models is non-trivial. This is primarily because of their unstructured nature, where data are irregularly distributed, and the correspondence information is unknown. 

\begin{figure*}[h]
\centering
\subfloat[EMD]{\includegraphics[width=0.17\textwidth]{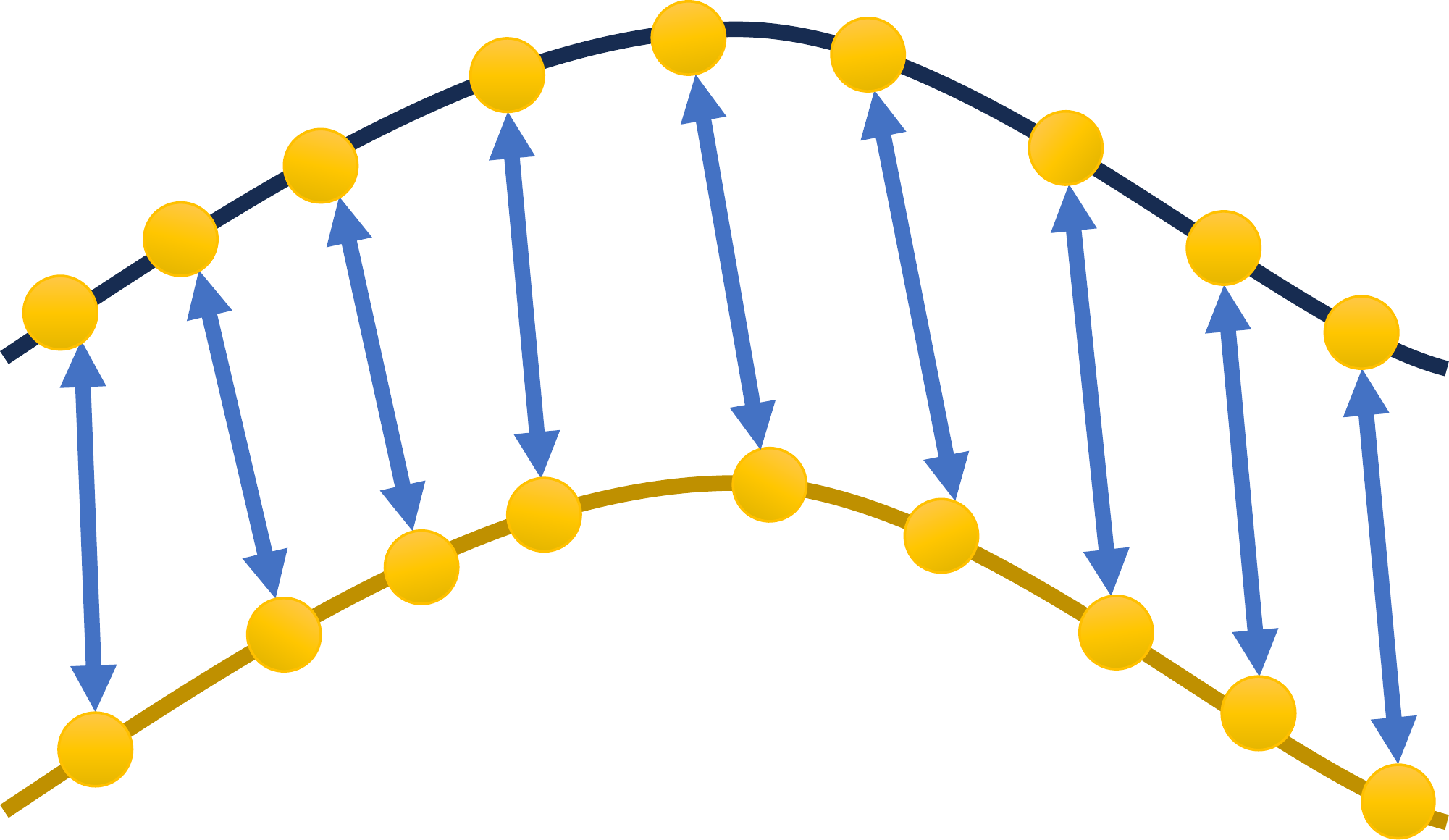}\label{P2PFIG}} \quad\
\subfloat[CD]{\includegraphics[width=0.17\textwidth]{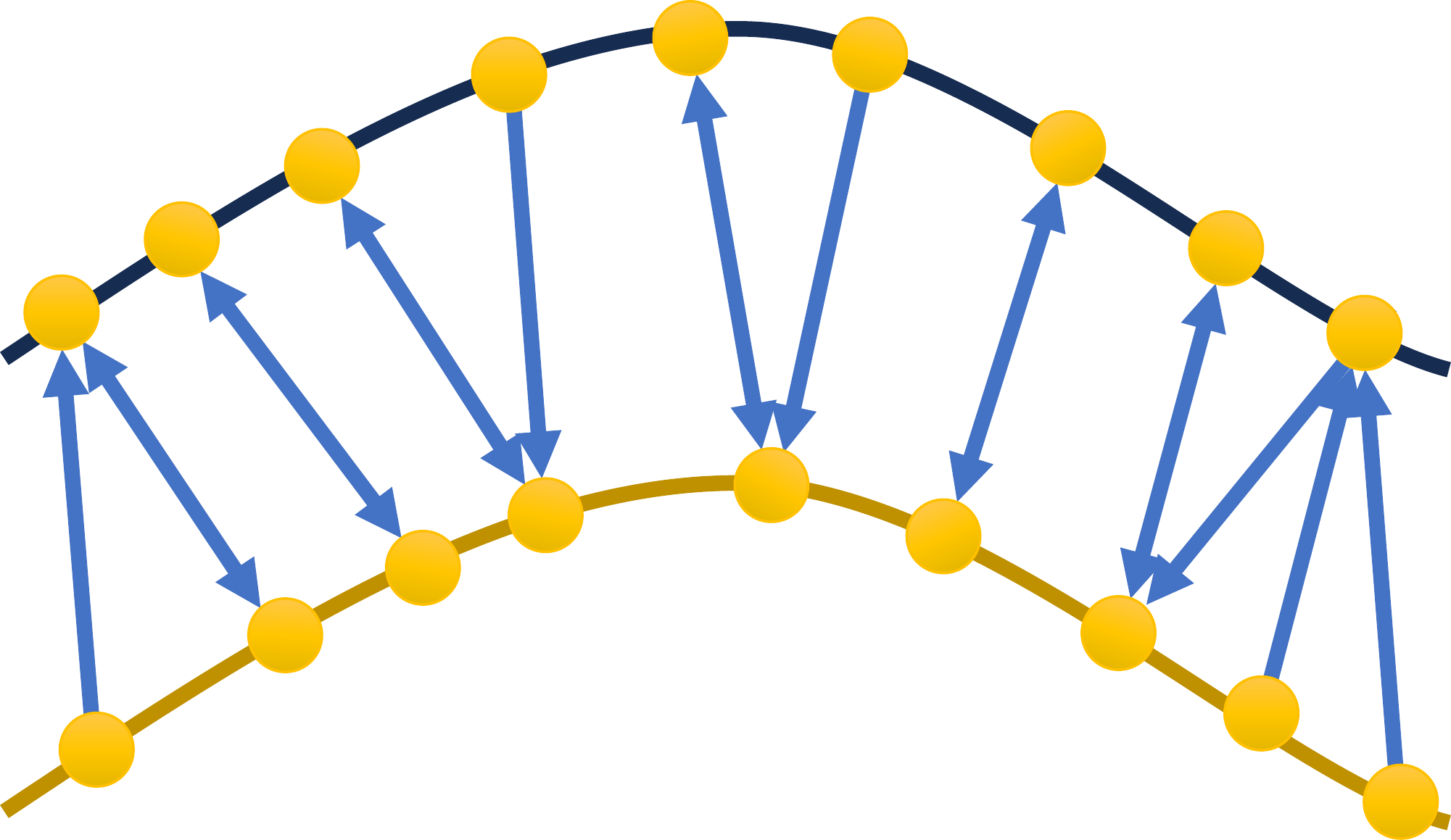}\label{P2P:CDFIG}} \quad\
\subfloat[P2F]{\includegraphics[width=0.17\textwidth]{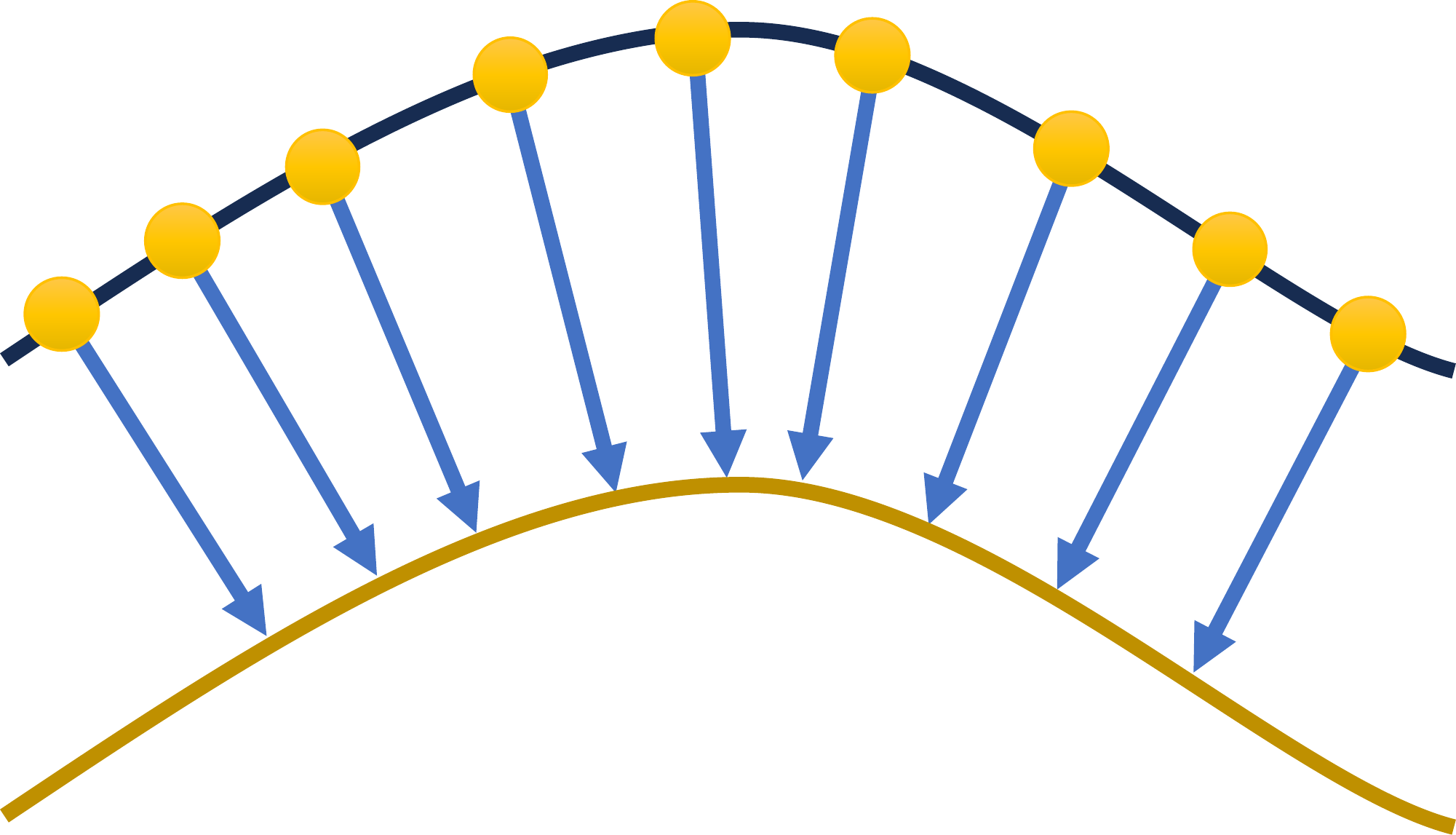}\label{P2FFIG}} \quad\
\subfloat[\revise{ARL}]{\includegraphics[width=0.17\textwidth]{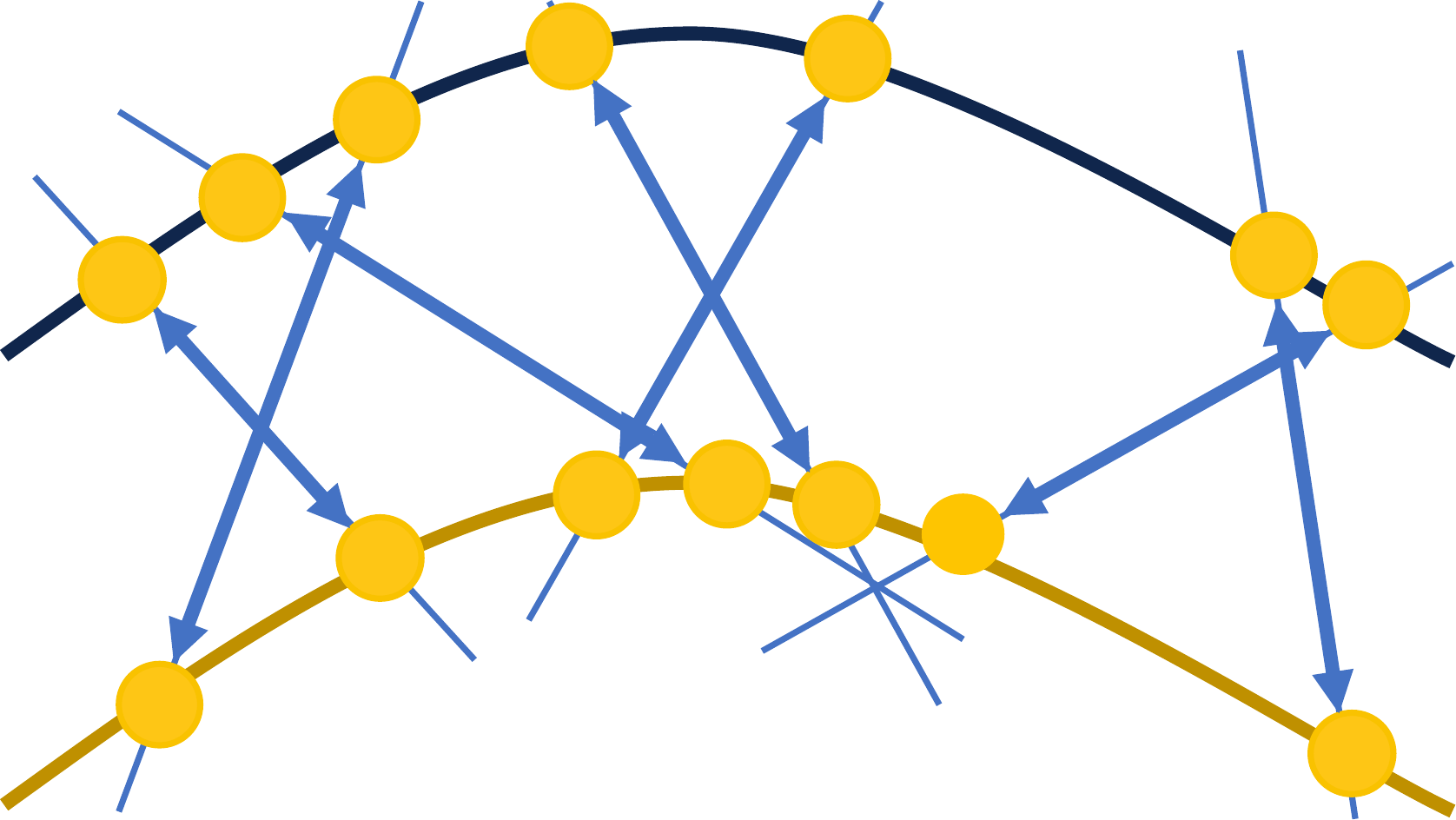}\label{ARLFIG}}\quad
\subfloat[Ours]{\includegraphics[width=0.17\textwidth]{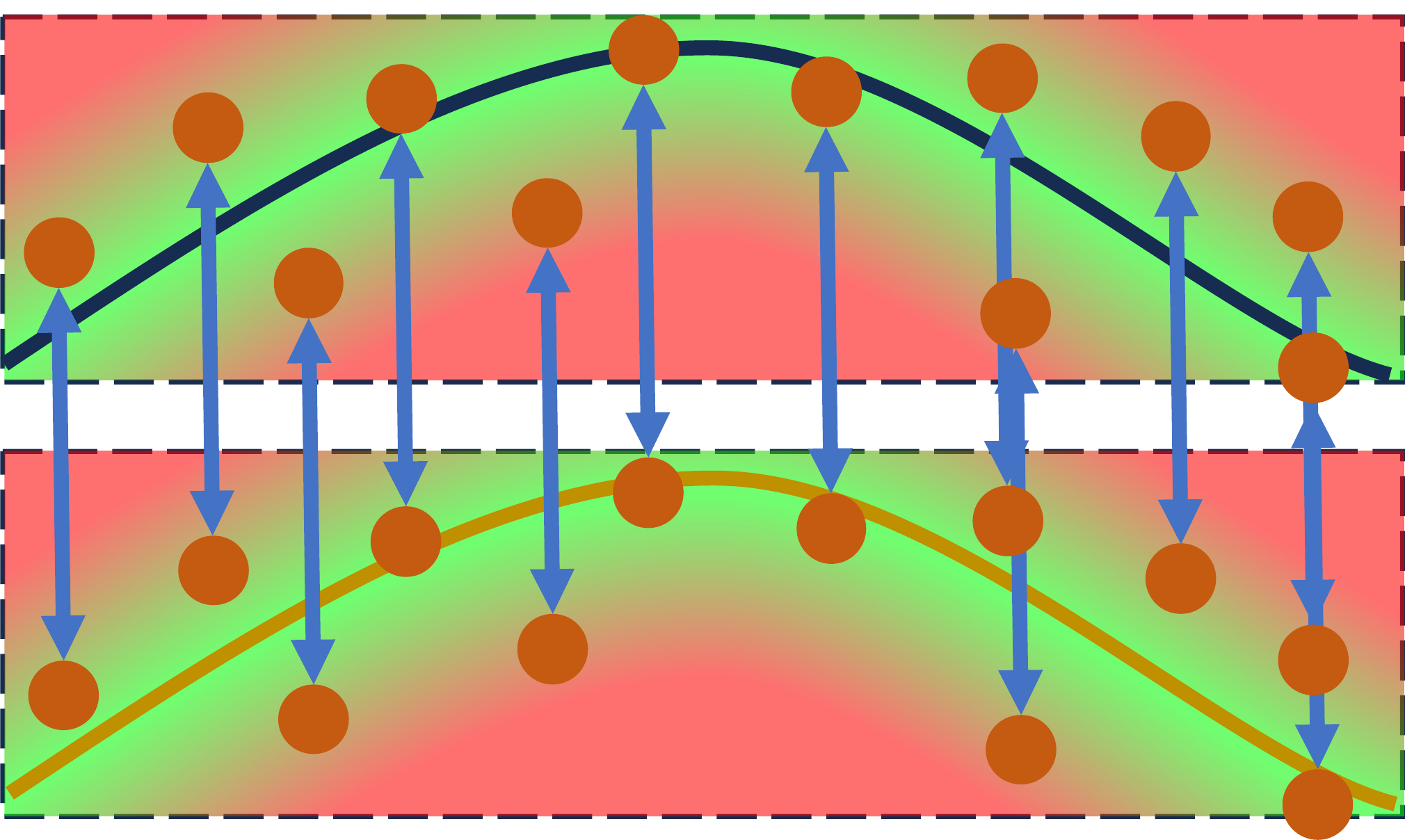}\label{OURSFIG}}
\caption{\revise{ Visual illustration of different distance metrics for 3D geometry data. For convenience, we use 2D illustration. The \textcolor[rgb]{1,0.752,0}{yellow} points in (a), (b), (c), and (d) refer to 3D points located on 3D surfaces indicated by curves, and the \textcolor[rgb]{0.772,0.352,0.066}{brown} points in (e) represent the generated reference points. The \textcolor[rgb]{0.266,0.447,0.768}{blue} arrows represent the established correspondence. The color in (e) changes from \textcolor{green}{green} to \textcolor[rgb]{1,0.478,0.478}{red} indicating the distance fields of the two surfaces indicated by curves (i.e., the set of the distances of arbitrary points in 3D space to the surfaces).  }
}
\end{figure*}

Existing distance metrics for 3D geometry data predominantly fall into two main categories, point-to-point (P2P) and point-to-face (P2F) distances.  In the P2P category, two representative and commonly employed metrics are Earth Mover's Distance (EMD) \cite{EMD} and Chamfer Distance (CD) \cite{CD}, as illustrated in Figs. \ref{P2PFIG} and \ref{P2P:CDFIG}, respectively. These metrics generally establish point-to-point correspondences between two point clouds\footnote{Directly measuring the discrepancy between two continuous 3D surfaces presents challenges. When dealing with continuous 3D surfaces, such as triangle meshes, a typical sampling method is often applied to sample a set of points on the surface for computation.}
and subsequently compute the aggregation of the point-wise distances between the corresponding points. However,  these metrics suffer from limitations as they operate on sampled points, disregarding the continuity of the surfaces, which renders them ineffective. Additionally, the correspondence establishment process can be time-consuming, such as the bijection adopted in EMD.  Alternatively, the P2F method \cite{METRO}, as illustrated in Fig. \ref{P2FFIG}, mitigates the shortcomings of P2F to some extent, which involves sampling points from one surface and then computing each sampled point's closest distance to the other surface to measure the discrepancy between two surfaces. However, the closest point search could make it prone to suboptimal solutions. \revise{In contrast, the ARL method \cite{ARL} employs randomly sampled lines to intersect with the two surfaces, utilizing the resulting intersection points to establish correspondences, as illustrated in Fig. \ref{ARLFIG}. However, the positions of these intersection points are inherently dependent on the spatial arrangement of the sampled lines. Consequently, the use of randomly sampled lines may introduce inaccuracies in the correspondence between the geometric data, potentially leading to erroneous alignments.}
Although several improved methods \cite{BCD, SWD, DPDIST, P2F_improved, P2F_improved2} have been proposed, they still suffer from inefficiency or ineffectiveness. See the detailed review in Sec. \ref{sec:RW}.

To address this fundamental and challenging issue of quantifying the discrepancy between 3D geometric models, we propose a novel robust, efficient, and effective distance metric called DDM.Distinguished from existing metrics, DDM emphasizes the implicit representation of 3D geometric models, as depicted in Fig. \ref{OURSFIG}. Methodologically, for any pair of 3D models, we begin by constructing their individual implicit fields, referred to as directional distance fields (DDFs). The discrepancy between these fields serves as an indicator of the distance between the corresponding 3D geometric models. To compute DDM, we generate a set of reference points that are distributed near the models and shared by both models. For each reference point, we calculate its directional distances to the two models, capturing the local surface geometry. By taking the weighted average of the directional distance discrepancies contributed by all reference points, we obtain the final DDM. The use of DDFs enables DDM to robustly handle optimization convergence issues by effectively capturing surface geometry. Moreover, DDM does not require a direct correspondence establishment process, resulting in high efficiency. 

Notably, DDM is differentiable, allowing seamless integration as a module in various tasks. We leverage DDM to develop methods for diverse 3D geometric modeling and processing tasks, including template surface fitting, rigid and non-rigid registration, scene flow estimation, and human pose optimization. Extensive experiments validate the superiority of DDM in terms of robustness, efficiency, and effectiveness

In summary, the main contributions of this paper are:
\begin{compactitem}
\item an efficient, effective, robust, and generic distance metric for 3D geometry data, dubbed DDM;
\item various state-of-the-art geometric modeling and processing methods driven by our DDM;
\end{compactitem}

The remainder of this paper is structured as follows. Section \ref{sec:RW} provides a brief overview of existing literature on 3D geometry representations, distance metrics for 3D geometry data, and distance metric-driven 3D geometry processing. Section \ref{sec:proposed method} introduces our proposed DDM in detail. Section \ref{sec:tasks} outlines the specifics of the five general distance metric-driven 3D geometric tasks, followed by comprehensive experiments on these tasks in Section \ref{sec:exp}. Finally, Section \ref{CONCLUSION} concludes this work.

\section{Related Work}
\label{sec:RW}
\subsection{Explicit and Implicit 3D Geometry Representations}
\noindent In the domain of 3D modeling and computer graphics, surface representations are predominantly categorized as either explicit or implicit, and these can be transformed into one another. 

\subsubsection{Explicit Representation} Voxelization \cite{IFNET} is the most intuitive representation method for surfaces, which utilizes the regularly distributed grids to represent the surface, converting the surfaces as 3D `images'. However, it requires large memory consumption, limiting its application. Point clouds, consisting of discrete points sampled from surfaces, have emerged as a predominant method for surface representation \cite{ICP,POINTNET,POINTNET2}. While they are widely adopted, point clouds have their constraints, particularly in downstream applications like rendering. Images rendered from point clouds can exhibit gaps or holes due to the inherent sparsity of the points. Triangle meshes are a more precise and efficient form of surface representation, using numerous triangles to approximate surfaces. Triangle meshes store the topologies of each triangular face, making them more accurate and efficient for some downstream tasks compared to point clouds. However, the inherent topologies make triangle meshes more challenging to process compared to point clouds, especially when inputting them into neural networks \cite{DIFFUSIONNET}. 

\subsubsection{Implicit Representation} Implicit representation leverages the isosurface of a function or field for surface depiction. Binary Occupancy Field (BOF) and Signed Distance Field (SDF) are two widely used implicit representations and used in many reconstruction methods \cite{PSR, SPSR, OCCNET, CONVOCCNET, SAP, POCO, SDFORIGINAL, IMLS, IMLS2, DEEPIMLS, DEEPSDF}. However, these representations require that the surfaces are watertight shapes and cannot represent more general shapes because they devide the whole space into inside and outside regions, limiting their applications. On the other hand, Unsigned Distance Field (UDF) is derived from the absolute value of SDF and does not categorize the space into inside and outside regions. This characteristic allows UDF to represent more general shapes than BOF and SDF. However, this also increase the difficulty when extracting the surfaces from UDF. Thus UDF is often used along with its gradient \cite{NDF, MESHUDF} or other attributes \cite{GIFS} to represent the surfaces.

\subsubsection{Conversion between Surface Representations 
}
Surface representations can be converted between explicit and implicit forms. For surfaces depicted as point clouds, there exists a suite of traditional methodologies \cite{PSR, SPSR, IMLS, IMLS2} in addition to data-driven approaches \cite{OCCNET, CONVOCCNET, DEEPIMLS, DEEPSDF, DOG, NDF, GIFS, GEOUDF} for their transformation into implicit fields. When dealing with surfaces characterized by triangle meshes, the process becomes relatively straightforward due to the inherent approximation of surfaces by the discrete triangle faces, and techniques such as \cite{MESH2SDF, MESH2SDF2} facilitate this conversion. Conversely, when converting from implicit to explicit forms, algorithms like Marching Cubes \cite{MARCHINGCUBE} and its derivatives \cite{MESHUDF, GEOUDF} can reconstruct triangle meshes from various implicit fields, encompassing BOF, SDF, and UDFs.\\

\subsection{Distance Metrics for 3D Geometry Data}
Based on the established correspondence category, the prevailing distance metrics for surfaces can be categorized into two primary types, point-to-point (P2P) and point-to-face (P2F) distances.

\subsubsection{P2P Distance} After sampling points on the surfaces to convert them into point clouds, the distance between surfaces can be represented by that between point sets. 
Earth Mover's Distance (EMD) \cite{EMD} and Chamfer Distance (CD) \cite{CD} are two most widely used distance metrics. Specifically, EMD establishes a comprehensive bi-directional mapping between two point sets, subsequently leveraging the summation or average of distances between associated points to determine their overall distance. However, this bijection computation proves to be computationally intensive, especially with a surge in point count. In contrast, CD, by determining the nearest point in the alternative set, achieves a more efficient local mapping, though at times it's susceptible to local minima or non-ideal results. The Hausdorff Distance (HD) \cite{HD}, an adaptation of CD, emphasizes outliers, which often compromises its ability to capture finer point cloud details, relegating its usage more to evaluation rather than primary computation. Balanced Chamfer Distance (BCD) \cite{BCD}, with its innovative approach, infuses density information as weights into CD, yielding a model with enhanced resilience to outliers. Meanwhile, the Sliced Wasserstein Distance (SWD) \cite{SWD} harnesses the prowess of sliced Wasserstein distances and its derivatives, showcasing efficiency and efficacy superior to both EMD and CD, especially in shape representation endeavors.

\subsubsection{P2F Distance} 
An alternative distance metric between surfaces employs the P2F distance \cite{P2FORIGINAL}. In this approach, points are sampled from one of the two given surfaces. Subsequently, the aggregate or mean distance between these sampled points and the opposing surface is computed, serving as a representation of the distance between the two surfaces.  
Numerous registration methods, as described in \cite{POINT2PLANE,SICP}, rely on measuring the distance to tangent planes of sampled points to approximate 
the P2F distance. Nevertheless, these approaches exhibit inherent bias, particularly in curved regions.
 Pottmann \textit{et al.} \cite{P2F_improved, P2F_improved2} utilized curvature information of the surfaces to approximate the P2F distance to accelerate the optimization.  Mesh Hausdorff Distance (MHD) \cite{MESHHD} is a variant of the P2F distance, placing greater emphasis on surface outliers while often overlooking intricate surface details. This characteristic renders it particularly suitable as an evaluation metric. 
DPDist \cite{DPDIST} employs a neural network to estimate the point-to-face distance between two surfaces. Nonetheless, the accuracy of the trained network may diminish if there's a shift in the data distribution, posing challenges to its generalizability. Recently, ARL \cite{ARL} introduces randomly sampled lines and then calculates the intersections of these lines on the surfaces, where the disparity between these intersections is used to measure the distance between the two surfaces. However, the randomness of the sampled lines may cause instability for its measurement.

\subsection{Distance Metric-driven 3D Geometry Processing}

Evaluating the similarity between surfaces is fundamental for multiple tasks in surface analysis and processing. Notably, this is pivotal in both unsupervised rigid and non-rigid surface registrations. Conventional registration techniques, such as \cite{ICP,AMM}, harness the distance between surfaces as an objective function to adjust the pose. Moreover, some unsupervised learning-based strategies \cite{ARL} \cite{RMA} incorporate surface distance as an alignment factor within their training loss functions. In scene flow estimation task, recent unsupervised methodologies, including PointPWC-Net \cite{POINTPWC}, NSFP \cite{NSFP}, and SCOOP \cite{SCOOP}, apply distance between surfaces in their loss functions to align the two point clouds after deformation, enabling training without explicit supervision. Recent template-based surface fitting works \cite{POINT2MESH, LARGESTEP, MDA} utilize the distance between 3D shapes to supervise the deformation process of the predefined template surfaces.
In the task of unsupervised human pose estimation from 3D data, LoopReg \cite{SMPL_REG}, utilizes the difference between the Skinned Multi-Person Linear Models (SMPL) \cite{SMPL} and the scanned data as loss function to optimize the human pose.

\section{Proposed Method}
\label{sec:proposed method}

\subsection{Rethinking the Distance Metric for RGB Images 
} \label{SEC:IMG:DIST}
Essentially, an RGB image 
$\mathbf{I}$ could be encoded through a color mapping $\mathcal{F}:\mathbb{R}^2\to\mathbb{R}^{3}$, where the input is the 2D coordinate of 
a pixel, denoted as $(u,~v)$, and the output is the corresponding pixel value, denoted as $[r,~g,~b]$. 
Alternatively, an RGB can be re-parameterized as $\mathbf{I}:=\{\mathcal{F}(u,v)|(u,v)\in\mathbf{U}\}$, where $\mathbf{U}$ is the image area.

Denote by $\mathbf{I}_1$ and $\mathbf{I}_2$ two RGB images with color mapping $\mathcal{F}_{\mathbf{I}_1}$ and $\mathcal{F}_{\mathbf{I}_2}$, respectively. Assume that $\mathbf{I}_2$ is the counterpart of $\mathbf{I}_1$, which could be generated or reconstructed from a typical task, e.g., image synthesis, restoration, compression, etc.  
As shown in Fig. \ref{IMG:CORR}, we can easily compute the discrepancy 
between $\mathbf{I}_1$ and $\mathbf{I}_2$ as 
\begin{equation}
    \mathcal{D}(\mathbf{I}_1,\mathbf{I}_2)=\iint_{\mathbf{U}}d_{\rm img}(\mathcal{F}_{\mathbf{I}_1}(u,v),\mathcal{F}_{\mathbf{I}_2}(u,v)){\rm d}u{\rm d}v, \label{IMG:DIS}
\end{equation}
where $d_{\rm img}(\cdot,~\cdot)$ returns the distance between two vectors (e.g., $l_1$, $l_2$, cosine similarity, among others). In practice, Eq. (\ref{IMG:DIS}) is usually calculated discretely by sampling regularly distributed pixels in $\mathbf{U}$.

\begin{figure}[t]
    \centering
    \subfloat[RGB Image]{
    \includegraphics[width=0.9\linewidth]{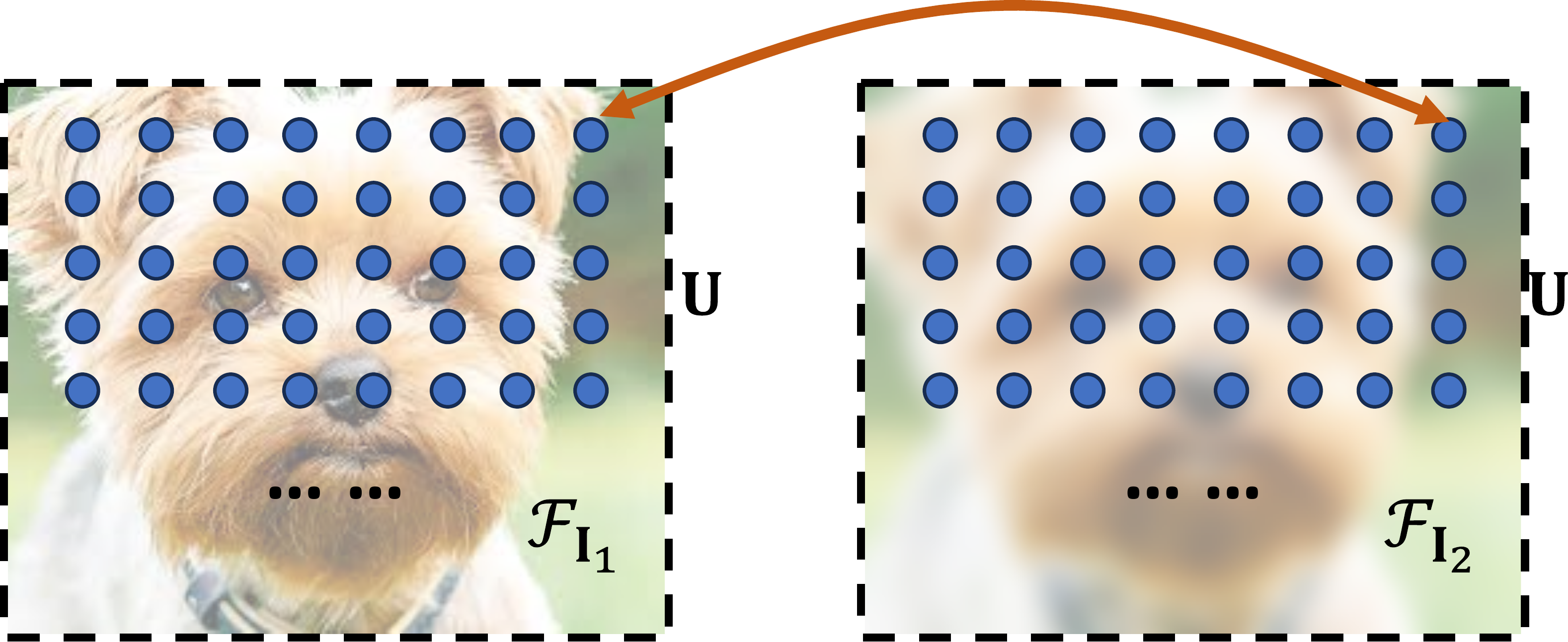}\label{IMG:CORR}} \\ 
    
    \subfloat[3D Geometry]{
    \includegraphics[width=0.9\linewidth]{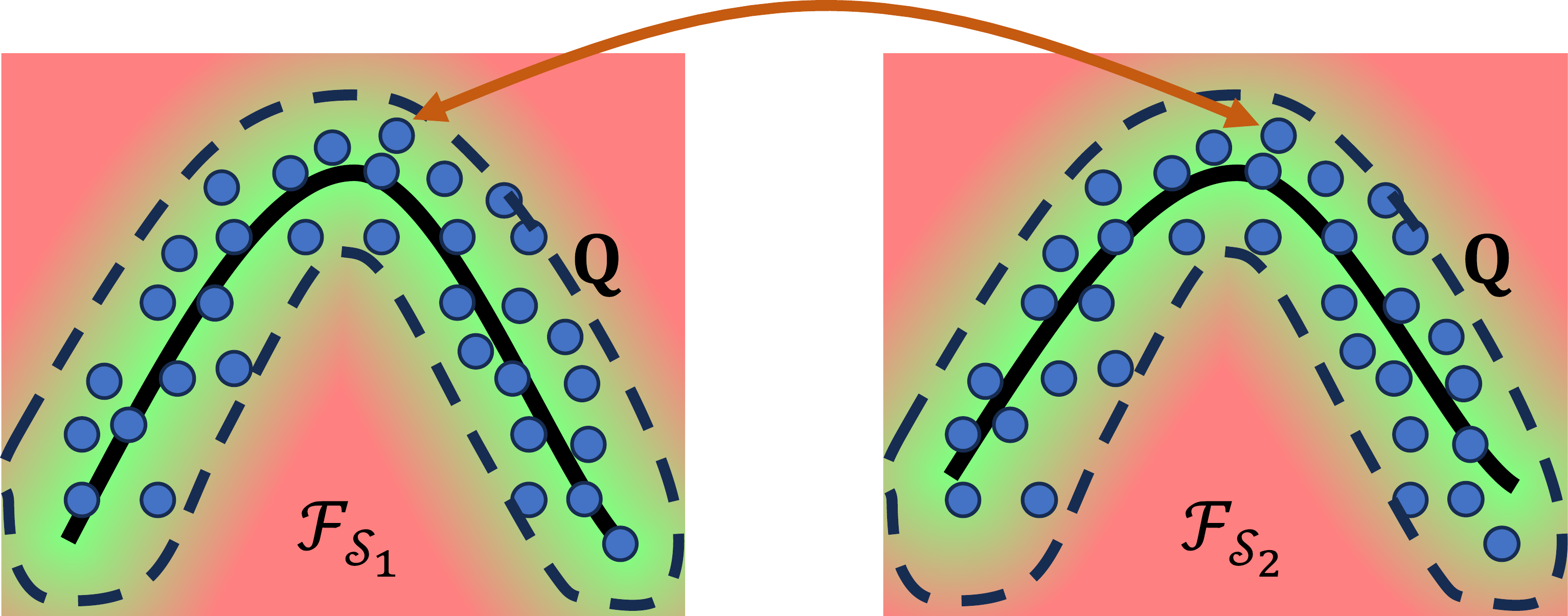}\label{SURF:CORR}}
    \caption{
    (\textbf{a})  \textit{Direct} correspondence establishment between 2D RGB images through the pixel locations uniformly distributed on a regular 2D grid (the \textcolor[rgb]{0.266,0.447,0.768}{blue} points). 
    (\textbf{b}) \textit{Indirect} correspondence establishment between 3D geometry shapes through a set of additional reference points highlighted in \textcolor[rgb]{0.266,0.447,0.768}{blue} distributed near the surfaces. Note that the two 3D shapes share an \textit{identical} set of reference points. 
    }
    \label{IMG:PC:CORR}
\end{figure}

\subsection{Problem Statement and Motivation}
Given a pair of 3D geometric models, denoted as $\mathcal{S}_1$ and $\mathcal{S}_2$, which could be represented using either point clouds or triangle meshes, we aim to devise a differentiable distance metric that efficiently and effectively measures the discrepancy between them. As mentioned earlier, the unstructured nature of 3D geometric models poses a significant challenge.  Existing metrics primarily concentrate on establishing direct correspondences between $\mathcal{S}_1$ and $\mathcal{S}_2$ in a P2P or P2F manner, and then aggregating the distances between corresponding points. However, they tend to be either time-consuming or ineffective. 

Motivated by the distance metric for 2D images in Sec. \ref{SEC:IMG:DIST}, to address the above-mentioned limitations, an intuitive solution is to seek an appropriate manner to transfer $\mathcal{S}_1$ and $\mathcal{S}_2$ as the outputs defined on an identical domain.
Specifically, as shown in Fig. \ref{SURF:CORR}, we pre-define a surface domain, $\mathbf{Q}$, for $\mathcal{S}_1$ and $\mathcal{S}_2$, representing the area near them. In practice, $\mathbf{Q}$ is discretized as a set of points located in it, called reference points $\mathbf{Q}=\{\mathbf{q}_m\in\mathbb{R}^3\}_{m=1}^{M}$.
Then, for each reference point, we establish its relationship (similar to the color mapping of images) with $\mathcal{S}_1$ and $\mathcal{S}_2$ separately, symbolized as $\mathcal{F}_{\mathcal{S}_1}(\mathbf{q}_m)$ and $\mathcal{F}_{\mathcal{S}_2}(\mathbf{q}_m)$, and the discrepancy between $\mathcal{F}_{\mathcal{S}_1}(\mathbf{q}_m)$ and $\mathcal{F}_{\mathcal{S}_2}(\mathbf{q}_m)$ reflects the difference between typical regions of $\mathcal{S}_1$ and $\mathcal{S}_2$. By aggregating the discrepancies of all reference points, we can obtain the discrepancy between $\mathcal{S}_1$ and $\mathcal{S}_2$. This indirect calibration eliminates the time-consuming correspondence-matching process employed by previous distance metrics, thereby achieving high efficiency. Technically, to realize $\mathcal{F}_{\mathcal{S}_1}(\cdot)$ and $\mathcal{F}_{\mathcal{S}_2}(\cdot)$, we propose directional distance field (DDF), which is capable of implicitly capturing the local surface geometry of $\mathcal{S}_1$ and $\mathcal{S}_2$ at the specific location of a reference point. This implicit field-based geometric modeling departs from previous methods solely focusing on P2P or P2F differences,making it a more effective approach.

In what follows, we will detail the generation of reference points in Sec. \ref{REF:GEN1} and the definition of DDF in Sec. \ref{DIR:DIS}, leading to our efficient yet effective distance metric for 3D geometry data named DDM in Sec. \ref{DIS:METRIC}. 

\subsection{Generation of Reference Points} 
\label{REF:GEN1}

In the optimization/learning process of a typical geometric modeling task, it is common for one 3D model to remain unchanged as the ground truth, \revise{while the other one is optimized/learned such that it closely approximates the ground truth. Consequently, we focuses on the regions near the unchanged model, and the reference points should be generated within these regions.}

\begin{wrapfigure}{r}{3.5cm}
    \centering
    \includegraphics[width=1\linewidth]{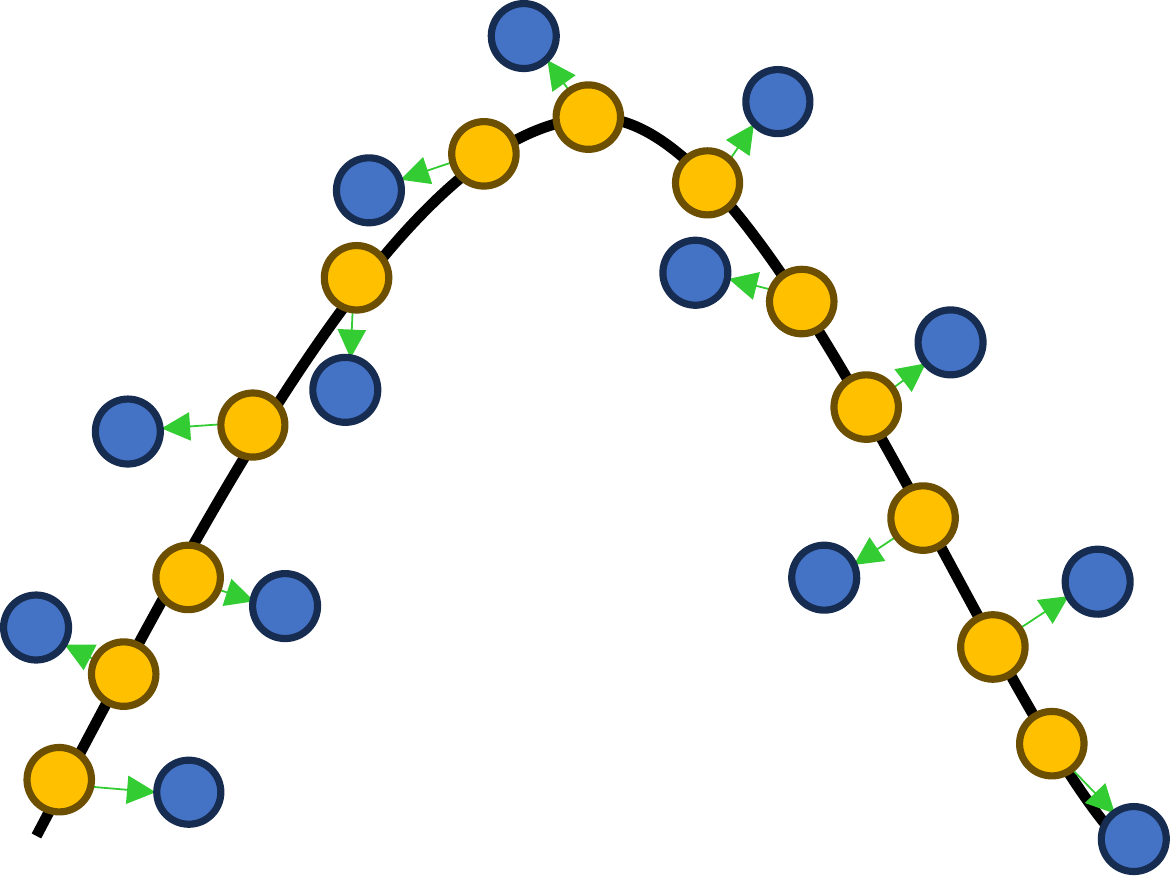}
    \caption{Illustration of the procedure for generating reference points. Here,  
    reference points in \textcolor[rgb]{0.266,0.447,0.768}{blue} are generated by adding offsets to the points sampled in \textcolor[rgb]{1,0.752,0}{orange} on the surface. 
    }\label{REF:GEN}
    \vspace{-0.4cm}
\end{wrapfigure}
Without loss of generality, let's assume $\mathcal{S}_1$ is the unchanged 3D model. In this scenario, we pre-define $\mathbf{Q}$ based on $\mathcal{S}_1$.
Specifically, we add Gaussian noise with a standard deviation $\sigma$ to the points on $\mathcal{S}_1$, displacing them away from the surface, as shown in Fig. \ref{REF:GEN}. If $\mathcal{S}_1$ is represented in the form of a 3D point cloud, we directly introduce the Gaussian noise to its points and iterate this noise addition process 
multi times randomly, resulting in $M$ reference points situated closely to the surface. If $\mathcal{S}_1$ is represented with a triangle mesh, we first sample points on it to convert it to a point cloud, and then we generate $M$ reference points from the sampled point cloud 
through the aforementioned operation.

\subsection{Directional Distance Field} \label{DIR:DIS}

Based on the definition of the reference point introduced in the preceding section, we propose 
Directional Distance Field (DDF) 
to implicitly capture 
the local surface geometry of a 3D model. 

To be specific, denote by $\mathcal{S}$ a continuous 3D surface associated with reference point $\mathbf{q}\in\mathbb{R}^3$. Let $\mathbf{\hat{q}}\in \mathcal{S}$ be the 
closest point to $\mathbf{q}$. 
Thus, the unsigned distance function (UDF) at $\mathbf{q}$ is written as  
\begin{equation} \label{UDF}
    f_\mathcal{S}(\mathbf{q})=\|\mathbf{\hat{q}}-\mathbf{q}\|_2,
\end{equation}
where $\|\cdot\|_2$ is the $\ell_2$ norm of the vector. Moreover, 
we use the direction of UDF to assist in modeling the geometric structure. 
Here, we define the direction as the vector pointing from $\mathbf{\hat{q}}$ to $\mathbf{q}$: 
\begin{equation} \label{UDF:GRAD}
    \mathbf{h}_\mathcal{S}(\mathbf{q})=\mathbf{\hat{q}}-\mathbf{q}.
\end{equation}
By concatenating $f_\mathcal{S}(\mathbf{q})$ and $\mathbf{h}_\mathcal{S}(\mathbf{q})$, we derive a 4D vector as the value of the DDF of $\mathcal{S}$ at location 
$\mathbf{q}$, i.e., 
\begin{equation}
    \mathcal{F}_{\mathcal{S}}(\mathbf{q})=[f_\mathcal{S}(\mathbf{q})||\mathbf{h}_\mathcal{S}(\mathbf{q})]\in\mathbb{R}^4. \label{ddf:concat}
\end{equation}
In the following, we will detail the calculation of $\mathbf{\hat{p}}$ when $\mathcal{S}$ is represented using  either a point cloud or a triangle mesh. 
\\

\begin{figure}[t]
\centering
\subfloat[]{\includegraphics[width=0.23\textwidth]{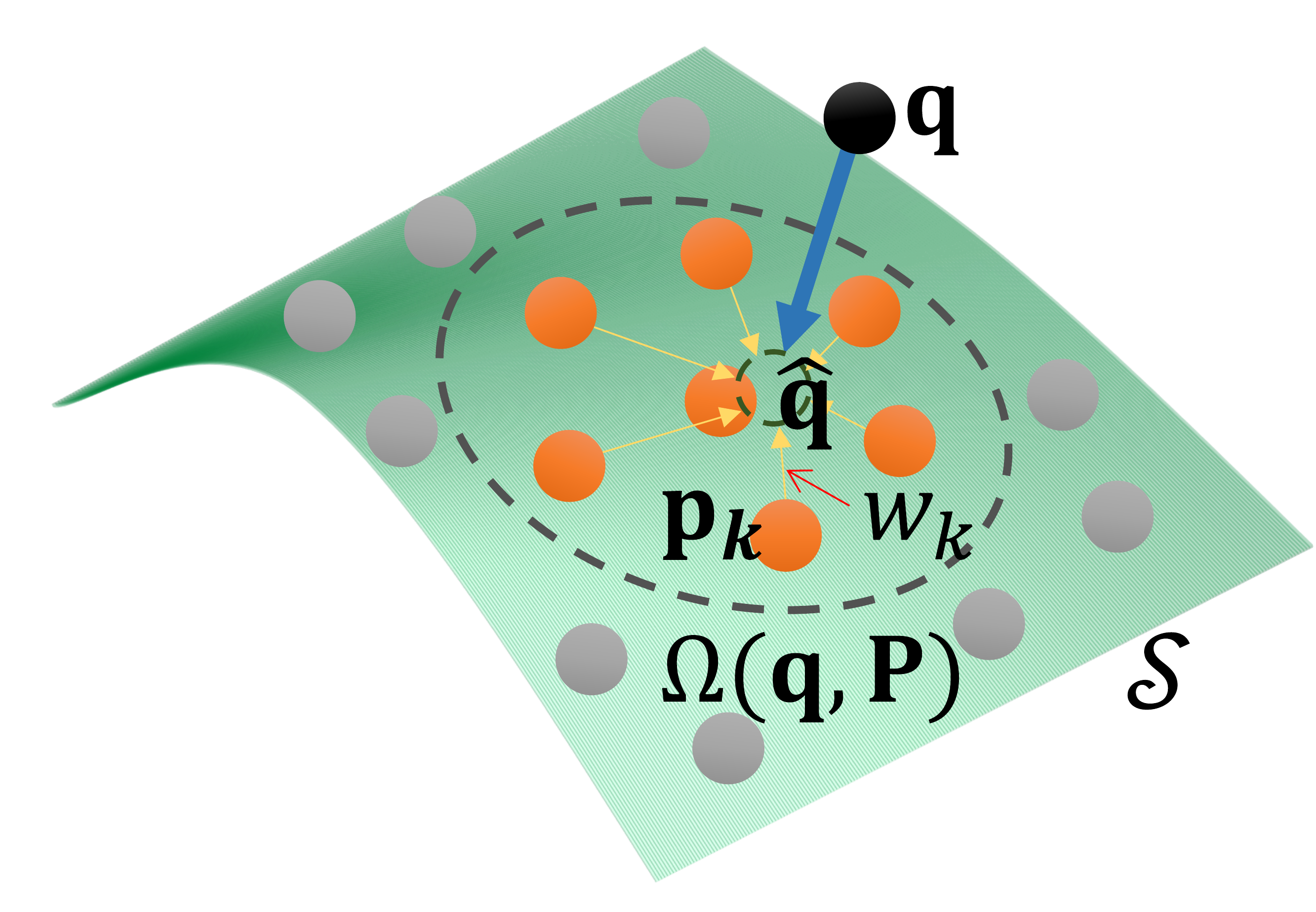} \vspace{-0.2cm}\label{CLOSEST:PC}} \
\subfloat[]{\includegraphics[width=0.23\textwidth]{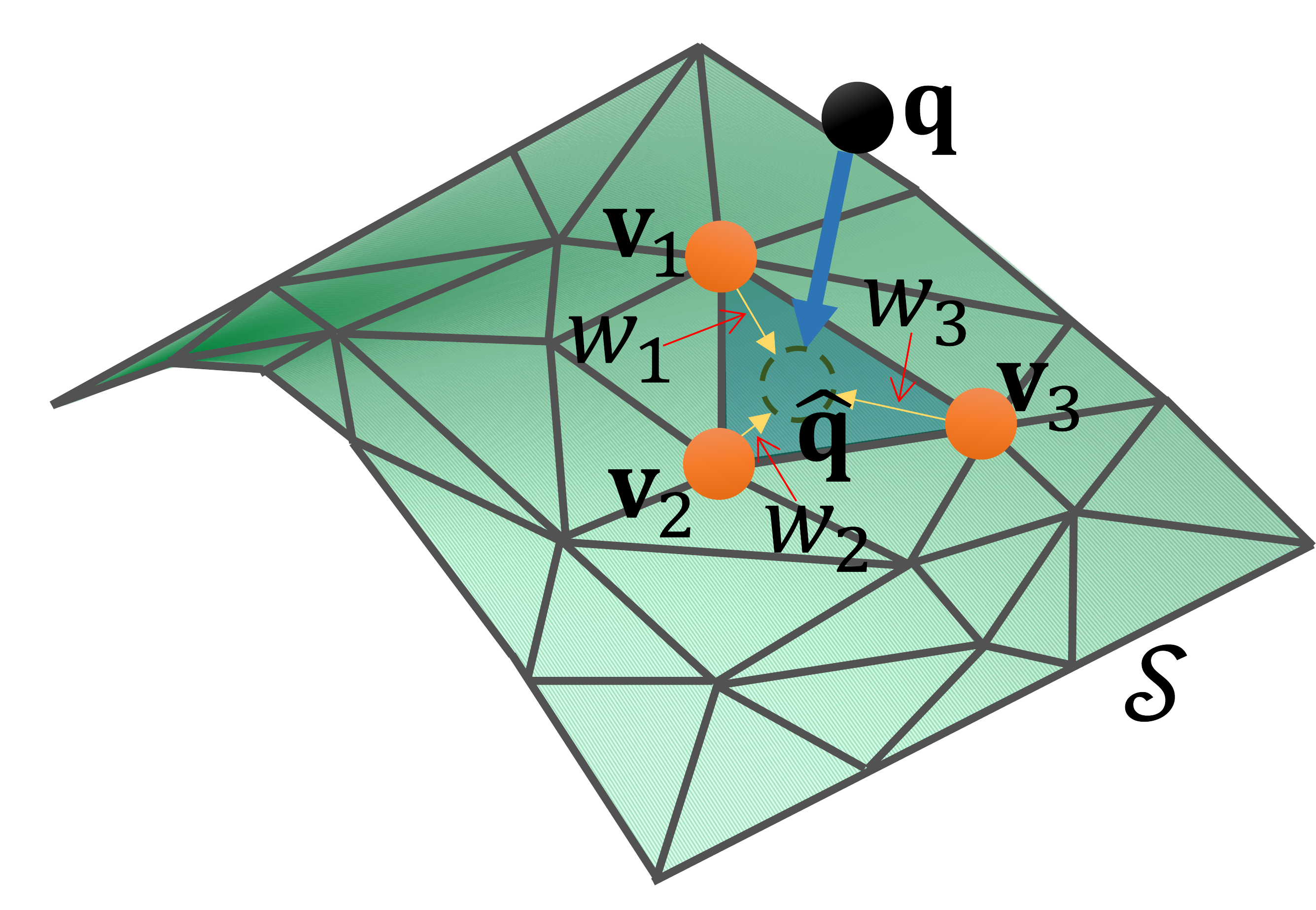}\vspace{-0.2cm}\label{CLOSEST:MESH}}
\caption{ Visualization of the closest point estimation on the surface, depicted as (a) point clouds and (b) triangle meshes.
}
\end{figure}

\noindent\textbf{Point Cloud.} Let $\mathbf{P}:=\{\mathbf{p}_i\in\mathbb{R}^3\}_{i=1}^N$ be the point cloud representing the 3D model of $\mathcal{S}$ 
and $\Omega(\mathbf{q},\mathbf{P}):=\{\mathbf{p}_k\}_{k=1}^{K}$ 
the set of $K$-NN ($K$ Nearest Neighbor) points of $\mathbf{P}$ to $\mathbf{q}$. As shown in Fig. \ref{CLOSEST:PC}, $\mathbf{\hat{q}}$ must lie in the area that $\Omega(\mathbf{q},\mathbf{P})$ covers, and we thus approximate $\mathbf{\hat{q}}$ with the weighted averaging of the points in $\Omega(\mathbf{q},\mathbf{P})$:
\begin{equation}
    \mathbf{\hat{q}}\approx\frac{\sum_{k=1}^K w(\mathbf{q},\mathbf{p}_k)\cdot\mathbf{p}_k}{\sum_{k=1}^K w(\mathbf{q},\mathbf{p}_k)}, \label{EQ:PC}
\end{equation}
where $w(\mathbf{q},\mathbf{p}_k)=1 / \|\mathbf{q}-\mathbf{p}_k\|_2^2$.\\

\noindent\textbf{Triangle Mesh.} Let $\mathbf{V}\in\mathbb{R}^{N\times3}$ and $\mathbf{F}\in\mathbb{N}^{E\times 3}$ be the sets of vertices and face indexes of a triangle mesh representing the 3D model of $\mathcal{S}$, respectively. 
By employing the point-to-surface projection method introduced in \cite{CLOSESTPOINT}, we can easily find the triangle face closest to $\mathbf{q}$, 
with vertices $\mathbf{v}_{i_1}$, $\mathbf{v}_{i_2}$ and $\mathbf{v}_{i_3}$. Furthermore, as shown in Fig. \ref{CLOSEST:MESH},
$\mathbf{\hat{q}}$ is located on $\mathbf{q}$'s closest triangle face, 
so $\mathbf{\hat{q}}$ could be represented as the weighted sum of the three vertices,
\begin{equation}
    \mathbf{\hat{q}}=w_{i_1}\mathbf{v}_{i_1}+w_{i_2}\mathbf{v}_{i_2}+w_{i_3}\mathbf{v}_{i_3}, \label{EQ:MESH}
\end{equation}
where $w_{i_1}$, $w_{i_2}$ and $w_{i_3}$ are the weights for the three vertices satisfying $w_{i_1}+w_{i_2}+w_{i_3}=1$, and can be calculated through the projection method in \cite{CLOSESTPOINT}.

\subsection{DDF-based Distance Metric} \label{DIS:METRIC}
Based on the previously introduced DDF, we can reformulate the representations of  $\mathcal{S}_1$ and $\mathcal{S}_2$ 
as $\mathcal{S}_1=\{\mathcal{F}_{\mathcal{S}_1}(\mathbf{q})|\mathbf{q}\in\mathbf{Q}\}$ and $\mathcal{S}_2=\{\mathcal{F}_{\mathcal{S}_2}(\mathbf{q})|\mathbf{q}\in\mathbf{Q}\}$, respectively, where the correspondence between $\mathcal{S}_1$ and $\mathcal{S}_2$ is established in a indirect fashion through $\mathbf{Q}$. 
We finally define the our distance metric for 3D geometry data named \textbf{DDM} as 
\begin{equation}
    \mathcal{D}_{\rm DDM}(\mathcal{S}_1,\mathcal{S}_2)=\iiint_{\mathbf{Q}}s(\mathbf{q})\cdot d(\mathbf{q},\mathcal{S}_1,\mathcal{S}_2){\rm d}\mathbf{q}, \label{DDF:EQUATION}
\end{equation}
with 
\begin{equation} d(\mathbf{q},\mathcal{S}_1,\mathcal{S}_2)=\|\mathcal{F}_{\mathcal{S}_1}(\mathbf{q})-\mathcal{F}_{\mathcal{S}_2}(\mathbf{q})\|_1, 
\end{equation}
\begin{equation}
    s(\mathbf{q})=\texttt{Exp}(-\beta\cdot d(\mathbf{q},\mathcal{S}_1,\mathcal{S}_2)),
    \label{equ:s(q)}
\end{equation}
where $s(\mathbf{q})$ 
is the confidence score of $d(\mathbf{q},\mathcal{S}_1,\mathcal{S}_2)$ with $\beta\geq 0$ being a hyperparameter. We introduce $s(\mathbf{q})$ to cope with the case where $\mathcal{S}_1$ and $\mathcal{S}_2$ are partially overlapped, i.e., the difference introduced by the reference points located at the overlapping regions have higher confidence scores than those located at the non-overlapping regions, as shown in Fig. \ref{SCORE}. 
\begin{figure}[t]
\centering
    \includegraphics[width=0.95\linewidth]{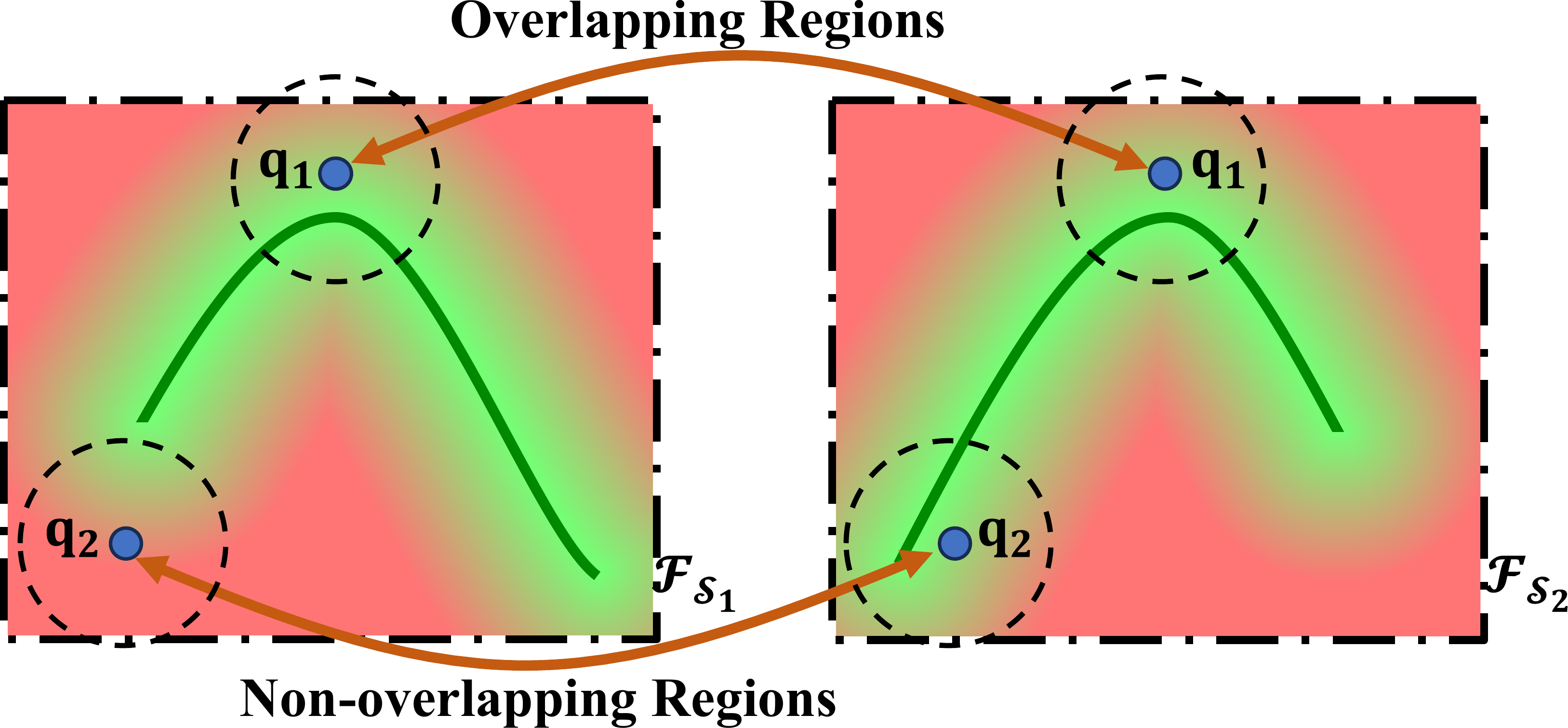}
    \caption{
    Visual illustration of the confidence scores for reference points in both overlapping and non-overlapping regions. Here, $\mathbf{q}_1$ resides in the overlapping region, while $\mathbf{q}_2$ is situated in the non-overlapping region, leading to $s(\mathbf{q}_1)>s(\mathbf{q}_2)$.
    }\label{SCORE}
\end{figure}

The proposed DDM in Eq. \eqref{DDF:EQUATION}  possesses essential properties that align with standard distance metrics, including non-negativity, symmetry, the identity of indiscernibles, and the triangle inequality. Obviously, the first three properties are satisfied once the reference points are given. 
Theorem \ref{theorem:the1} shows that it also satisfies the triangle inequality. \revise{In addition, under some specific settings, existing distance metrics, such as CD and P2F, can be regarded as special cases of our DDM. 
Theorems \ref{theorem:the2} and \ref{theorem:the3}  illustrate these properties.
}

\begin{theorem}
\label{theorem:the1}

Given three surfaces, denoted as $\mathcal{S}_1$, $\mathcal{S}_2$, and $\mathcal{S}_3$, along with the generated reference point set $\mathbf{Q}$, the following inequality holds 
\begin{equation}
    \mathcal{D}_{\rm DDF}(\mathcal{S}_1,\mathcal{S}_2)+\mathcal{D}_{\rm DDF}(\mathcal{S}_2,\mathcal{S}_3)\geq\mathcal{D}_{\rm DDF}(\mathcal{S}_1,\mathcal{S}_3).
    \nonumber
\end{equation}
\end{theorem}
\begin{proof}
    For each $\mathbf{q}\in\mathbf{Q}$, we denote $d(\mathbf{q},\mathcal{S}_1,\mathcal{S}_3)$, $d(\mathbf{q},\mathcal{S}_1,\mathcal{S}_2)$, and $d(\mathbf{q},\mathcal{S}_2,\mathcal{S}_3)$ as $d_{13}$, $d_{12}$, and $d_{23}$, respectively. According to the triangle inequality, they satisfy
\begin{equation}
    d_{12}+d_{23}\geq d_{13}. \nonumber
\end{equation}
Let $z(x)=\texttt{Exp}(-\beta x)\cdot x$, we only need to prove the following inequality under the inequal condition above
\begin{equation}
    z(d_{12})+z(d_{23})\geq z(d_{13}). \nonumber
\end{equation}
As described in Sec. \textcolor{red}{III-C} of the manuscript, the reference points $\mathbf{Q}$ are generated at the area near the surfaces, thus, the differences introduced by the reference points are sufficiently small, such that they are less than $\frac{1}{\beta}$, i.e., $d_{13},d_{12},d_{23}<\frac{1}{\beta}$.

On the other hand, we notice that $z(x)$ has the following two properties:
\begin{compactitem}
    \item $z(x)$ is monotonically increasing from 0 to $\frac{1}{\beta}$,
    \item for any non-negative constant $\alpha\leq 1$, it satisfies $z(ax)\geq a\cdot z(x)$ in the period from 0 to $\frac{1}{\beta}$.
\end{compactitem}
If at least one of $d_{12}$ and $d_{23}$ is greater than $d_{13}$, the triangle inequality obviously holds owing to the monotonically increasing behavior of $z(x)$. Otherwise, both $d_{12}$ and $d_{23}$ are less than or equal to $d_{13}$, and they could be represented as $d_{12}=a_1d_{13}$ and $d_{23}=a_2d_{23}$, where $0<a_1,a_2\leq 1$ and $a_1+a_2\geq 1$. According to the property of $z(x)$, there are 
\begin{equation}
\begin{aligned}   
   &z(d_{12})=z(a_1d_{13})\geq a_1\cdot z(d_{13}), \\ \nonumber
   &z(d_{23})=z(a_2d_{13})\geq a_2\cdot z(d_{13}). \nonumber
\end{aligned}
\end{equation}
By adding these two inequalities, we finally have
\begin{equation}
    \begin{aligned}
        z(d_{12})+z(d_{23})&\geq (a_1+a_2)\cdot z(d_{13})   \\ \nonumber
        &\geq z(d_{13}).
    \end{aligned}
\end{equation}
\end{proof}

\begin{theorem}
\label{theorem:the2}
\revise{
Given two point clouds, $\mathbf{P}_1$ and $\mathbf{P}_2$, when the reference points are defined as $\mathbf{Q}=\mathbf{P}_1\cup\mathbf{P}_2$, and the DDF is defined as $
\mathcal{F}=f$ with parameters $K=1$ and $\beta=0$, DDM becomes equivalent to CD.}
\end{theorem}
\begin{proof}
   Let first consider the reference points from $\mathbf{q}\in\mathbf{P}_1$, according to the definition of DDF, their DDFs of two point clouds can be calculated through 

\[
\begin{aligned}
&\mathcal{F}_{\mathbf{P}_1}(\mathbf{q})=f_{\mathbf{P}_1}(\mathbf{q}) =\|\texttt{NN}({\mathbf{q}},\mathbf{P}_1)-\mathbf{q}\|_2=\|\mathbf{q}-\mathbf{q}\|=0, \\
&\mathcal{F}_{\mathbf{P}_2}(\mathbf{q})=f_{\mathbf{P}_2}(\mathbf{q})=\|\texttt{NN}(\mathbf{q},\mathbf{P}_2)-\mathbf{q}\|_2,
\end{aligned}
\]
where \texttt{NN} is the 1-NN operation.
Thus, DDM under these reference points is 

\[
\begin{aligned}
\mathcal{D}_1&=\sum_{\mathbf{q}\in\mathbf{P}_1}|\mathcal{F}_{\mathbf{P}_1}(\mathbf{q})-\mathcal{F}_{\mathbf{P}_2}(\mathbf{q})|\\
&=\sum_{\mathbf{q}\in\mathbf{P}_1}|0-\|\texttt{NN}(\mathbf{q},\mathbf{P}_2)-\mathbf{q}\|_2|\\
&=\sum_{\mathbf{q}\in\mathbf{P}_1}\|\texttt{NN}(\mathbf{q},\mathbf{P}_2)-\mathbf{q}\|_2.
\end{aligned}
\]

Similarity, DDM under the reference points $\mathbf{q}\in\mathbf{P}_2$ is

\[
\mathcal{D}_2=\sum_{\mathbf{q}\in\mathbf{P}_2}\|\texttt{NN}(\mathbf{q},\mathbf{P}_1)-\mathbf{q}\|_2.
\]

Obviously, $\mathcal{D}_{\rm DDM}=\mathcal{D}_1+\mathcal{D}_2$ is the same as CD.

\end{proof}
\begin{theorem}
\label{theorem:the3}
\revise{
Given two triangle meshes, $\mathcal{S}_1$ and $\mathcal{S}_2$, when the reference points $\mathbf{Q}=\{\mathbf{q}_i\}_{i=1}^M$ are sampled on the surface, and the DDF is defined as $\mathcal{F}=f$ with parameter $\beta=0$, DDM becomes equivalent to the P2F distance.}
\end{theorem}
\begin{proof}
    Let first considering the reference point sampled from $\mathcal{S}_1$, $\mathbf{Q}_1=\{\mathbf{q}_i\}_{i=1}^{M_1}$, according to the definition of DDF, the DDFs of two triangle meshes can be represented as

\[
\begin{aligned}
    &\mathcal{F}_{\mathbf{P}_1}(\mathbf{q}_i)=f_{\mathbf{P}_1}(\mathbf{q}_i)=\|\texttt{NP}(\mathbf{q}_i,\mathcal{S}_1)-\mathbf{q}_i\|_2=\|\mathbf{q}_i-\mathbf{q}_i\|=0, \\
    &\mathcal{F}_{\mathbf{P}_2}(\mathbf{q}_i)=f_{\mathbf{P}_2}(\mathbf{q}_i)=\|\texttt{NP}(\mathbf{q}_i,\mathcal{S}_2)-\mathbf{q}_i\|_2,
\end{aligned}
\]
where $\texttt{NP}$ is the nearest point of the given point on the surface. Thus, DDM unde these reference points is
\[\begin{aligned}
\mathcal{D}_1&=\sum_{\mathbf{q}_i\in\mathbf{Q}_1}|\mathcal{F}_{\mathbf{P}_1}(\mathbf{q}_i)-\mathcal{F}_{\mathbf{P}_2}(\mathbf{q}_i)| \\
&=\sum_{\mathbf{q}_i\in\mathbf{Q}_1}|0-\|\texttt{NP}(\mathbf{q}_i,\mathcal{S}_2)-\mathbf{q}_i\|_2| \\
&=\sum_{\mathbf{q}_i\in\mathbf{Q}_1}\|\texttt{NP}(\mathbf{q}_i,\mathcal{S}_2)-\mathbf{q}_i\|_2
\end{aligned}
\]

Similarity, DDM under the reference points sampled from $\mathcal{S}_2$, $\mathbf{Q}_2=\{\mathbf{q}_i\}_{i=1}^{M_2}$ is
\[
\mathcal{D}_2=\sum_{\mathbf{q}_i\in\mathbf{Q}_2}\|\texttt{NP}(\mathbf{q}_i,\mathcal{S}_1)-\mathbf{q}_i\|_2.
\]

Obviously, $\mathcal{D}_{\rm DDM}=\mathcal{D}_1+\mathcal{D}_2$ is the same as P2F distance.
\end{proof}

\section{Distance Metric-driven 3D Geometric Modeling}
\label{sec:tasks}

To demonstrate the superiority of the proposed DDM, we apply it across an extensive array of fundamental 3D geometric modeling tasks, encompassing template-based surface reconstruction, rigid and non-rigid surface registration, scene flow estimation, and SMPL registration. In what follows, we will introduce the detailed method of each task.

\subsection{Template Surface Fitting} 
Template surface fitting is a widely studied approach to reduce the ill-posedness of 3D reconstruction, which is meaningful in the research of homeomorphic structures.
In particular, this process entails deforming an initial surface characterized by a regular shape (such as a cube or spherical mesh), denoted as $\mathcal{S}_{\rm init}$ with vertices $\mathbf{V}\in\mathbb{R}^{N\times 3}$ and face indices $\mathbf{F}\in\mathbb{R}^{E\times 3}$, into the desired target surface, denoted as $\mathcal{S}_{\rm tgt}$. Notably, the face indices remain constant throughout the deformation.

Here, we consider the recent pipeline introduced in 
\cite{LARGESTEP, MDA}. 
Specifically, in order to avoid face intersections during the surface deformation, we optimize the diffusion reparameterization rather than the vertices of $\mathcal{S}_{\rm init}$.
In this process, the distance between the deformed initial surface and the target surface is used to guide the optimization, which is critical in the whole process.
Technically, the diffusion reparameterization is defined as $\mathbf{u}=(\mathbf{I}+\alpha\mathbf{L})\mathbf{V}$, where $\mathbf{L}\in\mathbb{R}^{N\times N}$ is the discrete Laplace operator, $\mathbf{I}\in\mathbb{R}^{N\times N}$ refers to the identity matrix, and $\alpha$ is a constant weight. Consequently, the deformed surface, denoted as $\mathcal{S}^{'}$, possesses vertices $\mathbf{V}'=(\mathbf{I}+\alpha\mathbf{L})^{-1}\mathbf{u}$, while retaining the same face indices as $\mathcal{S}_{\rm init}$. $\mathbf{u}$ could be optimized by minimizing the distance between $\mathcal{S}^{'}$ and $\mathcal{S}_{\rm tgt}$:
\begin{equation}
\begin{aligned}
    \mathbf{\hat{u}}=\mathop{\arg\min}_{\mathbf{u}}  \left(  \mathcal{D}(\mathcal{S}{'},\mathcal{S}_{\rm tgt}) + \mathcal{R}_{\rm DA}(\mathcal{S}{'}) \right), \label{SURF:FITTING:EQ}
\end{aligned}
\end{equation}
where $\mathcal{D}(\cdot,\cdot)$ stands for a typical metric measuring the discrepancy between two 3D shapes, and $\mathcal{R}_{\rm DA}(\cdot)$ is the density adaptation regularization \cite{MDA}, defined as 
\begin{align}
    &\mathcal{R}_{\rm DA}(\mathcal{S}{'})=\lambda_1\mathcal{E}(\mathbf{V}{'}, {\bar{l}}_{\rm a})+ \lambda_2\mathcal{E}(\mathbf{V}{'}, {\bar{l}}_{\rm k})\\
    &\mathcal{E}(\mathbf{V}',{\bar{l}})=\frac{1}{|\mathbf{V}'|}\sum_{\mathbf{v}'\in\mathbf{V}'}|{l}(\mathbf{v}')-{\bar{l}}|^2,
\end{align}
where ${l}(\cdot)$ computes the mean length of all edges associated with a vertex, $\bar{l}_{\rm a}$ and $\bar{l}_{\rm k}$ are two kinds of expected edge lengths, 
and $\lambda_1$ and $\lambda_2$ are the weights to balance these regularization terms. The optimization problem in Eq. \eqref{SURF:FITTING:EQ} can be solved with gradient descent methods.

\subsection{Rigid Registration of 3D Point Clouds}
Given a source point cloud $\mathbf{P}_{\rm src}\in\mathbb{R}^{N_1\times 3}$ and a target point cloud $\mathbf{P}_{\rm tgt}\in\mathbb{R}^{N_2\times 3}$, rigid registration aims to estimate a spatial transformation $[\mathbf{R},\mathbf{t}]$ to align $\mathbf{P}_{\rm src}$ with $\mathbf{P}_{\rm tgt}$, where $\mathbf{R}\in\texttt{SO}(3)$ is the rotation matrix and $\mathbf{t}\in\mathbb{R}^3$ is the translation vector. Simply, we can achieve it by optimizing the following objective function:
\begin{equation}
    \{\hat{\mathbf{R}},\hat{\mathbf{t}}\}=\mathop{\arg\min}_{\mathbf{R},\mathbf{t}}\mathcal{D}\left(\mathcal{T}(\mathbf{P}_{\rm src},\mathbf{R},\mathbf{t}), \mathbf{P}_{\rm{tgt}}\right), \label{RIGID:REGISTRATION}
\end{equation}
where $\mathcal{D}(\cdot,\cdot)$ stands for a typical distance metric for 3D point clouds, and $\mathcal{T}(\cdot, \cdot, \cdot)$ is the rigid transformation operator. 
From Eq. \eqref{RIGID:REGISTRATION}, it is obvious that the distance metric between the two point cloud plays a critical role in this task, determining the accuracy of registration.

\begin{algorithm}[t]
  \caption{Construction of the deformation nodes.} \label{DEFORM:NODE}
  
  \revise{\KwIn{The source surface with vertex set $\mathbf{V}$ and face index set $\mathbf{F}$; the distance threshold $\epsilon$.}}
  \revise{\KwOut{The coordinates of deformation nodes $\mathbf{V}_{\rm DF}$.}}
  {Initialize $\mathbf{V}_{\rm DF}=\emptyset;$ \\
  Compute geodesic distance between any two points of $\mathbf{V}$; \\
  \While{$\mathbf{V}\neq\emptyset$}{
    Randomly select a point $\mathbf{v}$ from $\mathbf{V}$\;
    Add selected $\mathbf{v}$ to $\mathbf{V}_{\rm DF}$\;
    Delete all points in $\mathbf{V}$ that are within a geodesic distance of less than $\epsilon$ from $\mathbf{v}$\;
  }
  \Return $\mathbf{V}_{\rm DF}$\
  }
\end{algorithm}

\subsection{Non-Rigid 3D Mesh Registration}
Let $\mathcal{S}_{\rm src}$ and $\mathcal{S}_{\rm tgt}$ be a source and a target 3D shapes in the form of triangle meshes,  and $\mathbf{V}\in\mathbb{R}^{N\times3}$ and $\mathbf{F}\in\mathbb{N}^{E\times 3}$ the vertex and face index sets of $\mathcal{S}_{\rm src}$, respectively. Non-rigid registration targets computing a non-rigid deformation field 
align $\mathcal{S}_{\rm src}$ with $\mathcal{S}_{\rm tgt}$.

\begin{figure}[h]
\centering
\subfloat[]{\includegraphics[height=0.5\linewidth]{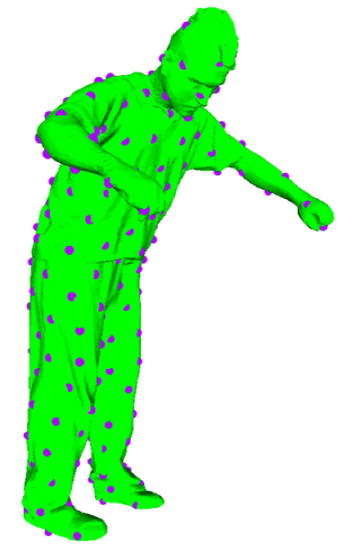}\vspace{-0.2cm}}\quad\quad
\subfloat[]{\includegraphics[height=0.5\linewidth]{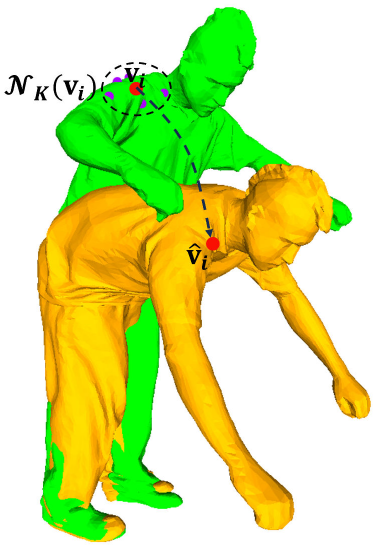}\vspace{-0.2cm}}
\caption{Visual illustration of (\textbf{a}) generated deformation nodes (i.e., the \textcolor[rgb]{0.666,0,1}{purple} points) and (\textbf{b}) the deformation between the source (\textcolor{green}{green}) and target (\textcolor{orange}{orange}) meshes. The \textcolor{red}{red} point refers to a typical vertex on the mesh. 
}   \label{NONRIGID:REG:PIPELINE} 
\vspace{-0.5cm}
\end{figure}

Following \cite{DEFORMATIONNODE}, we define the deformation field by an embedded deformation graph with deformation nodes $\mathbf{V}_{\rm DF}\in\mathbf{R}^{N'\times 3}$, which could be obtained through Algorithm \ref{DEFORM:NODE}. 
Each node encodes a rigid transformation, $[\mathbf{R}_j,\mathbf{t}_j]\ (j=1,...,N')$ with $\mathbf{R}\in{\rm \texttt{SO}(3)}$ and $\mathbf{t}_j\in\mathbb{R}^3$. For each vertex $\mathbf{v}_i\in\mathbf{V}$, we determine 
its $K$-NN nodes from $\mathbf{V}_{\rm DF}$ 
using the geodesic distance,
denoted as $\mathcal{N}_K(\mathbf{v}_i)$, then obtain its new position after deformation as 
\begin{equation}
\label{DEFORMATION:NODE}
    \mathbf{\hat{v}}_i= \frac{\sum_{\mathbf{v}_j\in\mathcal{N}_K(\mathbf{v}_i)}w(\mathbf{v}_i,\mathbf{v}_j)\cdot\big(\mathbf{R}_j(\mathbf{v}_i-\mathbf{v}_j)+\mathbf{v}_j+\mathbf{t}_j\big)}{\sum_{\mathbf{v}_j\in\mathcal{N}_K(\mathbf{v}_i)}w(\mathbf{v}_i,\mathbf{v}_j)}, 
\end{equation}
where $w(\mathbf{v}_i,\mathbf{v}_j)=\texttt{max}(0,(1-d_{G}(\mathbf{v}_i,~\mathbf{v}_j)^2/\epsilon^2)^3)$ and $d_G$ is the geodesic distance of two points on the surface. Obviously, the deformation of $\mathcal{S}_{\rm src}$ could be controlled by $\{[\mathbf{R}_j,\mathbf{t}_j]\}_{j=1}^{N'}\}$, as shown in Fig. \ref{NONRIGID:REG:PIPELINE}.

To derive the optimal $\{[\mathbf{R}_j,\mathbf{t}_j]\}_{j=1}^{N'}$, we 
optimize the following objective 
\begin{equation}
\begin{aligned}
\{[\mathbf{\hat{R}}_j,\mathbf{\hat{t}}_j]\}_{j=1}^{N'}\}=\mathop{\arg\min}_{\{[\mathbf{R}_j,\mathbf{t}_j]\}_{j=1}^{N'}} &\Big(\mathcal{D}(\mathcal{\hat{S}},~\mathcal{S}_{\rm tgt})\Big.
+\Big.\lambda \mathcal{R}_{\rm smooth}(\mathbf{V})\Big),
\end{aligned}
\label{eq:non-rigid}
\end{equation}
where $\mathcal{D}(\cdot,~\cdot)$ is a typical distance metric for computing the discrepancy between two triangle meshes, 
$\mathcal{\hat{S}}$ stands for the deformed mesh from $\mathcal{S}_{\rm src}$ with vertex set $\mathbf{\hat{V}}$ and the same face indexes as $\mathcal{S}_{\rm src}$, and $\mathcal{R}_\texttt{smooth}(\mathbf{V})$ is the spatial smooth regularization for the offsets of each vertex, defined as 
\begin{equation}
\begin{aligned}
    \mathcal{R}_{\rm smooth}(\mathbf{V})&=\frac{1}{3|\mathbf{F}|}\sum_{(i_1,i_2,i_3)\in\mathbf{F}}\big(\|\mathbf{\Delta v}_{i_1}-\mathbf{\Delta v}_{i_2}\|_2\\
    &+\|\mathbf{\Delta v}_{i_1}-\mathbf{\Delta v}_{i_3}\|_2+\|\mathbf{\Delta v}_{i_2}-\mathbf{\Delta v}_{i_3}\|_2\big),
\end{aligned}
\end{equation}
where $\mathbf{\Delta v}_{*}=\mathbf{\hat{v}}_{*}-\mathbf{v}_{*}$ is the offset of the vertex. The optimization process can be solved with gradient descent methods.

\subsection{Scene Flow Estimation}
 Denote by $\mathbf{P}_{\rm src}\in\mathbb{R}^{N_{\rm 1}\times 3}$ and $\mathbf{P}_{\rm tgt}\in\mathbb{R}^{N_{\rm 2}\times 3}$ a source and a target 3D point clouds, 
 where $N_{\rm src}$ and $N_{\rm tgt}$ are the number of points.
Scene flow estimation aims to predict point-wise offsets $\mathbf{\Delta P}\in\mathbb{R}^{N_{\rm src}\times 3}$ for $\mathbf{P}_{\rm src}$ to align it with 
$\mathbf{P}_{\rm tgt}$. This task can be achieved by directly solving the following optimization problem: 
\begin{equation}
    \mathbf{\Delta\hat{P}}=\mathop{\arg\min}_{\mathbf{\Delta P}}\mathcal{D}(\mathbf{P}_{\rm src}+\mathbf{\Delta P},\mathbf{P}_{\rm tgt})+\lambda \mathcal{R}_{\rm smooth}(\mathbf{\Delta P}),
\end{equation}
where $\mathcal{R}_{\rm smooth}(\cdot)$ is the spatial smooth regularization term, defined as 
\begin{equation}
    \mathcal{R}_{\rm smooth}(\mathbf{\Delta P})=\frac{1}{3N_{\rm src}K_s}\sum_{\mathbf{p}\in\mathbf{P}_{\rm src}}\sum_{\mathbf{p}'\in\mathcal{N}(\mathbf{p})}\|\mathbf{\Delta}\mathbf{p}-\mathbf{\Delta}\mathbf{p}'\|_2^2,
\end{equation}
where $\mathcal{N}(\mathbf{p})$ is the operator returning $\mathbf{p}$'s $K_s$-NN points in $\mathbf{P}_{\rm src}$.

In addition to the above-mentioned optimization-based method, we also consider unsupervised learning-based scene flow estimation. 
Specifically, we can predict the scene flow between $\mathbf{P}_{\rm src}$ and $\mathbf{P}_{\rm tgt}$ by using a neural network $h_{\bm{\theta}}(\cdot,~\cdot)$ parameterized by $\bm{\theta}$ taking them as inputs, i.e., $\mathbf{\Delta P}=\mathbf{h}_{\bm{\theta}}(\mathbf{P}_{\rm src},\mathbf{P}_{\rm tgt})$. 
Then the network can  be trained by minimizing the following loss function: 

\begin{equation}
\hat{\bm{\theta}}=\mathop{\arg\min}_{\bm{\theta}}\mathcal{D}\left(\mathbf{P}_{\rm src}+\mathbf{h}_{\bm{\theta}}(\mathbf{P}_{\rm src},\mathbf{P}_{\rm tgt})\right)+\lambda\mathcal{R}_{\rm smooth}(\mathbf{\Delta P}).
\end{equation}

\subsection{Human Pose Optimization from Point Clouds}
In the domain of digital human research, the 3D data representing humans are usually collected through some common 3D scanners, e.g. RGBD cameras and Lidars. As shown in Fig. \ref{SMPL：PIPELINE}, given the initial human pose, which differs from the scanned data, we need to refine it according to the scanned data, making it aligned with the scanned data. 
Skinned Multi-Person Linear Model (SMPL) \cite{SMPL} is the most widely used parametric model to represent the 3D human body. 

\begin{figure}[h]
    \centering
    \includegraphics[width=0.7\linewidth]{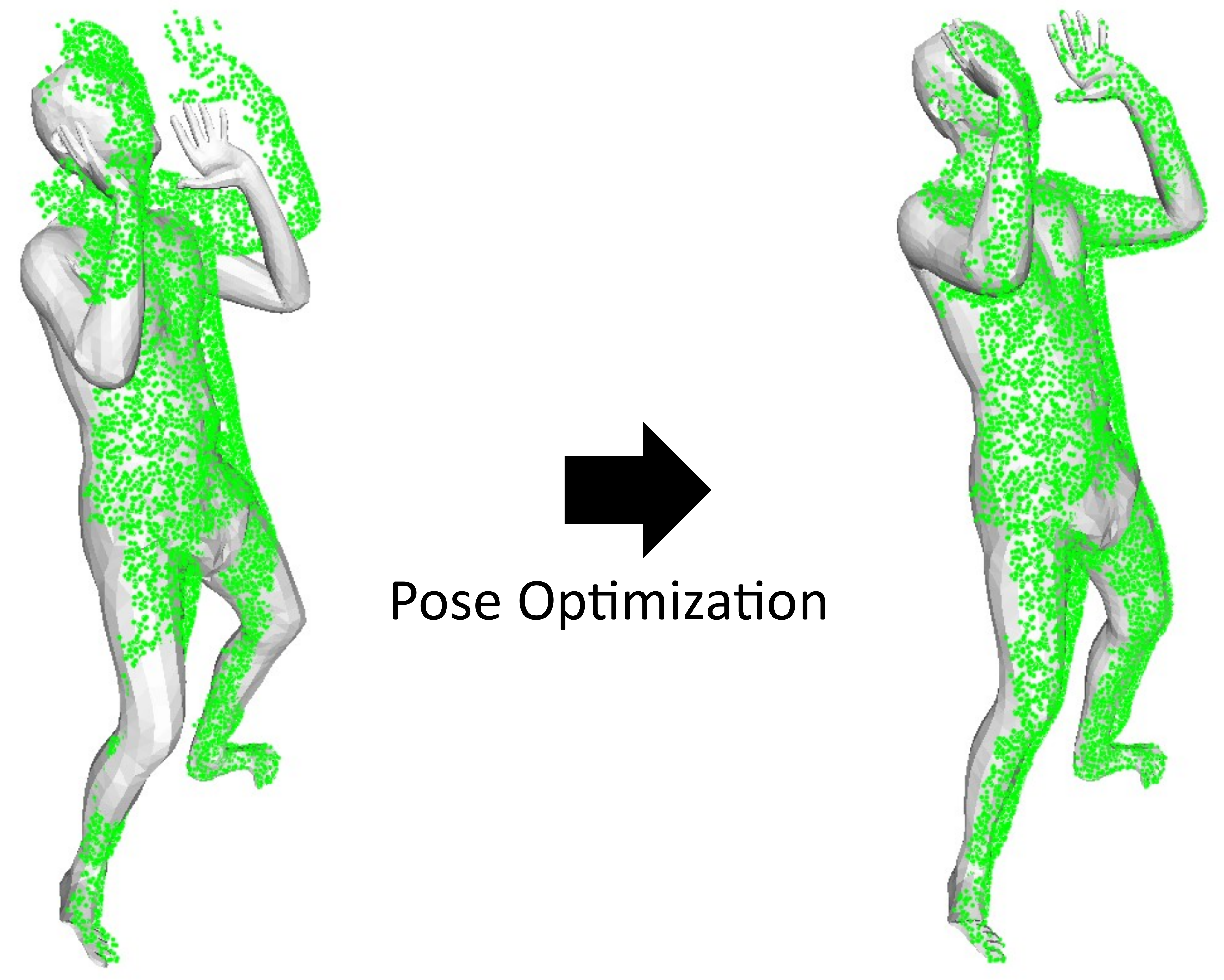}
    \caption{Visual illustration of human pose optimization from point clouds. The \textcolor{green}{green} points represent the scanned partial point cloud, and the \textcolor{gray}{gray} mesh model refers to the SMPL model.}\label{SMPL：PIPELINE}
    \vspace{-0.5cm}
\end{figure}

Specifically, the SMPL model is a triangle mesh with the human body shape, and its pose is controlled by a matrix $\mathbf{J}\in\mathbb{R}^{24\times 3}$, thus, various values of $\mathbf{J}$ could represent the human models with different pose. In reality, we need to optimize the human pose according to the given human scan, which is usually a whole or partial point cloud. This could be achieved by minimizing the difference between the SMPL model and the scanned point cloud.
Denote by $\mathbf{P}_{\rm scan}$ the scanned point cloud, the human pose estimation could be formulated as 
\begin{equation}
    \mathbf{\hat{J}}=\mathop{\arg\min}_{\mathbf{J}}\mathcal{D}(\texttt{SMPL}(\mathbf{J}),\mathbf{P}_{\rm scan}), \label{SMPL:REG:EQ}
\end{equation}
where $\mathcal{D}(\cdot,\cdot)$ serves as the distance metric to measure the disparity between the SMPL model and the reference point cloud. A good distance metric can boost the optimization and increase the accuracy of the estimated pose. We can employ gradient descent methods to solve the optimization problem in Eq. \eqref{SMPL:REG:EQ}. \\

\section{Experiments} \label{sec:exp}

In this section, we conducted extensive experiments to demonstrate the advantages of the proposed distance metric under the five fundamental 3D geometric modeling tasks introduced in Section \ref{sec:tasks}. In addition, we conducted comprehensive ablation studies to understand it better. We implemented all experiments with a system equipped with an NVIDIA RTX 3090 and an Intel(R) Xeon(R) CPU.

\subsection{Template Surface Fitting}

\subsubsection{Implementation Details} 
We utilized  the 3DCaricshop dataset \cite{3DCARICSHOP} for evaluation. 
The chosen initial surface is a unit ico-sphere, encompassing 10242 vertices and 20480 faces. The weight in diffusion reparameterization is set $\alpha=1$. 
And the weights to balance different items in objective function were set as $\lambda_1=1.5$, $\lambda_2=4.5$. Additionally, the values of $M$ and $\sigma$ were set to 4 and 0.05, respectively, for generating reference points. 
We run the Adam optimizer with a learning rate of 0.05 to optimize $\mathbf{u}$. \\

\begin{figure*}[h]
    \centering
    \resizebox{0.95\textwidth}{!}{
    \begin{tikzpicture}[]
    \node[] (a) at (-10,7.5) {\rotatebox{90}{ Will Smith}};
    \node[] (a) at (0,7.5) { \includegraphics[width=0.2\textwidth]{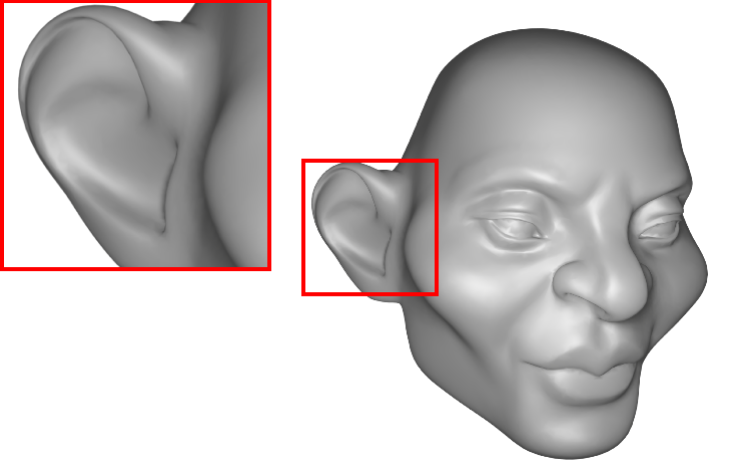}\quad
     \includegraphics[width=0.2\textwidth]{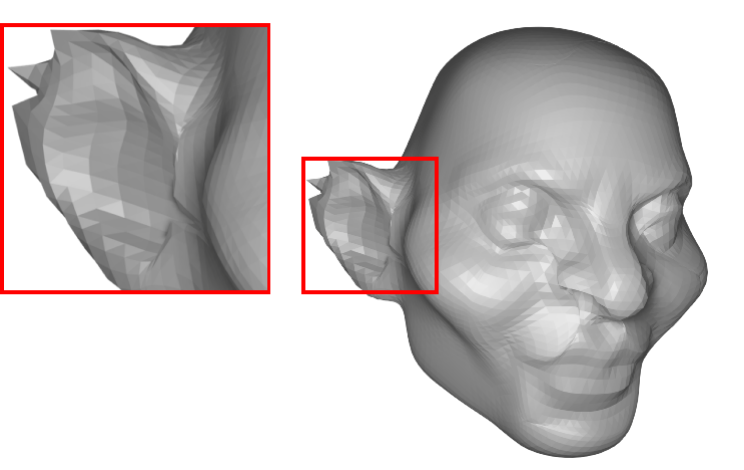}\quad
     \includegraphics[width=0.2\textwidth]{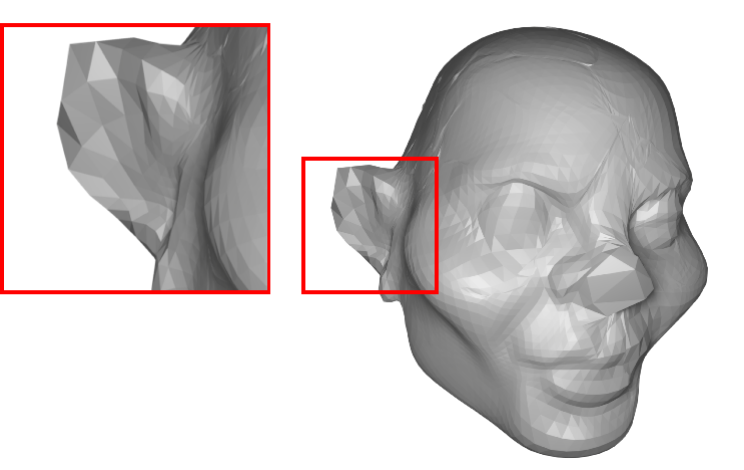}\quad
     \includegraphics[width=0.2\textwidth]{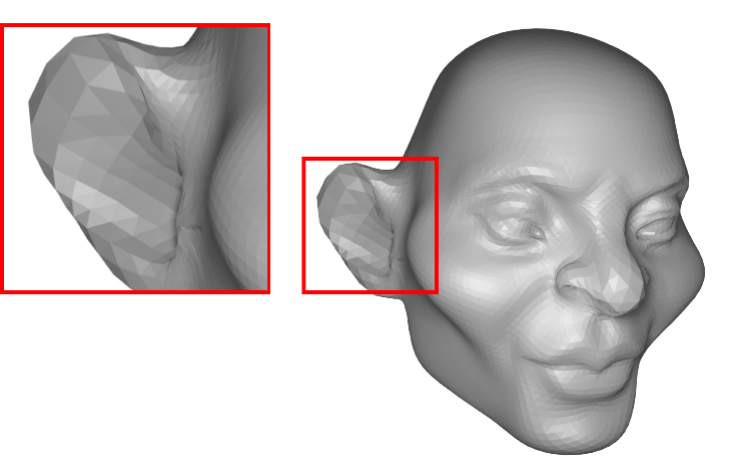}\quad
     \includegraphics[width=0.2\textwidth]{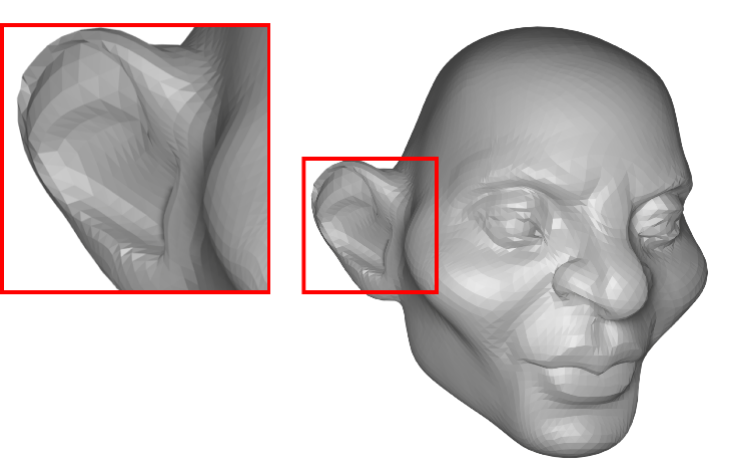}};
    
    \node[] (a) at (-10,5) {\rotatebox{90}{ Woody Allen}};
    \node[] (a) at (0,5) { \includegraphics[width=0.2\textwidth]{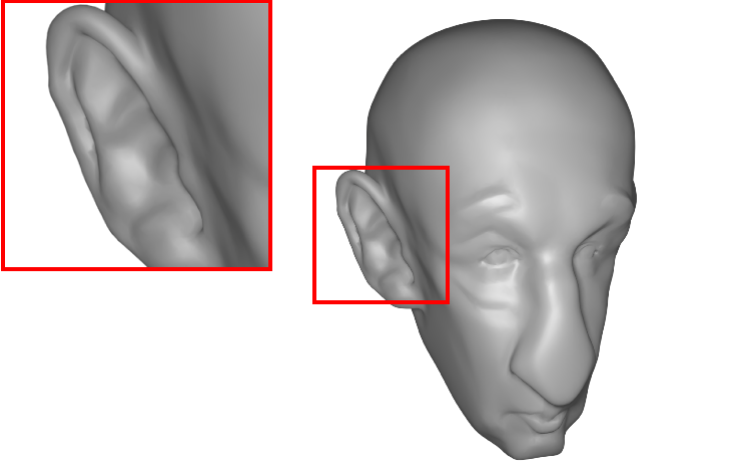}\quad
    \includegraphics[width=0.2\textwidth]{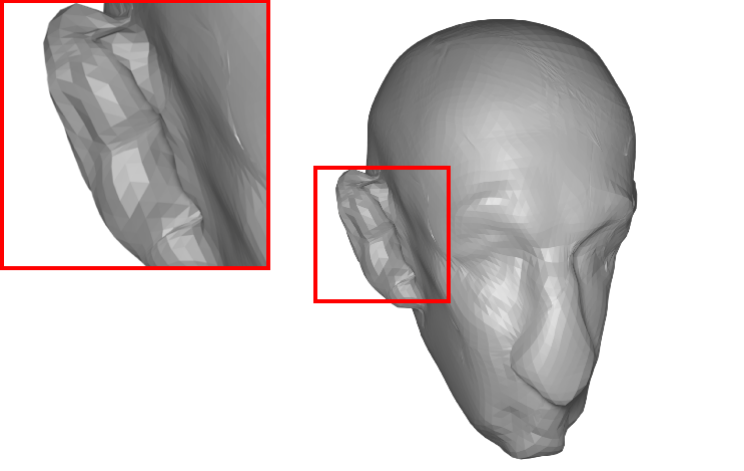}\quad
    \includegraphics[width=0.2\textwidth]{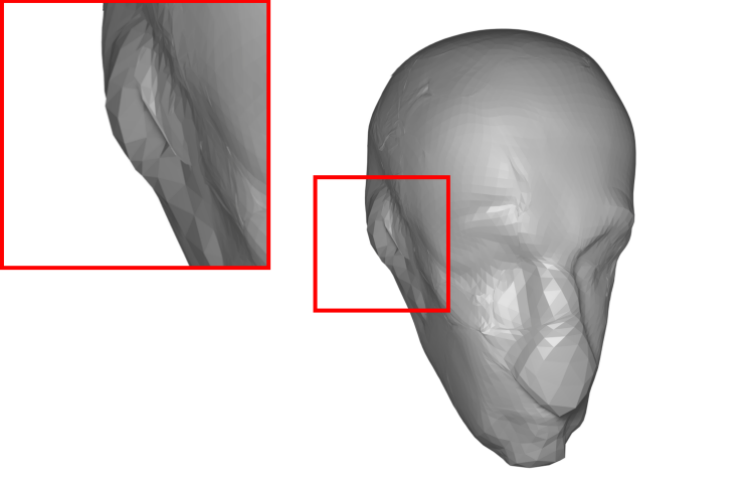}\quad
    \includegraphics[width=0.2\textwidth]{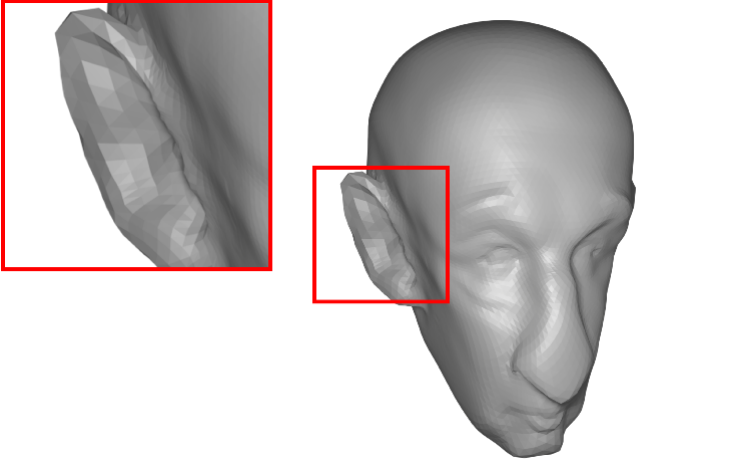}\quad
    \includegraphics[width=0.2\textwidth]{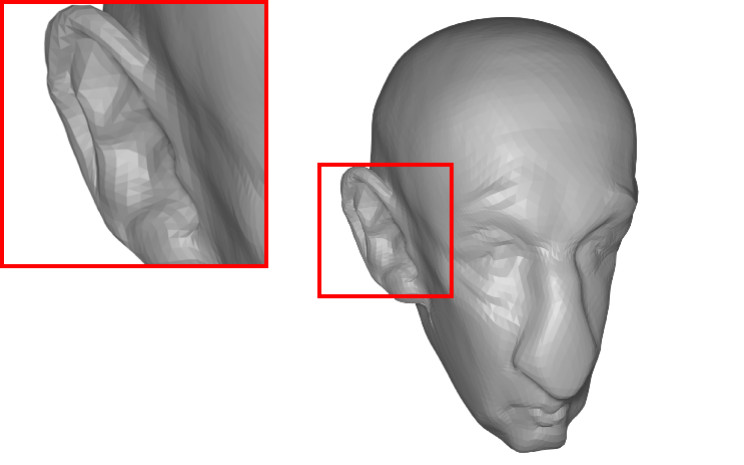}};

    \node[] (a) at (-10,2.5) {\rotatebox{90}{ Xiyuan Xu}};
    \node[] (a) at (0,2.5) { \includegraphics[width=0.2\textwidth]{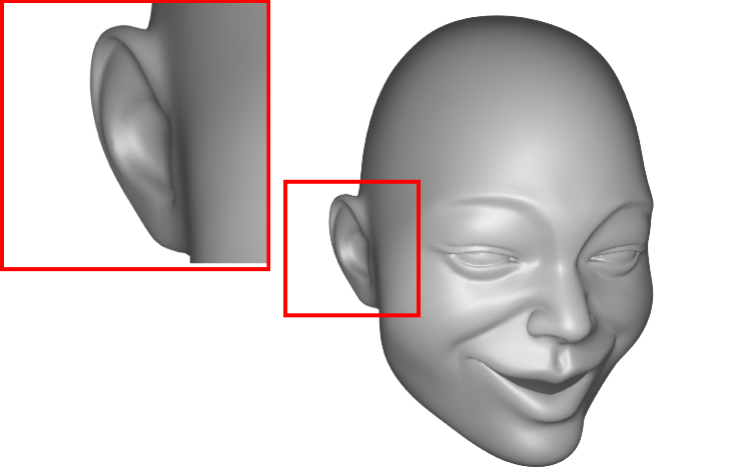}\quad
    \includegraphics[width=0.2\textwidth]{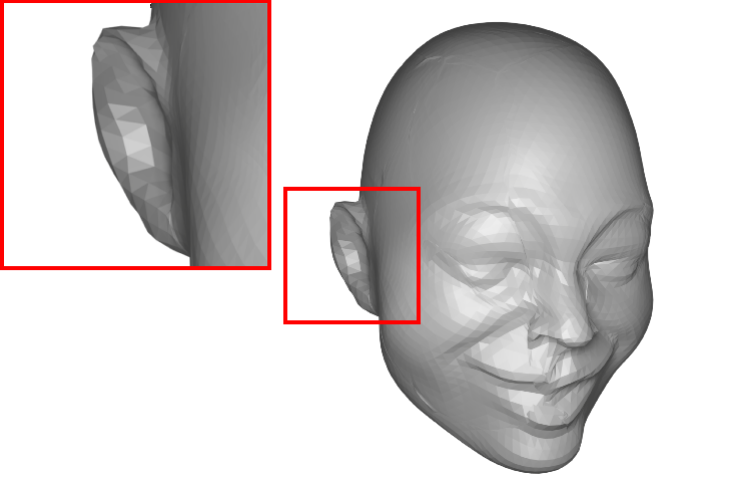}\quad
    \includegraphics[width=0.2\textwidth]{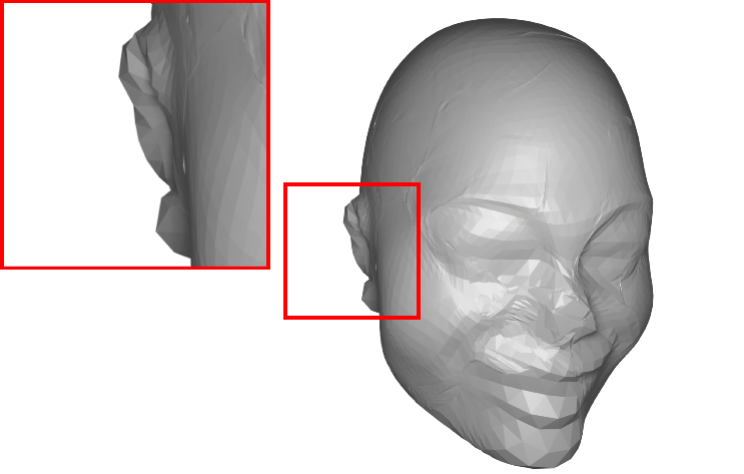}\quad
    \includegraphics[width=0.2\textwidth]{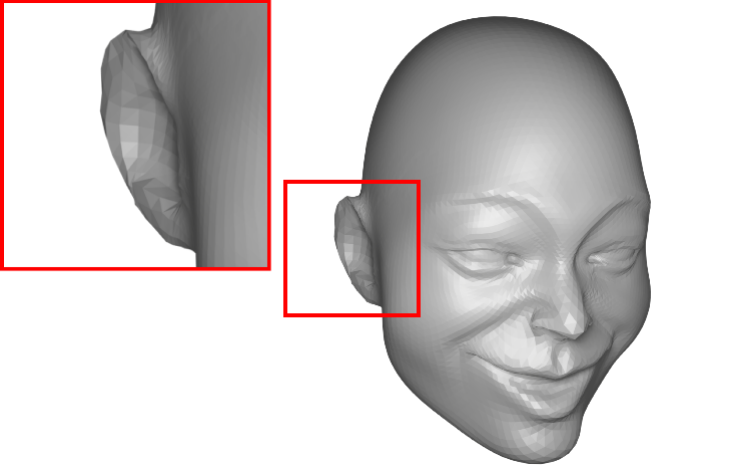}\quad
    \includegraphics[width=0.2\textwidth]{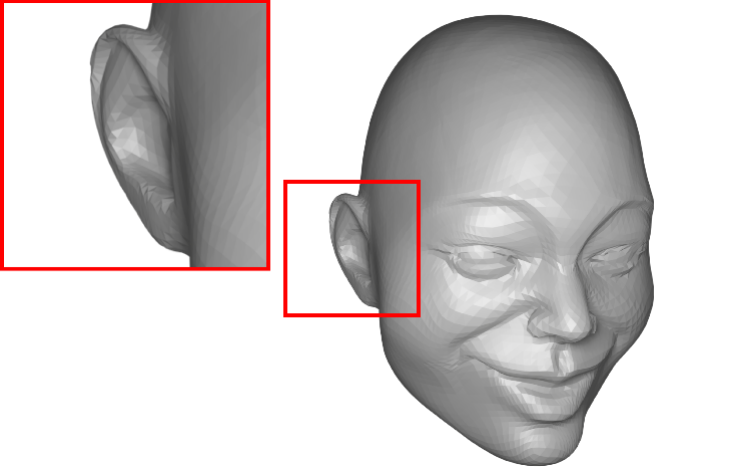}};

    \node[] (a) at (-10,0) {\rotatebox{90}{ Zooey Deschanel}};
    \node[] (a) at (0,0) { \includegraphics[width=0.2\textwidth]{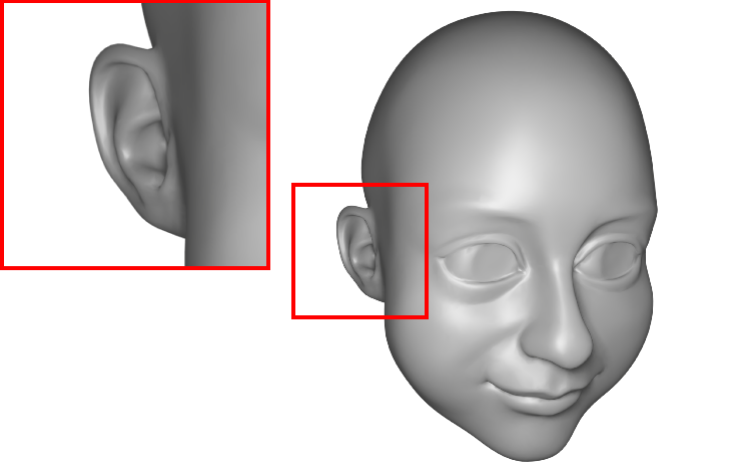}\quad
    \includegraphics[width=0.2\textwidth]{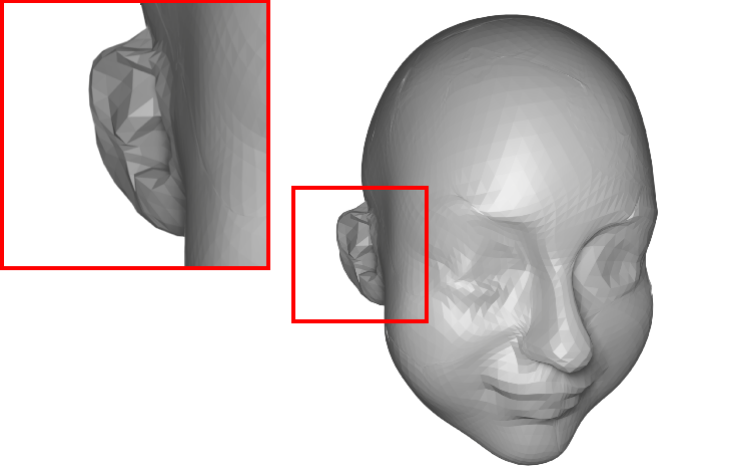}\quad
    \includegraphics[width=0.2\textwidth]{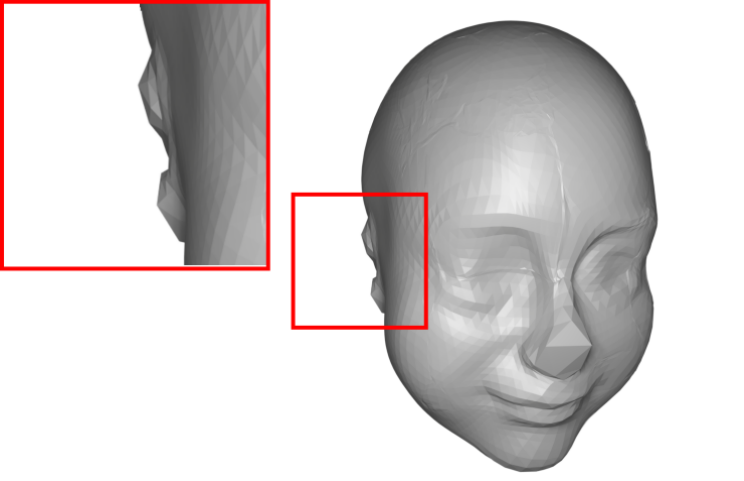}\quad
    \includegraphics[width=0.2\textwidth]{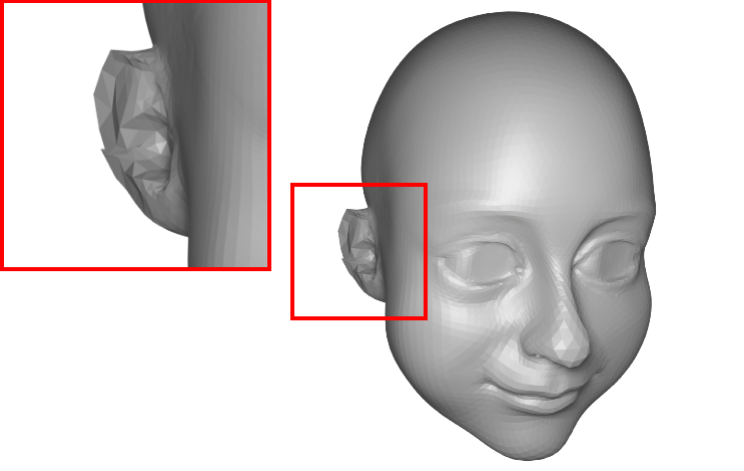}\quad
    \includegraphics[width=0.2\textwidth]{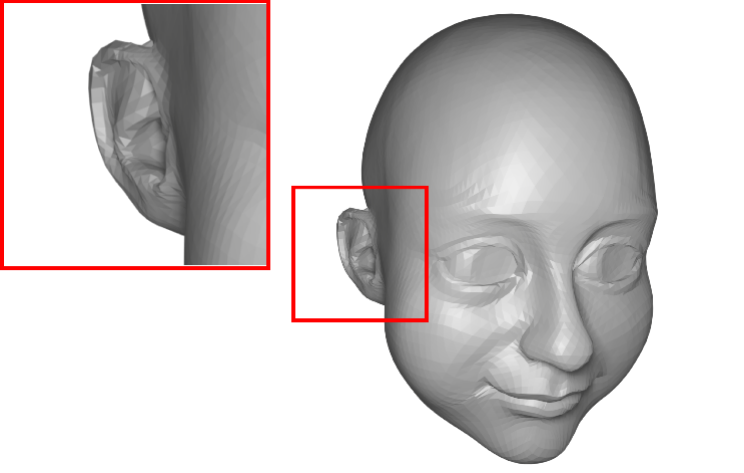}};

    \node[] (a) at (-8,-1.5) { GT};
    \node[] (a) at (-4,-1.5) { CD};
    \node[] (a) at (0,-1.5) { P2F};
    \node[] (a) at (4,-1.5) { MDA \cite{MDA}};
    \node[] (a) at (8,-1.5) { Ours};

    \end{tikzpicture}}
    \caption{Visual comparisons of reconstructed surfaces under different distance metrics.}
    \label{TEMPLATE:REC:FIG}
    \vspace{-0.3cm}
\end{figure*}

\begin{table}[h]
    \centering
    \vspace{-0.3cm}
    \caption{Quantitative comparisons of reconstructed surfaces under different distance metrics.}
     \renewcommand\arraystretch{1.0}
    \begin{tabularx}{0.95\linewidth}{{}X|>{\centering\arraybackslash}X|>{\centering\arraybackslash}X|>{\centering\arraybackslash}X}
    \toprule
    Method & NC$\uparrow$ & F-0.005$\uparrow$ & F-0.01$\uparrow$  \\
    \hline
    CD & 0.9871 & 0.9560 & 0.9930 \\  
    P2F &  0.9484 & 0.8586 & 0.9068 \\ 
    \hline
    MDA \cite{MDA} &  0.9934 & 0.6948 & 0.9092 \\ 
    Ours & \textbf{0.9939} & \textbf{0.9833} & \textbf{0.9989} \\
    \bottomrule
    \end{tabularx}
    \label{TEMPLATE:REC:TABLE}
    \vspace{-0.3cm}
\end{table}

\subsubsection{Comparisons}
We chose two baseline distance metrics for comparison, i.e., CD and P2F. 
For a fair comparison, the numbers of sampled points for computing CD and P2F were kept the same as our DDM. 
Besides, we also made a comparison with MDA \cite{MDA}, where a one-sided CD is used as the distance metric and the other settings were kept the same as their original paper.
We utilized \textit{Normal Consistency (NC)} and \textit{F-Score} with thresholds of 0.5\% and 1\%, denoted as \textit{F-0.005} and \textit{F-0.01}, as the evaluation metrics. We refer the readers to \cite{OCCNET} for the  detailed definitions. During the evaluation, 500K points were sampled on the deformed and target surfaces.
Table \ref{TEMPLATE:REC:TABLE} and Fig. \ref{TEMPLATE:REC:FIG} show the numerical and visual results, respectively, where both quantitative accuracy and visual quality of deformed shapes with our DDM are much better than baseline methods. While MDA exhibits superior visual results compared to CD and P2F, its numerical results about F-0.005 and F-0.01 in Table \ref{TEMPLATE:REC:TABLE} is worse than both. This discrepancy arises from the fact that the deformed surfaces produced by MDA are excessively smoothed, leading to a convergence to local optima, because of its loss function. Although the overall shape is similar, this oversmoothing effect is evident in the error maps illustrated in Fig. \ref{ERROR:MAP:SURF:REC}. \revise{Table \ref{TEMPLATE:REC:TABLE:TIME} presents the running time and GPU memory costs of different distance metrics per iteration. Notably, our DDM incurs a computational cost similar to P2F but higher than CD. This is because both DDM and P2F require identifying the closest points on triangle meshes, a more complex task compared to CD, which only involves finding the nearest points in point clouds.}

\begin{figure}
    \centering
    \quad\quad\quad\quad\quad\quad
    \quad\quad\quad\quad\quad\quad
    \quad\quad\quad\quad\quad\quad
    \quad\includegraphics[width=0.1\textwidth]{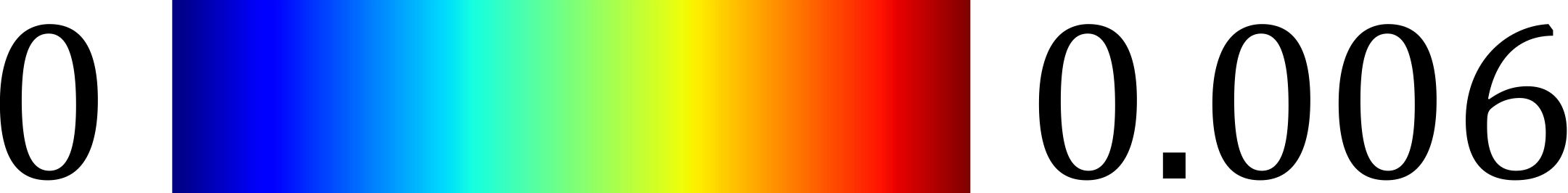}\\
    \includegraphics[width=0.1\textwidth]{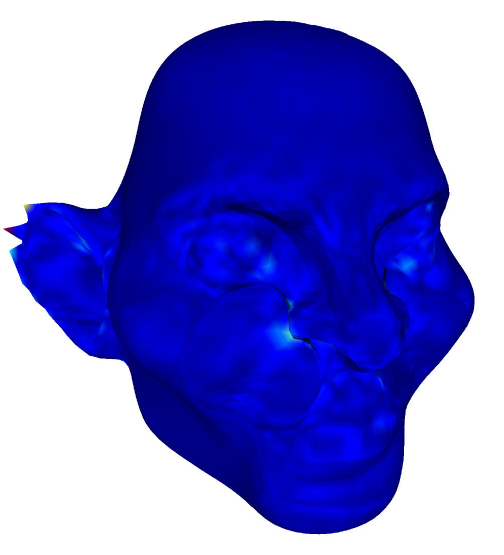}\quad
    \includegraphics[width=0.1\textwidth]{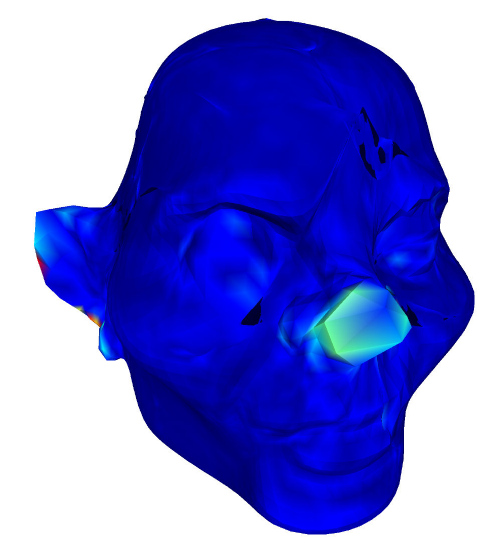} \quad
    \includegraphics[width=0.1\textwidth]{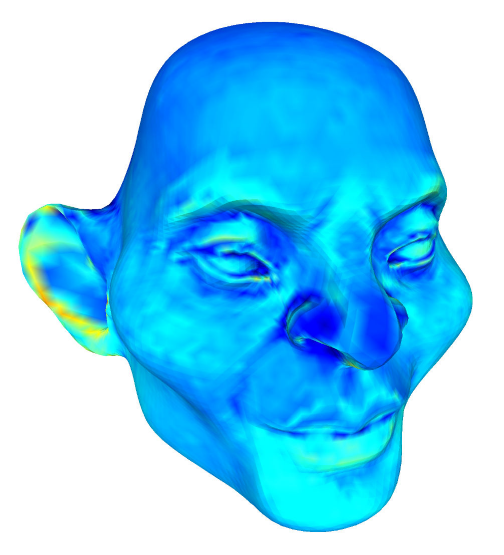}\quad
    \includegraphics[width=0.1\textwidth]{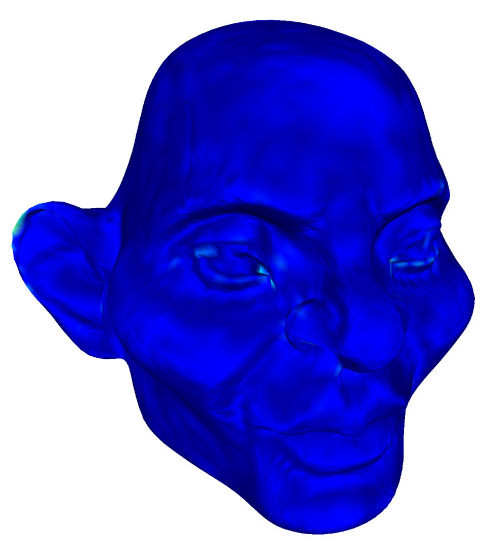}\\
 {\footnotesize CD \quad\quad\quad\quad\quad\quad P2F \quad\quad\quad\quad\quad\quad MDA \cite{MDA} \quad\quad\quad\quad Ours}
 \caption{Error map of the reconstructed surfaces under different distance metrics.}\label{ERROR:MAP:SURF:REC}
 \vspace{-0.8cm}
\end{figure}

\begin{table}[h]
    \centering
    \caption{\revise{Running time and GPU memory costs of different distance metrics in the template surface fitting task.}}
    \renewcommand\arraystretch{1.0}
    \revise{\begin{tabularx}{1\linewidth}{>{\hsize=1.8\hsize}X|>{\centering\arraybackslash\hsize=0.7\hsize}X>{\centering\arraybackslash\hsize=0.7\hsize}X>{\centering\arraybackslash\hsize=0.7\hsize}X}
    \toprule
        Method & CD & P2F & Ours \\
    \hline
        Running Time (ms/Iter.) & 49 &161 & 159\\
        GPU Memory (MB) & 873 & 891 & 893 \\
    \bottomrule
    \end{tabularx}}
    \vspace{-0.5cm}
    \label{TEMPLATE:REC:TABLE:TIME}
\end{table}

\begin{table*}[h]  
\small
\centering
\caption{Quantitative comparisons of rigid registration on the 3DMatch dataset \cite{3DMATCH}, where the coarse transformation is estimated through MAC \cite{MAC}.
 }\label{REGISTRATION:TABLE} \label{RIGID:REGISTRATION:TABLE}
 \renewcommand\arraystretch{1.0}
\begin{tabular}{l|c c c|c c c|c c c}
\toprule %
  \multirow{2}{*}{Method} & \multicolumn{3}{c|}{Clean} & \multicolumn{3}{c|}{Noise (2\%)} & \multicolumn{3}{c}{Outlier (50\%)} \\
\cline{2-10}
  ~ & SR ${(\%)}$ $\uparrow$ & ${\rm RE}$ ($^\circ$) $\downarrow$ & TE (cm) $\downarrow$  & SR ${(\%)}$ $\uparrow$ & ${\rm RE}$ ($^\circ$) $\downarrow$ & TE (cm) $\downarrow$ & SR ${(\%})$ $\uparrow$ & ${\rm RE}$ ($^\circ$) $\downarrow$ & TE (cm) $\downarrow$    \\
\hline
MAC \cite{MAC}      & 82.58 & 2.266 & 7.156 & 66.54 & 3.440 & 10.630 & \textbf{80.50} & 2.432 & 7.749 \\ 
\hline
ICP \cite{ICP}      & 35.25 & 5.768 & 15.361 & 30.49 & 5.909 & 15.864 & 15.85 & 6.604 & 15.951 \\ 
SICP \cite{SICP}    & 55.88 & \textbf{1.376} & \textbf{5.492} & 49.22 & 1.761 & 6.181 & 20.78 & 2.362 & 6.793 \\ 
\hline
EMD \cite{EMD}      & 14.86 & 8.782 & 19.317 & 14.23 & 8.547 & 20.081& 17.76 & 8.538 & 19.597 \\ 
CD \cite{CD}        & 25.65 & 6.889 & 16.368 & 22.79 & 6.760 & 16.597 & 11.10 & 8.298 & 18.700 \\ 
BCD \cite{BCD}      & 70.38 & 2.639 & 9.128 & 58.41 & 3.066 & 10.031 & 24.92 & 3.915 & 11.671 \\ 
ARL \cite{ARL}      & 21.87 & 3.194 & 20.403 & 16.69 & 3.373 & 20.456 & 14.00 & 9.764 & 15.087 \\ 
\hline 
\textbf{Ours}       & \textbf{82.77} &  1.485 & 6.098 & \textbf{69.87} & \textbf{1.469} & \textbf{5.977} & {78.22} & \textbf{1.581} & \textbf{6.665} \\ 
\bottomrule %
\end{tabular}
\vspace{-0.5cm}
\end{table*} 

\begin{figure*}[!h]
\centering
\resizebox{0.95\linewidth}{!}{
\begin{tikzpicture}[]
\node[] (a) at (0,2.3) {\includegraphics[width=0.2\linewidth]{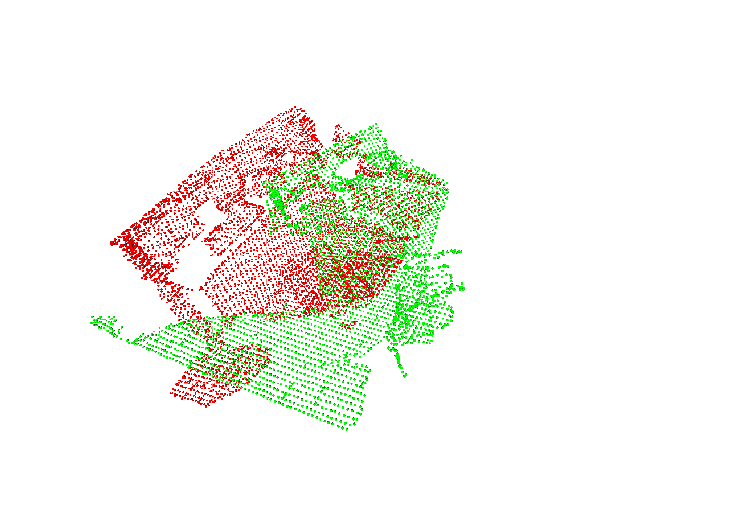} };
\node[] (a) at (16/9*2,2.3) {\includegraphics[width=0.2\linewidth]{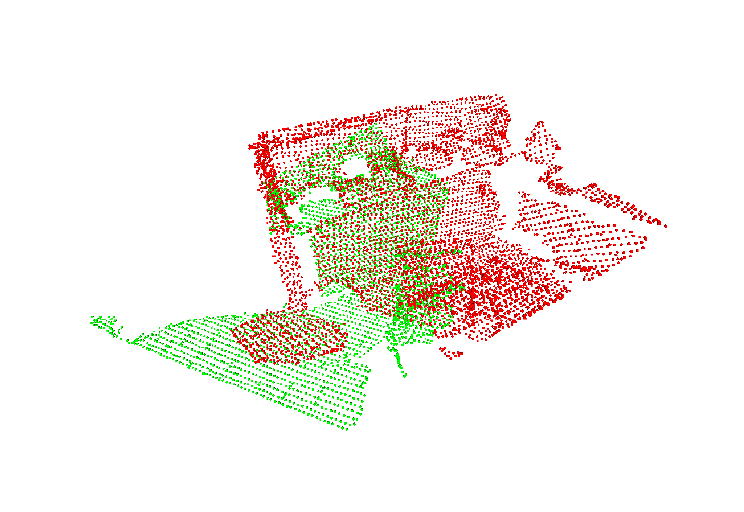} };
\node[] (a) at (16/9*4,2.3) {\includegraphics[width=0.2\linewidth]{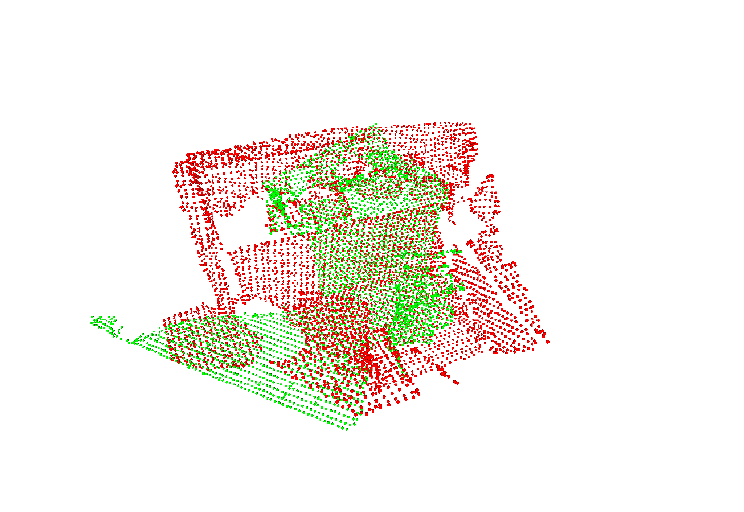} };
\node[] (a) at (16/9*6,2.3) {\includegraphics[width=0.2\linewidth]{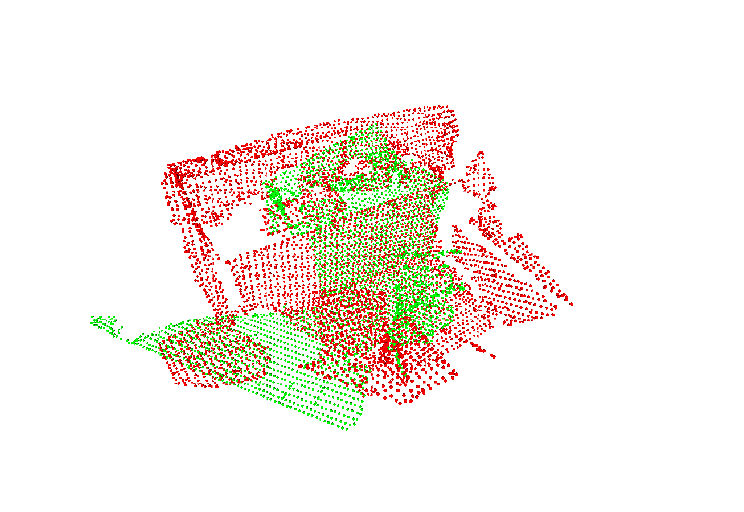} };
\node[] (a) at (16/9*8,2.3) {\includegraphics[width=0.2\linewidth]{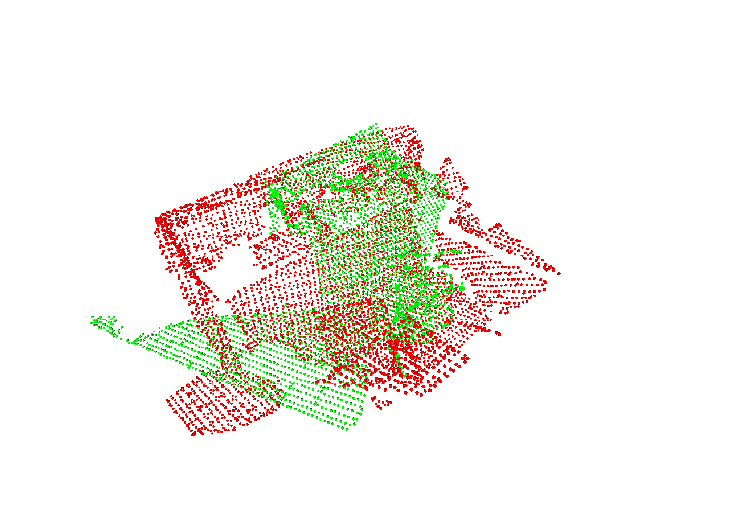} };

\node[] (a) at (0,0) {\includegraphics[width=0.2\linewidth]{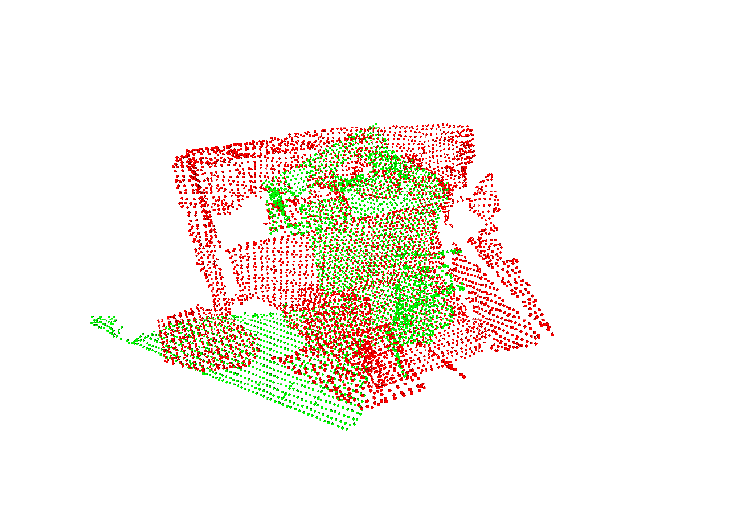} };
\node[] (a) at (16/9*2,0) {\includegraphics[width=0.2\linewidth]{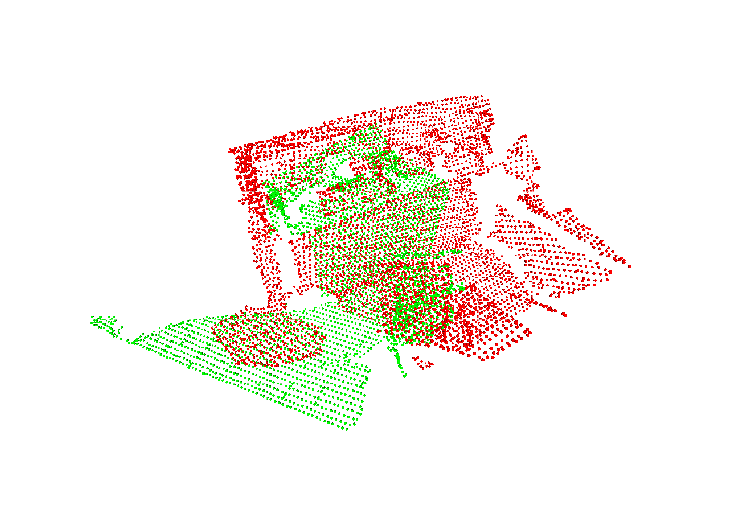} };
\node[] (a) at (16/9*4,0) {\includegraphics[width=0.2\linewidth]{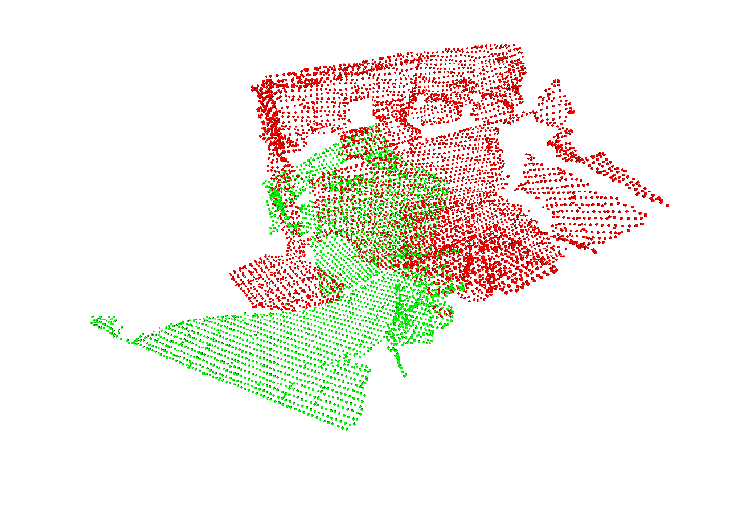} };
\node[] (a) at (16/9*6,0) {\includegraphics[width=0.2\linewidth]{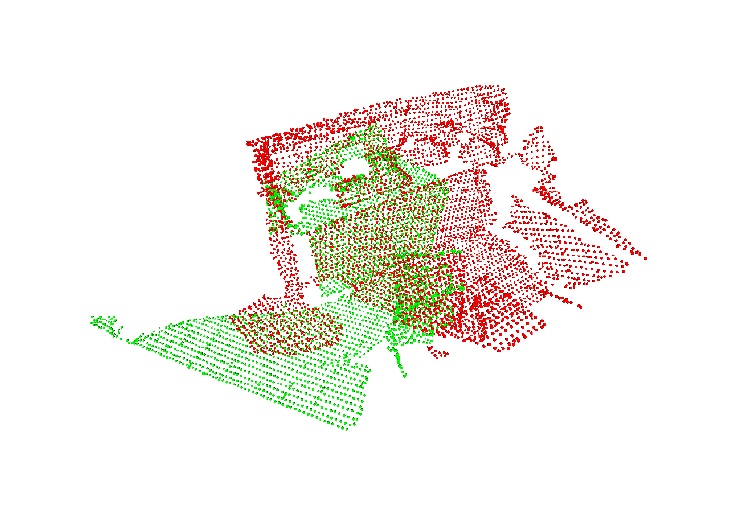} };
\node[] (a) at (16/9*8,0) {\includegraphics[width=0.2\linewidth]{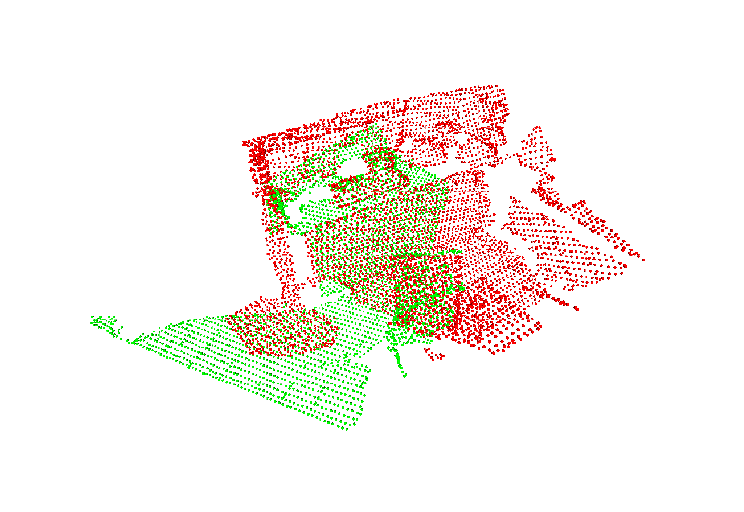} };

\node[] (a) at (0,-2.3) {\includegraphics[width=0.2\linewidth]{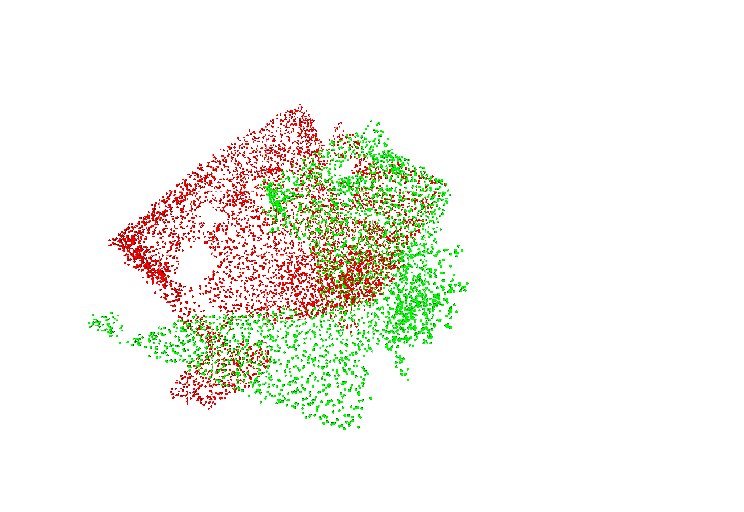} };
\node[] (a) at (16/9*2,-2.3) {\includegraphics[width=0.2\linewidth]{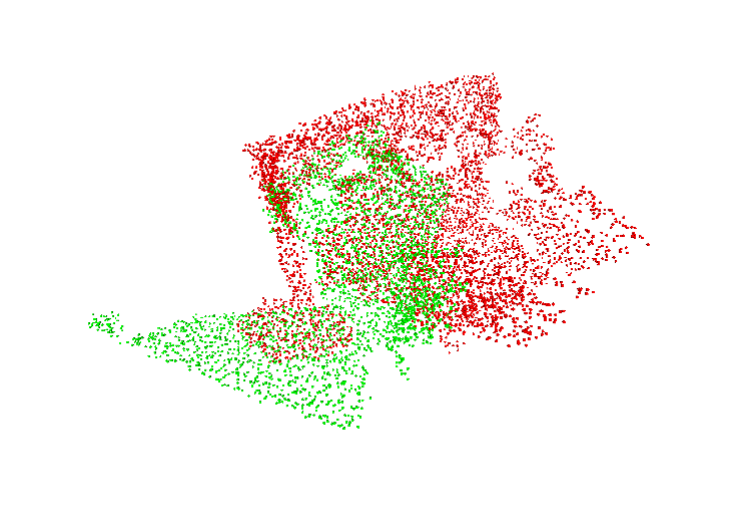} };
\node[] (a) at (16/9*4,-2.3) {\includegraphics[width=0.2\linewidth]{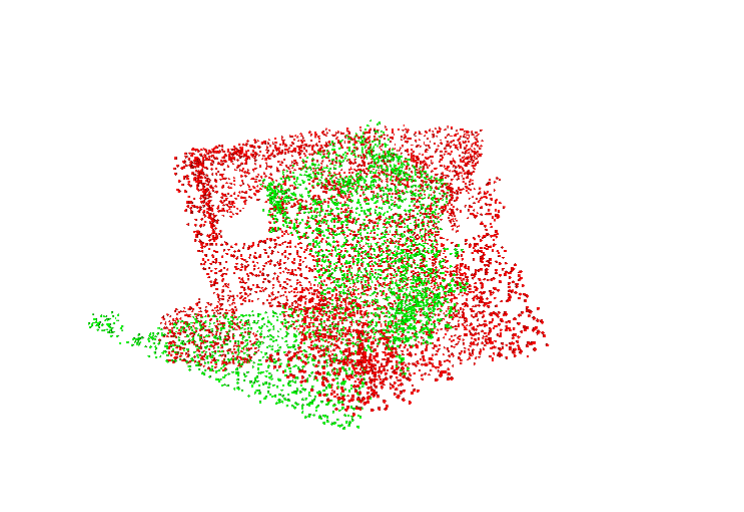} };
\node[] (a) at (16/9*6,-2.3) {\includegraphics[width=0.2\linewidth]{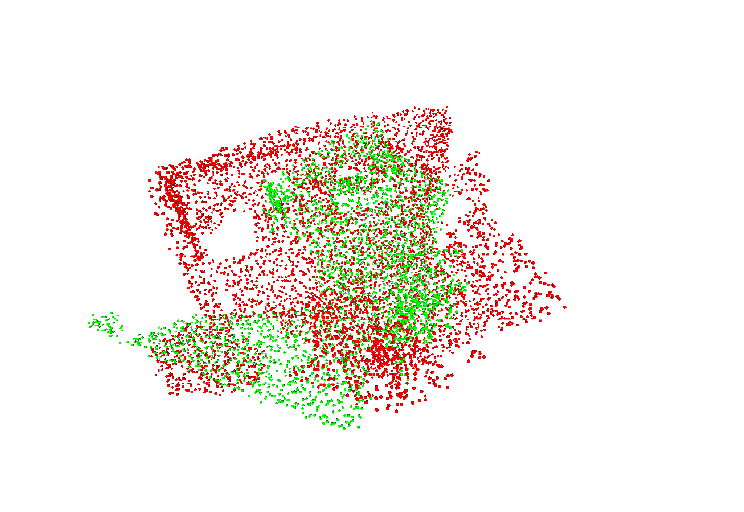} };
\node[] (a) at (16/9*8,-2.3) {\includegraphics[width=0.2\linewidth]{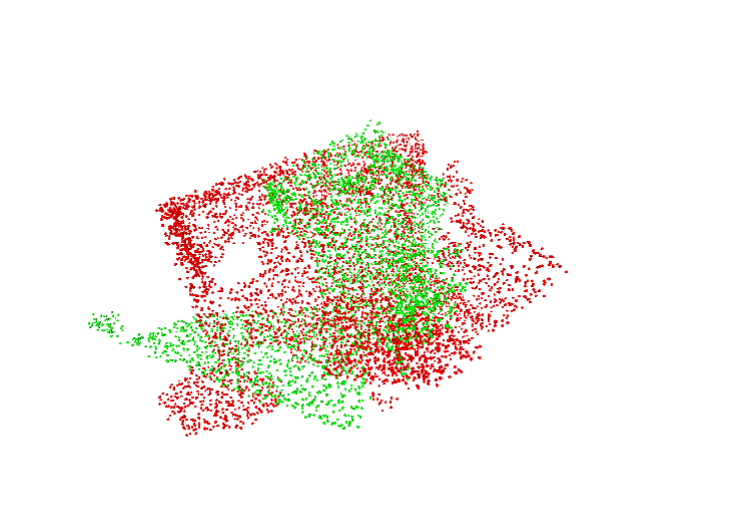} };

\node[] (a) at (0,-4.6) {\includegraphics[width=0.2\linewidth]{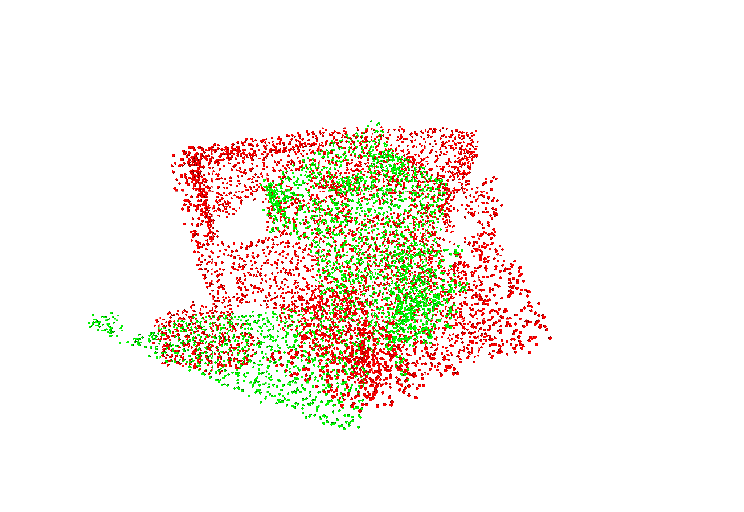} };
\node[] (a) at (16/9*2,-4.6) {\includegraphics[width=0.2\linewidth]{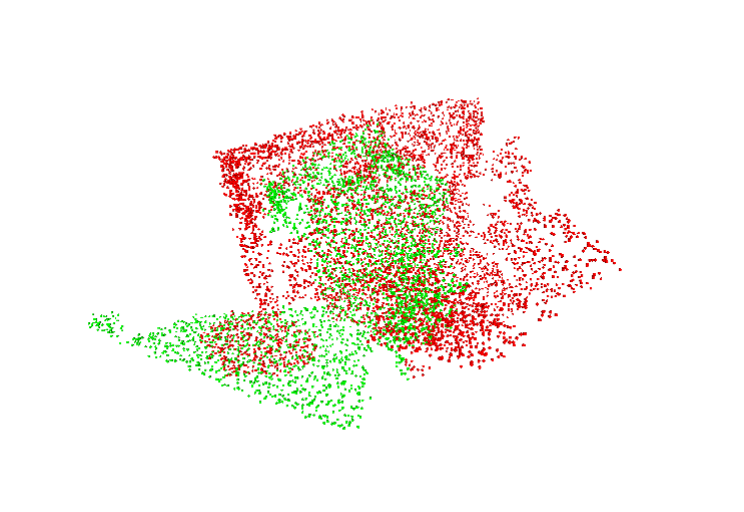} };
\node[] (a) at (16/9*4,-4.6) {\includegraphics[width=0.2\linewidth]{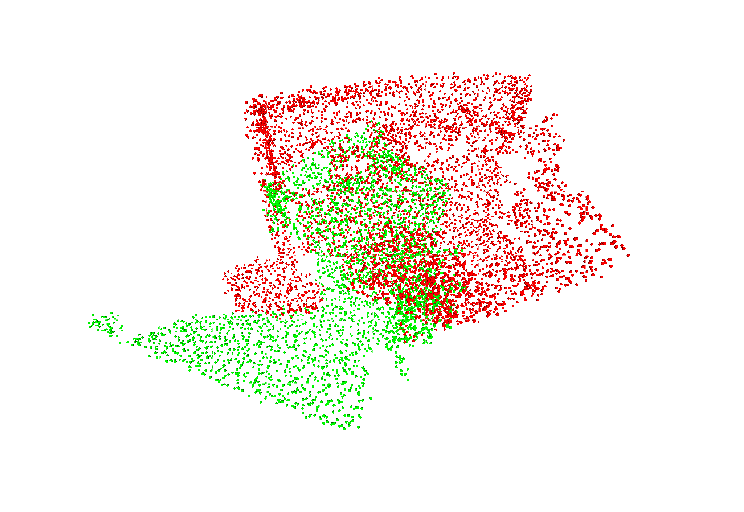} };
\node[] (a) at (16/9*6,-4.6) {\includegraphics[width=0.2\linewidth]{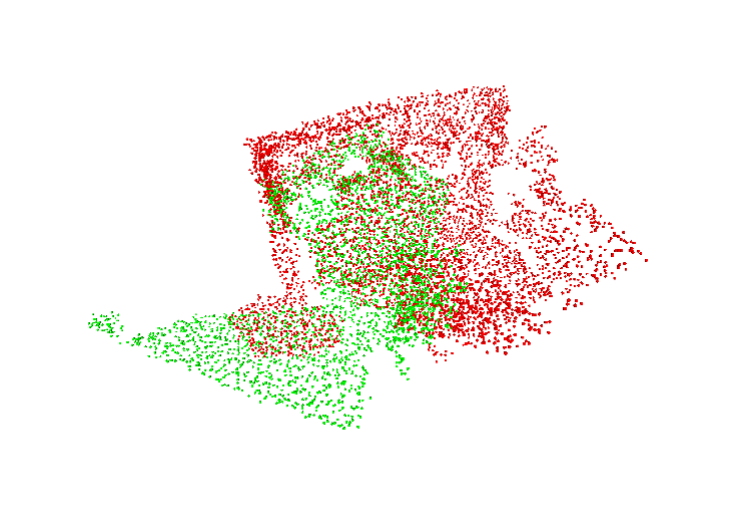} };
\node[] (a) at (16/9*8,-4.6) {\includegraphics[width=0.2\linewidth]{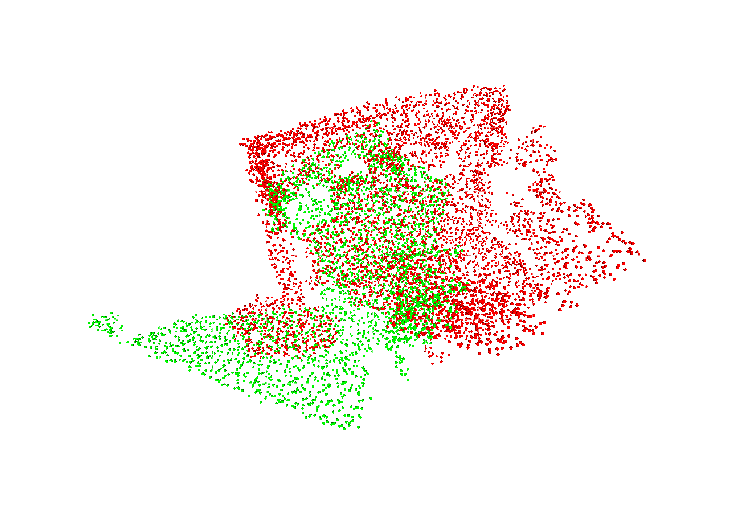} };

\node[] (a) at (0,-7.2) {\includegraphics[width=0.2\linewidth]{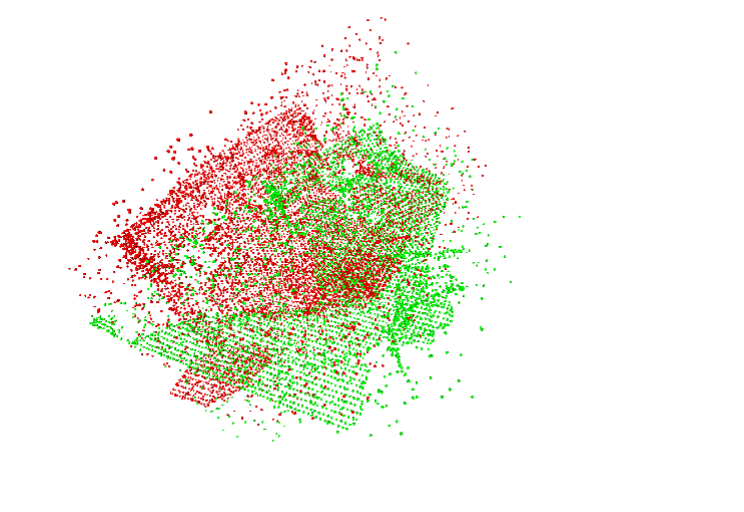} };
\node[] (a) at (16/9*2,-7.2) {\includegraphics[width=0.2\linewidth]{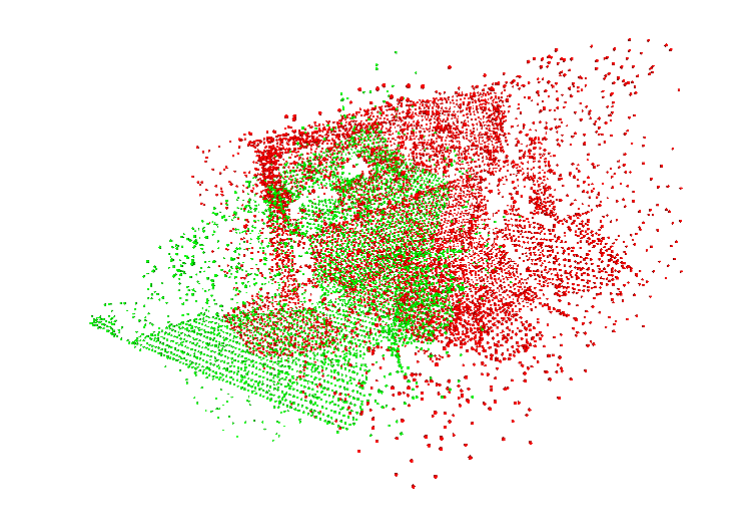} };
\node[] (a) at (16/9*4,-7.2) {\includegraphics[width=0.2\linewidth]{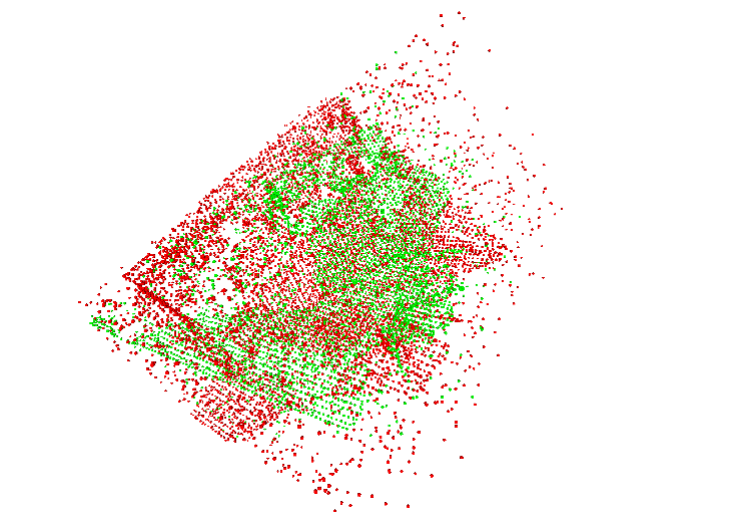} };
\node[] (a) at (16/9*6,-7.2) {\includegraphics[width=0.2\linewidth]{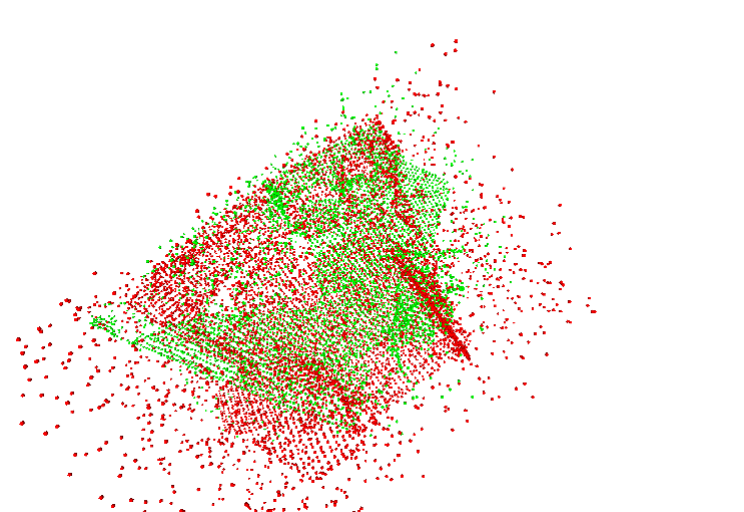} };
\node[] (a) at (16/9*8,-7.2) {\includegraphics[width=0.2\linewidth]{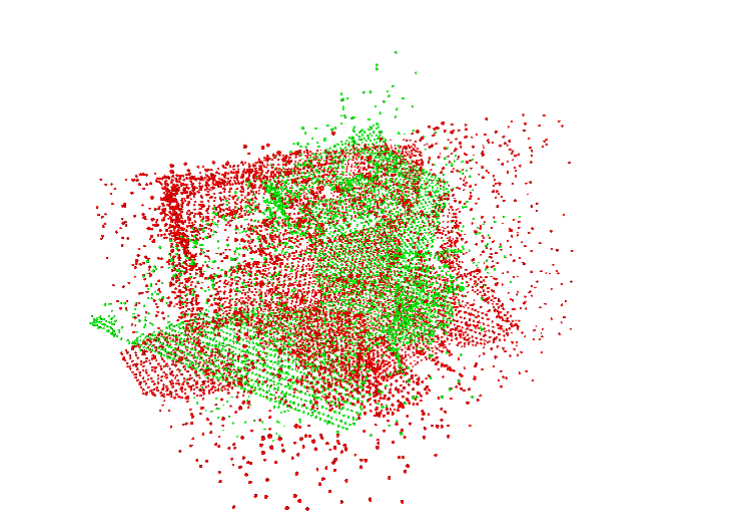} };

\node[] (a) at (0,-9.9) {\includegraphics[width=0.2\linewidth]{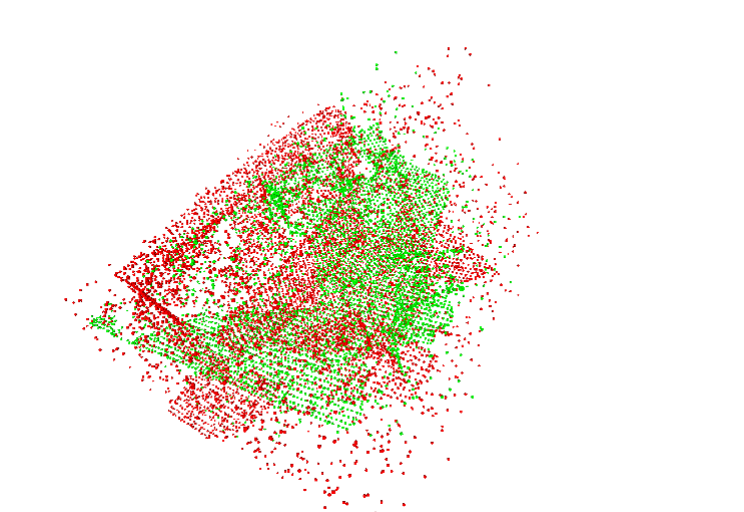} };
\node[] (a) at (16/9*2,-9.9) {\includegraphics[width=0.2\linewidth]{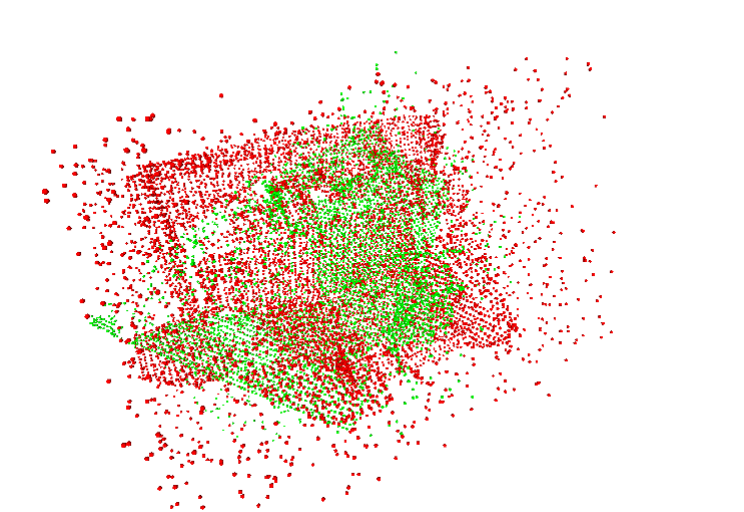} };
\node[] (a) at (16/9*4,-9.9) {\includegraphics[width=0.2\linewidth]{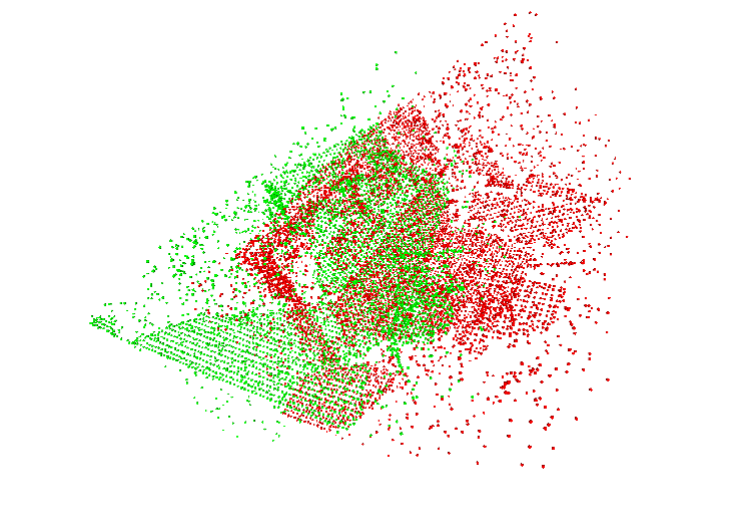} };
\node[] (a) at (16/9*6,-9.9) {\includegraphics[width=0.2\linewidth]{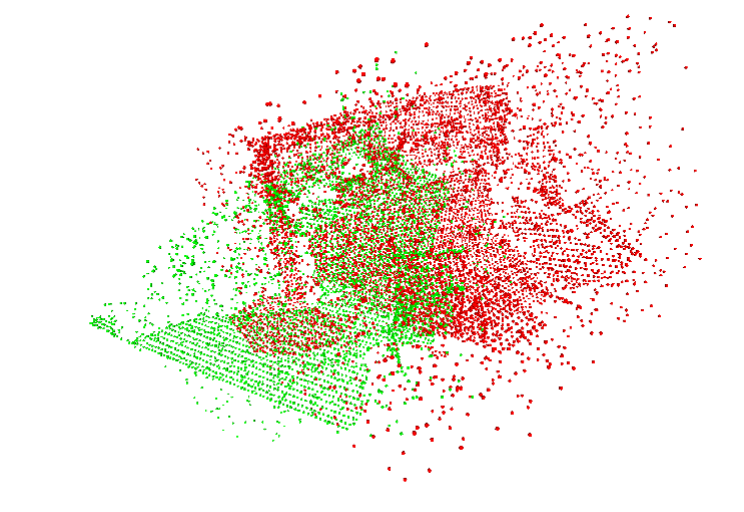} };
\node[] (a) at (16/9*8,-9.9) {\includegraphics[width=0.2\linewidth]{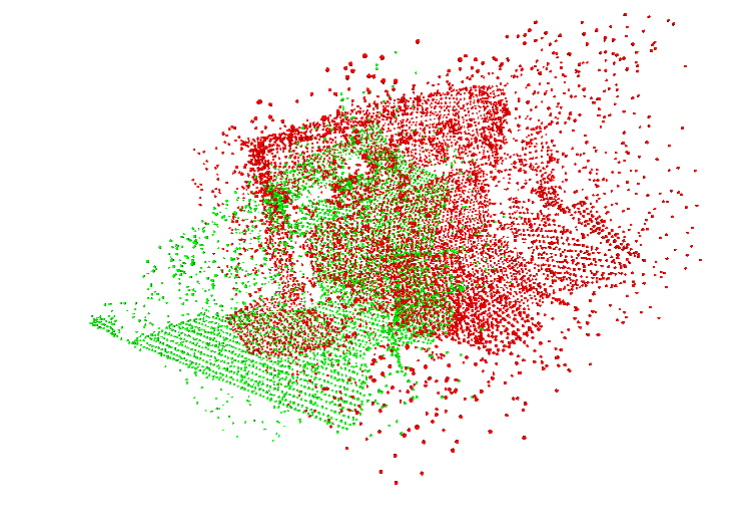} };

\node[] (a) at (0,1.2) {\footnotesize Initial };
\node[] (a) at (32/9,1.2) {\footnotesize MAC };
\node[] (a) at (64/9,1.2) {\footnotesize MAC + ICP };
\node[] (a) at (96/9,1.2) {\footnotesize MAC + SICP };
\node[] (a) at (128/9,1.2) {\footnotesize MAC + EMD };

\node[] (a) at (0,-1.1) {\footnotesize MAC + CD };
\node[] (a) at (16/9*2,-1.1) {\footnotesize MAC + BCD };
\node[] (a) at (16/9*4,-1.1) {\footnotesize MAC + ARL };
\node[] (a) at (16/9*6,-1.1) {\footnotesize MAC + Ours };
\node[] (a) at (16/9*8,-1.1) {\footnotesize GT };
\draw [dashed] (-16/9,-1.4) -- (16,-1.4);

\node[] (a) at (0,-3.4) {\footnotesize Initial };
\node[] (a) at (32/9,-3.4) {\footnotesize MAC };
\node[] (a) at (64/9,-3.4) {\footnotesize MAC + ICP };
\node[] (a) at (96/9,-3.4) {\footnotesize MAC + SICP };
\node[] (a) at (128/9,-3.4) {\footnotesize MAC + EMD };

\node[] (a) at (0,-5.7) {\footnotesize MAC + CD };
\node[] (a) at (16/9*2,-5.7) {\footnotesize MAC + BCD };
\node[] (a) at (16/9*4,-5.7) {\footnotesize MAC + ARL };
\node[] (a) at (16/9*6,-5.7) {\footnotesize MAC + Ours };
\node[] (a) at (16/9*8,-5.7) {\footnotesize GT };
\draw [dashed] (-16/9,-6) -- (16,-6);

\node[] (a) at (0,-8.5) {\footnotesize Initial };
\node[] (a) at (32/9,-8.5) {\footnotesize MAC };
\node[] (a) at (64/9,-8.5) {\footnotesize MAC + ICP };
\node[] (a) at (96/9,-8.5) {\footnotesize MAC + SICP };
\node[] (a) at (128/9,-8.5) {\footnotesize MAC + EMD };

\node[] (a) at (0,-11.3) {\footnotesize MAC + CD };
\node[] (a) at (16/9*2,-11.3) {\footnotesize MAC + BCD };
\node[] (a) at (16/9*4,-11.3) {\footnotesize MAC + ARL };
\node[] (a) at (16/9*6,-11.3) {\footnotesize MAC + Ours };
\node[] (a) at (16/9*8,-11.3) {\footnotesize GT };

\end{tikzpicture}
}
\vspace{-0.3cm}
\caption{\small Visual comparisons of rigid registration results. The \textcolor{red}{red} and \textcolor{green}{green} points represent the source and transformed point clouds, respectively. From top to bottom: Clean, Noise, and Outlier. \color{magenta}{\faSearch~} Zoom in to see details.} \label{RIGID:REGISTRATION:FIG}
\vspace{-0.6cm}
\end{figure*}

\subsection{Rigid Registration of 3D Point Clouds} 

\subsubsection{Implementation Details} 
We experimented with the commonly used 3DMatch dataset \cite{3DMATCH}. 
Following previous works \cite{POINTDSC, DGR},  we downsampled the original point clouds 
through Voxel Grid Filtering with a resolution of 5cm, producing point clouds each with \revise{approximately} 5K points. 
Then we chose a feature-based registration methods, MAC \cite{MAC}, to calculate the coarse transformation between the two point clouds, serving 
as the initialization of the optimization of Eq. (\ref{RIGID:REGISTRATION}).
To test the robustness of different distance metrics, we also conducted experiments on the point clouds with noise and outliers.
To simulate the noisy conditions, we added Gaussian noise with a mean of zero and a standard deviation of 2cm to the clean point clouds. For the outlier data, we infused the point clouds with an additional 50\% randomly selected points. In all registration experiments, we set $M$ 10 times the number of points, $K=5$, $\beta=20$, and 
$\sigma=0.05$. 
We employed the Adam optimizer to optimize, 
spanning 200 iterations with a learning rate set at 0.02.

\subsubsection{Comparisons} We compared our DDM with EMD \cite{EMD}, CD \cite{CD}, DCD \cite{BCD}, and ARL \cite{ARL}. 
Additionally, we made comparisons with two widely used registration methods: ICP \cite{ICP} and SICP \cite{SICP}. 
We employed two widely used evaluation metrics in the rigid registration task, \textit{Rotation Error (RE)} and \textit{Translation Error (TE)}, to measure the registration accuracy, respectively defined as  
\begin{align}
    &\texttt{RE}(\mathbf{\hat{R}},\mathbf{R}_{\rm GT})={\rm arccos}\left(\frac{\texttt{Tr}(\mathbf{R}_{\rm GT}^T\mathbf{\hat{R}})-1}{2}\right) \\
    &\texttt{TE}(\mathbf{\hat{t}},\mathbf{t}_{\rm GT})=\|\mathbf{\hat{t}}-\mathbf{t}_{\rm GT}\|_2,
\end{align}
where $[\mathbf{\hat{R}},~\mathbf{\hat{t}}]$ and $[\mathbf{R}_{\rm GT},\mathbf{t}_{\rm GT}]$ are the estimated and ground-truth transformations, respectively. 
Besides, we also used \textit{Successful Rate (SR)} to evaluate the performance of different methods, where the results with $\texttt{RE}<15^\circ$ and $\texttt{TE} < 30$cm are considered successful. It is worth noting that we only concentrated on these successful results when calculating mean RE and TE. 

The numerical results are listed in Table \ref{RIGID:REGISTRATION:TABLE}, where the refinement through our DDM further improves the accruacy of MAC significantly, 
while after the refinement with other distance metric, the registration accuracy even decrease dynamically. Although the RE and TE of SICP are slightly better 
than ours, its SR is much lower than ours, showing its limitation. The visual results are shown in Fig. \ref{RIGID:REGISTRATION:FIG}, and when dealing with the data with noise or outliers, the details of the point clouds become indistinct, our method still register the two point clouds successfully.
To evaluate the performance of different methods insightly, we also used the \textit{Registration Recall (RR)} under different \textit{thresholds} of RE and TE to measure the registration accuracy. As shown in Fig. \ref{REGISTRATION:RECALL}, the superiority of our DDM is further verified. 
\revise{Additionally, Table \ref{RIGID:REG:TABLE:TIME} presents the computational cost per iteration during the optimization process. As shown, EMD incurs the highest computational overhead due to the necessity of constructing a comprehensive bi-directional mapping between two point sets. Similarly, ARL requires significant computational resources, as it involves calculating the intersection of lines with the two point clouds. In contrast, CD, BCD, and our DDM exhibit comparable computational costs. However, DDM achieves superior registration results compared to these methods, demonstrating its effectiveness and efficiency in rigid point cloud registration.}

\begin{figure}[h]
\centering
\vspace{-0.5cm}
    \subfloat[Clean]{\includegraphics[width=0.25\textwidth]{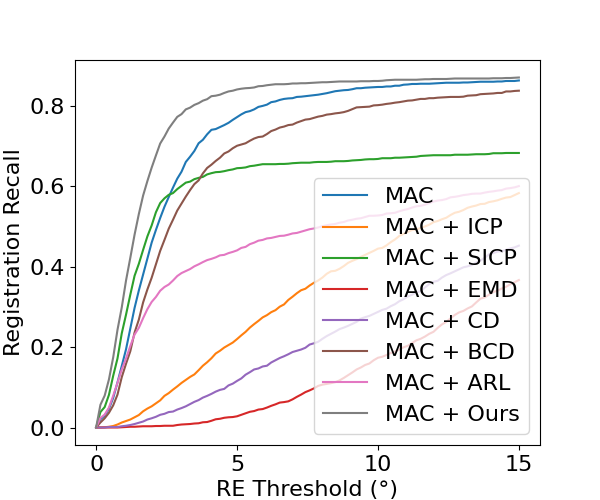}\includegraphics[width=0.25\textwidth]{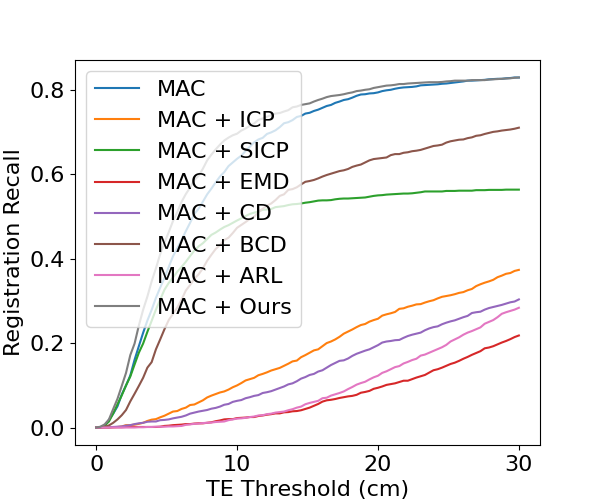}\vspace{-0.2cm}} \\ 
    \subfloat[Noise]{\includegraphics[width=0.25\textwidth]{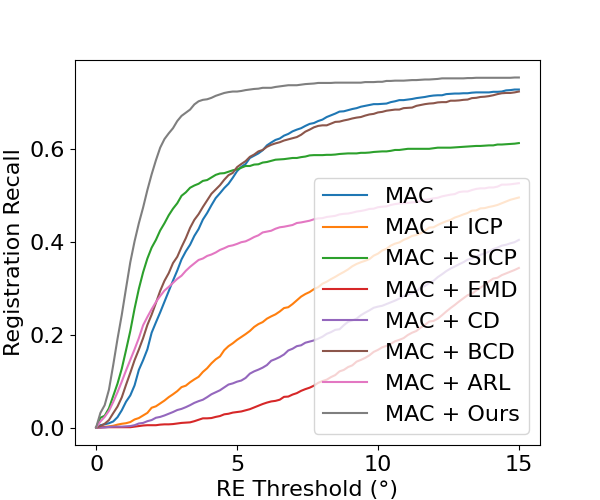}\includegraphics[width=0.25\textwidth]{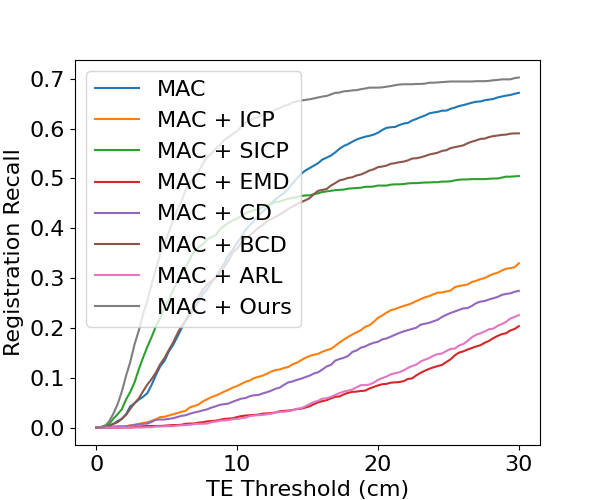}\vspace{-0.2cm}} \\ 
    \subfloat[Outlier]{\includegraphics[width=0.25\textwidth]{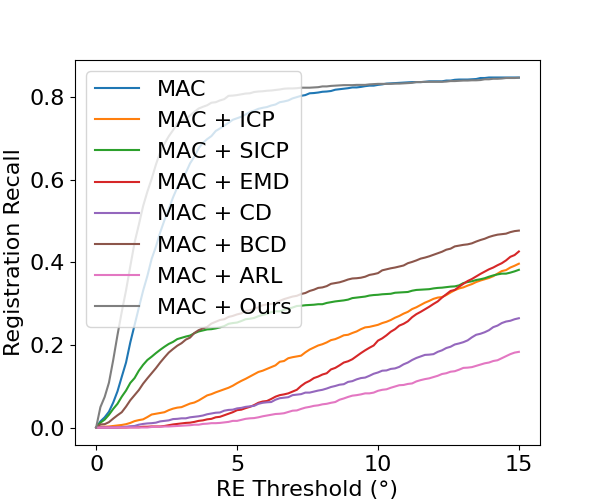}\includegraphics[width=0.25\textwidth]{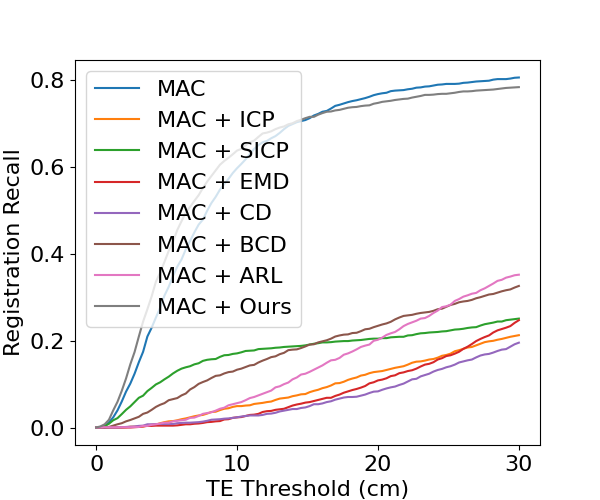}\vspace{-0.2cm}}
    \caption{ Registration recall with different RE and TE thresholds on the 3DMatch dataset \cite{3DMATCH}.}\label{REG:RECALL:THRESHOLD} \label{REGISTRATION:RECALL}
\end{figure}

\begin{table}[h]
    \centering
    \caption{Running time and GPU memory costs of different distance metrics in the rigid registration task.}
    \renewcommand\arraystretch{1.3}
    \resizebox{1\linewidth}{!}{\begin{tabular}{l|ccccc}
    \toprule
        Method & EMD & CD & BCD \cite{BCD}  &  ARL \cite{ARL} &  Ours  \\
    \hline
        Running Time (ms/Iter.) & 559  &20 & 21 &145 & 22\\
        GPU Memory (MB) &1977  & 1721 & 1730 &5743 & 1745\\
    \bottomrule
    \end{tabular}}
    \label{RIGID:REG:TABLE:TIME}
    \vspace{-0.3cm}
\end{table}

\subsection{Non-Rigid 3D Shape Registration}

\subsubsection{Implementation Details} 
Following the recent work named AMM \cite{AMM}, We employed four sequences ('handstand,' 'crane,' 'march1,' and 'swing') from the human motion dataset \cite{AMA} to evaluate the performance of various distance metrics. 
We aligned the $i$-th frame with the $(i+2)$-th frame for each sequence, where $i$ ranges from 10 to 60 to filter out frames with small motion. The distance threshold $\epsilon$ was set to 5 times the average length of edges of the source surface. We generated $M=4\times 10^4$ reference points with $\sigma=0.1$, 
and specified the $K$-NN size for the deformation nodes 
as $K=5$. Throughout the entire optimization process, the weight $\lambda$ for balancing different terms in Eq. \eqref{eq:non-rigid} 
was set to $500$. We realized the optimization through the SGD optimizer, spanning 1000 iterations with a learning rate of 2.0.

\begin{table}[h]
    \centering \small
    \caption{
    Quantitative comparison of non-rigid registration on the sequences from the human motion dataset \cite{AMA}.}
    \renewcommand\arraystretch{1.3}
    \resizebox{1\linewidth}{!}{
    \begin{tabular}{c|c c| c c }
    \toprule %
    \multirow{2}{*}{Sequence} & \multicolumn{4}{c}{RMSE $\downarrow$ $\pm$ STD $\downarrow$  ($\times 10^{-2}$)} \\
    \cline{2-5}
         & CD & P2F & AMM \cite{AMM} & Ours  \\
    \hline
    handstand & 3.571$\pm$1.968 & 5.646$\pm$4.744 & 1.331$\pm$1.112 &  \textbf{1.035$\pm$0.913}  \\
    crane & 2.919$\pm$1.326 & 3.871$\pm$1.866 &  1.308$\pm$1.608 &  \textbf{0.509$\pm$0.247}  \\
    march1 & 1.818$\pm$1.177 & 2.509$\pm$1.811 & 0.690$\pm$1.295 &  \textbf{0.259$\pm$0.176}  \\
    swing & 2.263$\pm$0.547  & 2.901$\pm$1.017 &  1.477$\pm$1.070 &  \textbf{0.648$\pm$0.302}  \\
    \bottomrule %
    \end{tabular}}
    \label{NON:RIGID:REG:TAB} 
\end{table}

\subsubsection{Comparisons} We compared our DDM with CD and the P2F distance. For a fair comparison, we kept the number of sampled points in CD and the P2F distance the same as the generated reference points of our DDM.  In addition, we also compared with the state-of-the-art optimization-based non-rigid registration method called AMM \cite{AMM}. We computed \textit{RMSE} of the estimated and ground truth vertices of the deformed source surfaces for quantitative evaluation. The results in Table \ref{NON:RIGID:REG:TAB} and Fig. \ref{NON:RIGID:REG:FIG} demonstrate the significant superiority of our DDM. Additionally, we show the registration errors of all pairs in a sequence in Fig. \ref{RMSE:CURVE} for a more comprehensive comparison, where it can be seen that the compared methods produce much larger RMSE values for pairs with significant motion, while our DDM always works well. \revise{Additionally, Table \ref{TEMPLATE:NONRIGID:TABLE:TIME} presents the computational time and GPU memory requirements per iteration during the optimization process. Specifically, CD achieves the fastest running speed, as it only involves identifying the nearest points within the point clouds. In contrast, P2F and our proposed DDM require more computational time due to the added complexity of locating the closest point on the triangle mesh, which is inherently more complex than finding the nearest points in a point cloud.
}

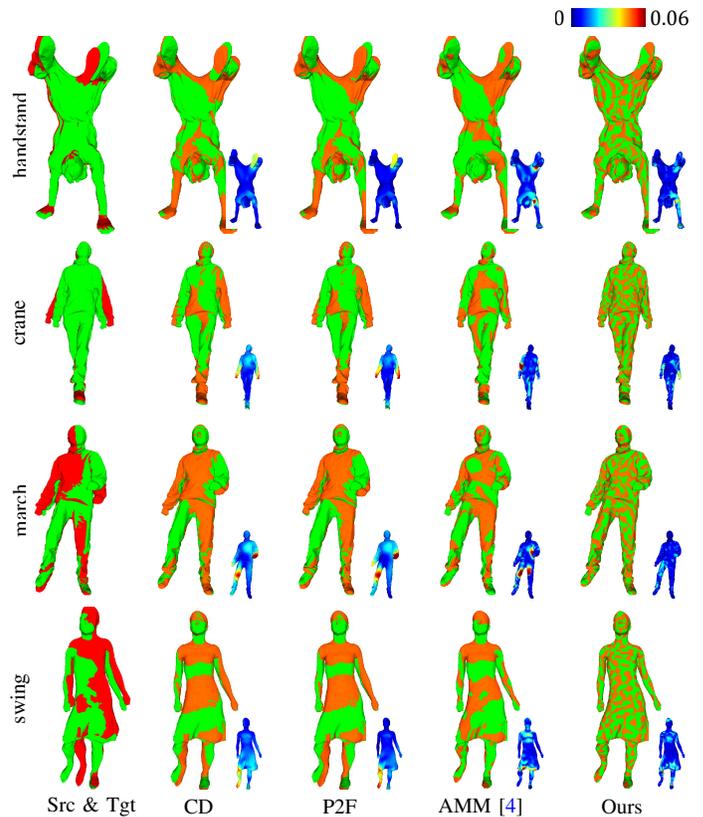
\begin{figure}[t]
\centering
{
\begin{tikzpicture}[]

\node[] (a) at (-0.5,0) { \rotatebox{90}{\footnotesize swing} };
\node[] (a) at (-0.5,4/1.6) { \rotatebox{90}{\footnotesize march} };
\node[] (a) at (-0.5,8/1.6) { \rotatebox{90}{\footnotesize crane} };
\node[] (a) at (-0.5,12/1.6) { \rotatebox{90}{\footnotesize handstand} };

\node[] (a) at (0.4,0) {\includegraphics[width=0.15\linewidth]{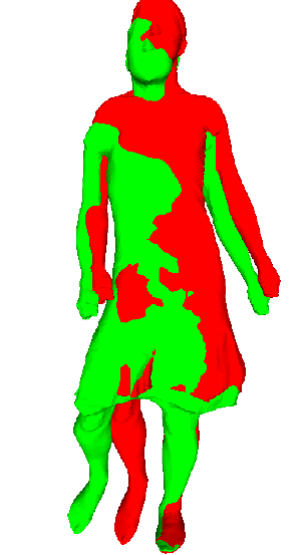} };
\node[] (a) at (3/1.6,0) {\includegraphics[width=0.15\linewidth]{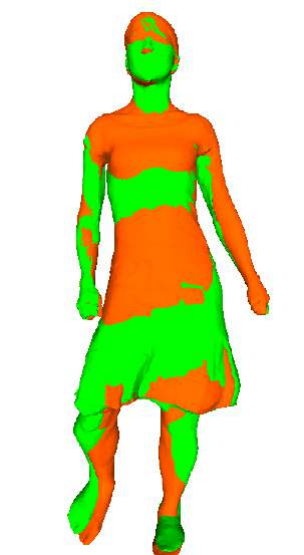} };
\node[] (a) at (6/1.6,0) {\includegraphics[width=0.15\linewidth]{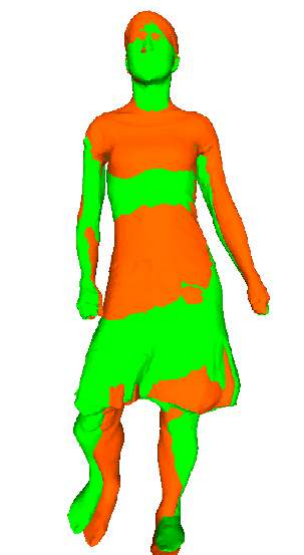} };
\node[] (a) at (9/1.6,0) {\includegraphics[width=0.15\linewidth]{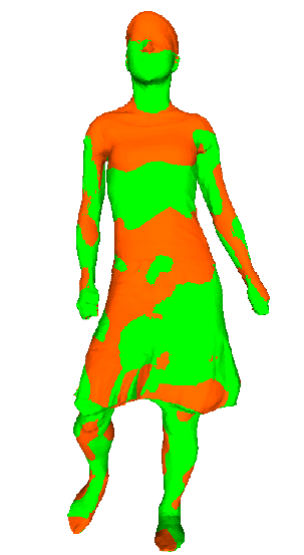} };
\node[] (a) at (12/1.6,0) {\includegraphics[width=0.15\linewidth]{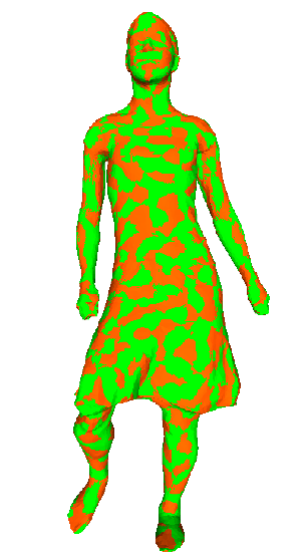} };

\node[] (a) at (4/1.6,-1/1.6-0.1) {\includegraphics[width=0.06\linewidth]{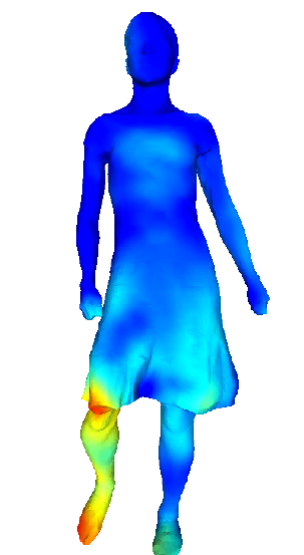} };
\node[] (a) at (7/1.6,-1/1.6-0.1) {\includegraphics[width=0.06\linewidth]{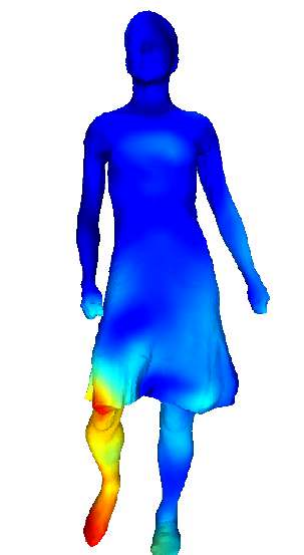} };
\node[] (a) at (10/1.6,-1/1.6-0.1) {\includegraphics[width=0.06\linewidth]{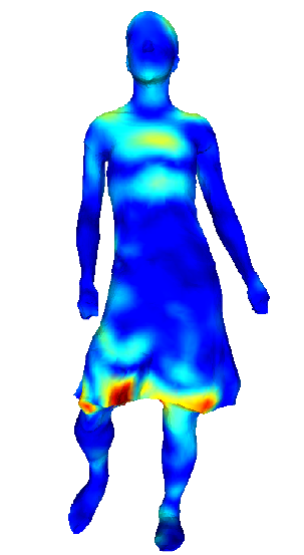} };
\node[] (a) at (13/1.6,-1/1.6-0.1) {\includegraphics[width=0.06\linewidth]{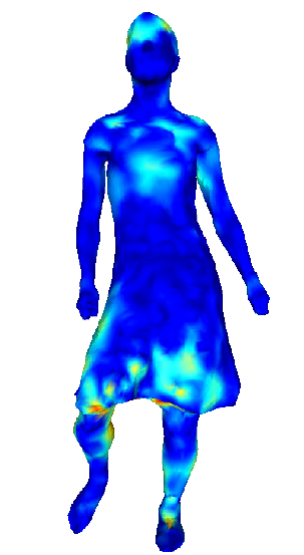} };

\node[] (a) at (0.4/1.6,4/1.6) {\includegraphics[width=0.15\linewidth]{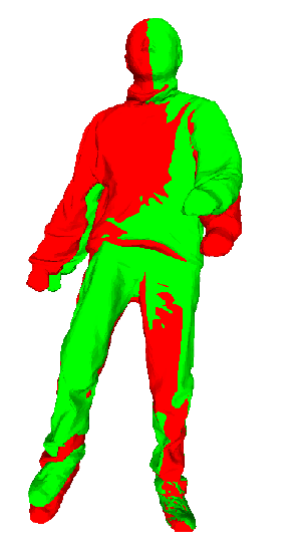} };
\node[] (a) at (3/1.6,4/1.6) {\includegraphics[width=0.15\linewidth]{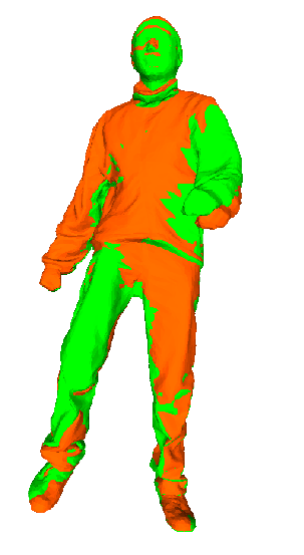} };
\node[] (a) at (6/1.6,4/1.6) {\includegraphics[width=0.15\linewidth]{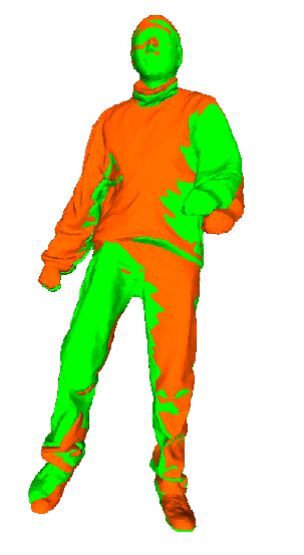} };
\node[] (a) at (9/1.6,4/1.6) {\includegraphics[width=0.15\linewidth]{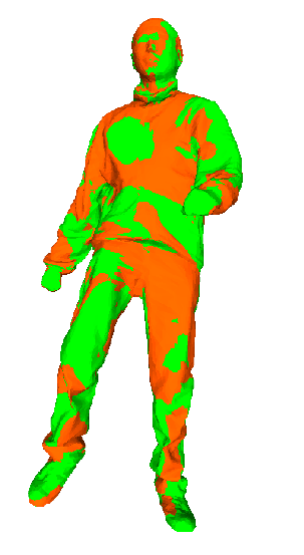} };
\node[] (a) at (12/1.6,4/1.6) {\includegraphics[width=0.15\linewidth]{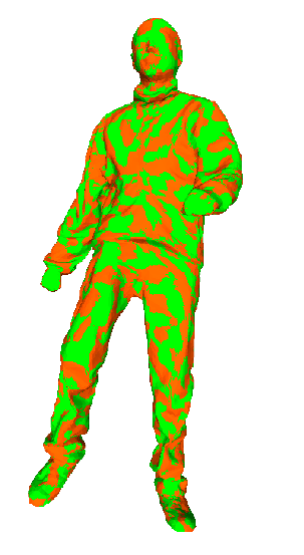} };

\node[] (a) at (4/1.6,3/1.6-0.1) {\includegraphics[width=0.06\linewidth]{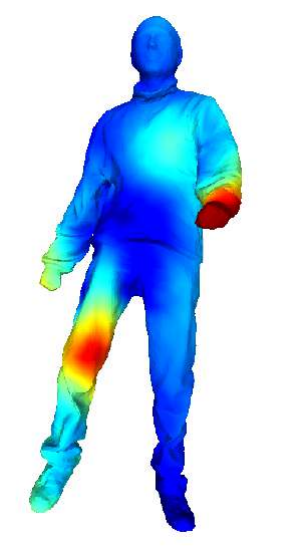} };
\node[] (a) at (7/1.6,3/1.6-0.1) {\includegraphics[width=0.06\linewidth]{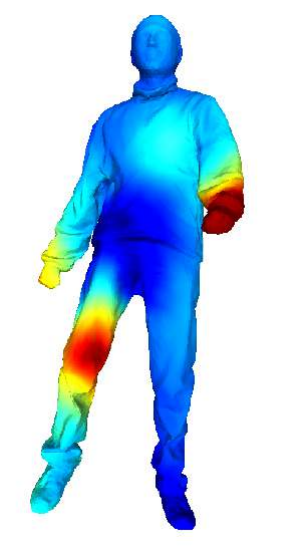} };
\node[] (a) at (10/1.6,3/1.6-0.1) {\includegraphics[width=0.06\linewidth]{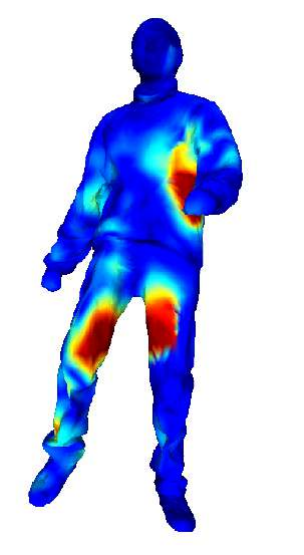} };
\node[] (a) at (13/1.6,3/1.6-0.1) {\includegraphics[width=0.06\linewidth]{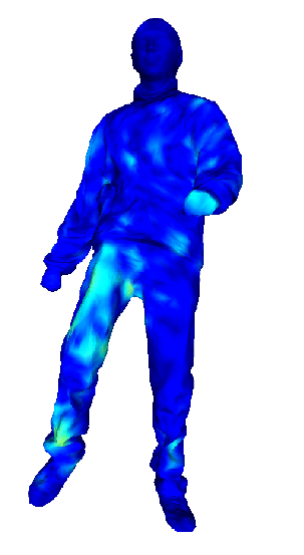} };

\node[] (a) at (0.4/1.6,8/1.6) {\includegraphics[width=0.15\linewidth]{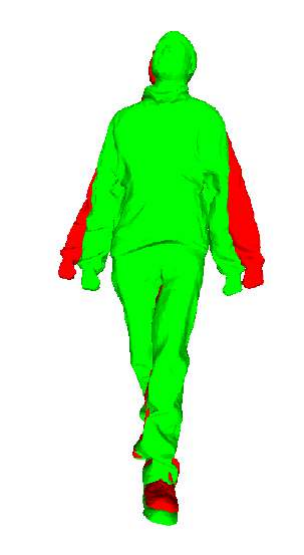} };
\node[] (a) at (3/1.6,8/1.6) {\includegraphics[width=0.15\linewidth]{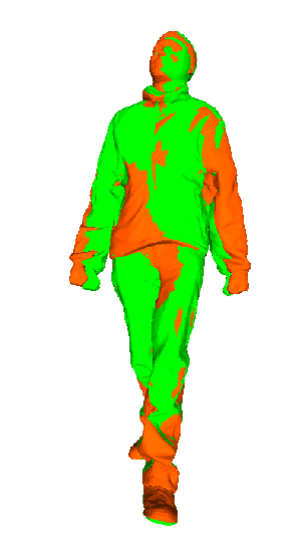} };
\node[] (a) at (6/1.6,8/1.6) {\includegraphics[width=0.15\linewidth]{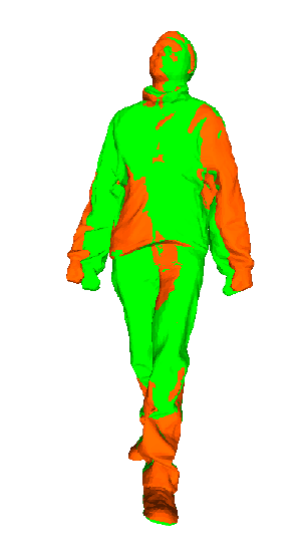} };
\node[] (a) at (9/1.6,8/1.6) {\includegraphics[width=0.15\linewidth]{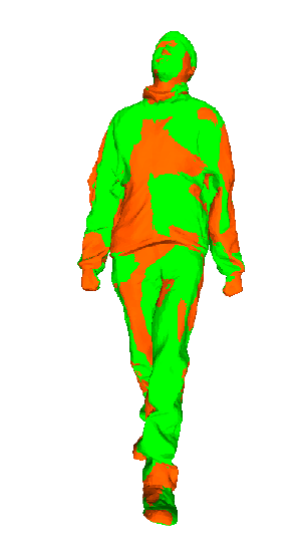} };
\node[] (a) at (12/1.6,8/1.6) {\includegraphics[width=0.15\linewidth]{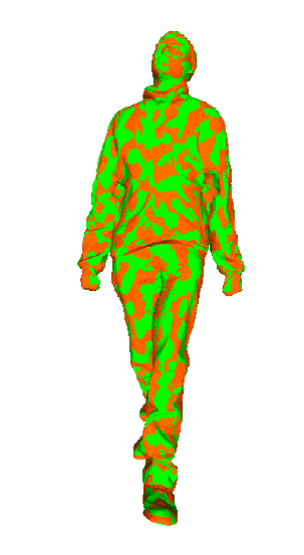} };

\node[] (a) at (4/1.6,7/1.6-0.1) {\includegraphics[width=0.06\linewidth]{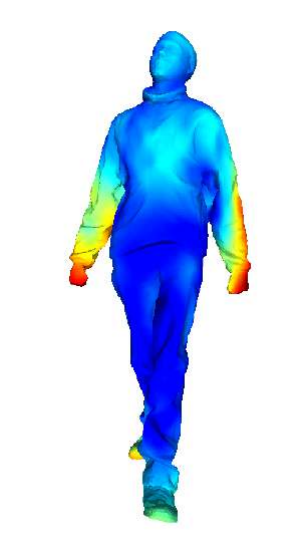} };
\node[] (a) at (7/1.6,7/1.6-0.1) {\includegraphics[width=0.06\linewidth]{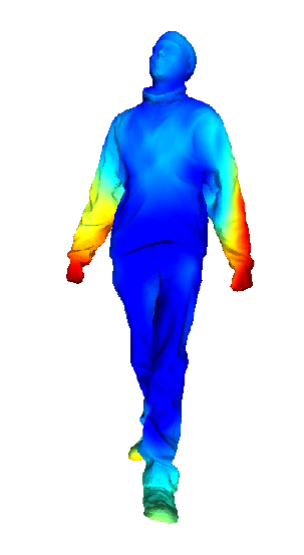} };
\node[] (a) at (10/1.6,7/1.6-0.1) {\includegraphics[width=0.06\linewidth]{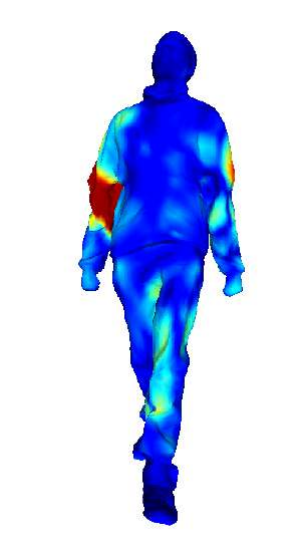} };
\node[] (a) at (13/1.6,7/1.6-0.1) {\includegraphics[width=0.06\linewidth]{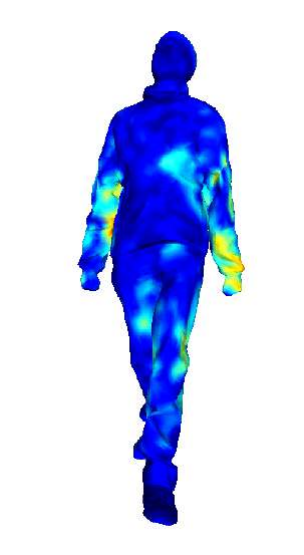} };

\node[] (a) at (0.4/1.6,12/1.6) {\includegraphics[width=0.15\linewidth]{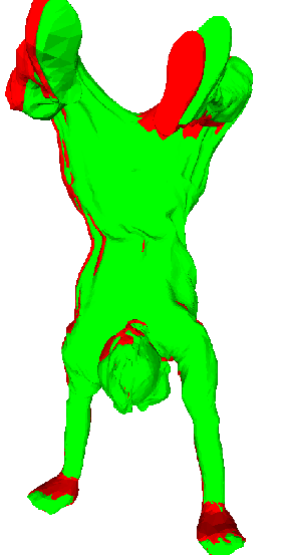} };
\node[] (a) at (3/1.6,12/1.6) {\includegraphics[width=0.15\linewidth]{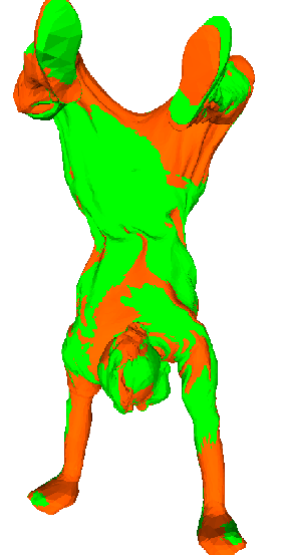} };
\node[] (a) at (6/1.6,12/1.6) {\includegraphics[width=0.15\linewidth]{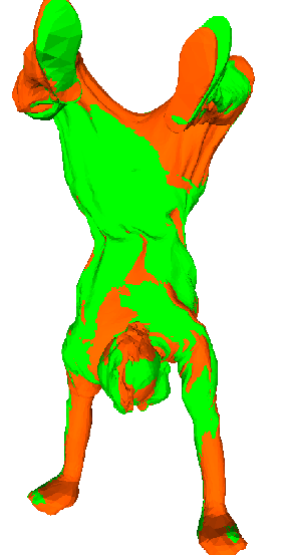} };
\node[] (a) at (9/1.6,12/1.6) {\includegraphics[width=0.15\linewidth]{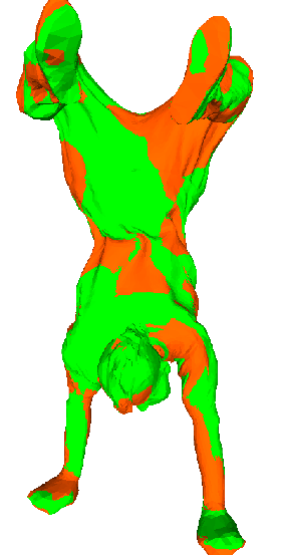} };
\node[] (a) at (12/1.6,12/1.6) {\includegraphics[width=0.15\linewidth]{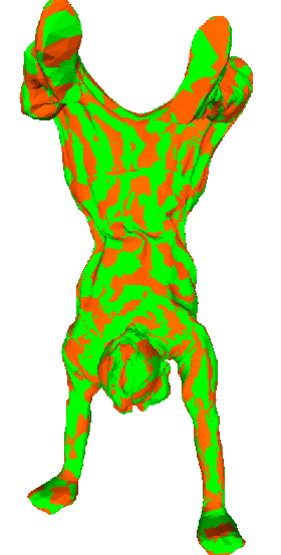} };

\node[] (a) at (4/1.6,11/1.6-0.1) {\includegraphics[width=0.06\linewidth]{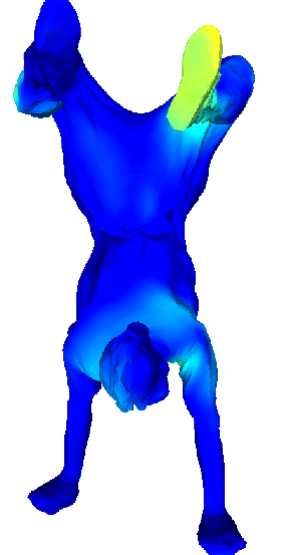} };
\node[] (a) at (7/1.6,11/1.6-0.1) {\includegraphics[width=0.06\linewidth]{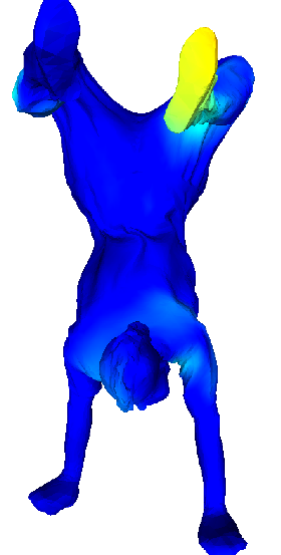} };
\node[] (a) at (10/1.6,11/1.6-0.1) {\includegraphics[width=0.06\linewidth]{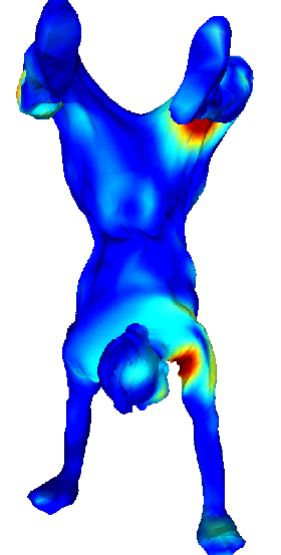} };
\node[] (a) at (13/1.6,11/1.6-0.1) {\includegraphics[width=0.06\linewidth]{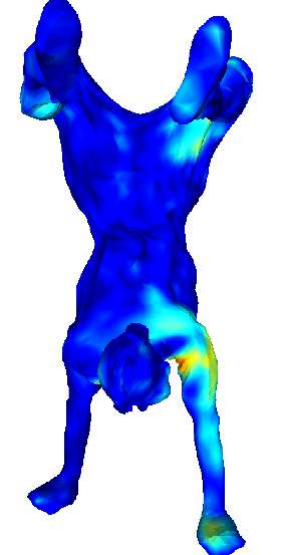} };

\node[] (a) at (0.4/1.6+0.2,-2.3/1.6) {\footnotesize Src \& Tgt};
\node[] (a) at (3/1.6,-2.3/1.6) {\footnotesize CD};
\node[] (a) at (6/1.6,-2.3/1.6) {\footnotesize P2F};
\node[] (a) at (9/1.6,-2.3/1.6) {\footnotesize AMM \cite{AMM}};
\node[] (a) at (12/1.6,-2.3/1.6) {\footnotesize Ours};

\node[] (a) at (12/1.6,14.5/1.6){\includegraphics[width=0.2\linewidth]{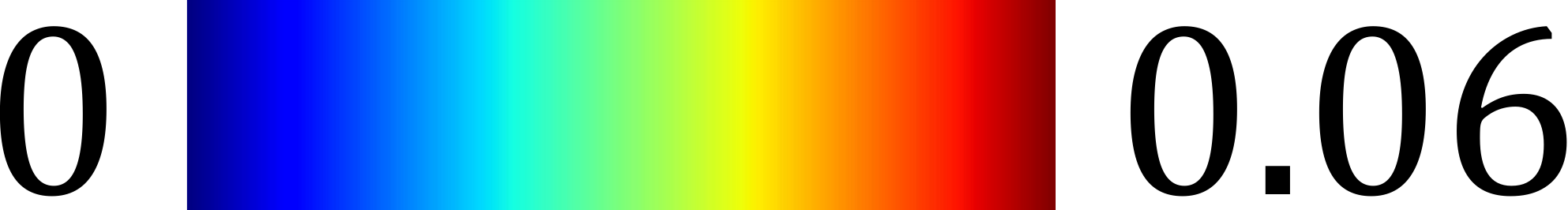}};

\end{tikzpicture}
}
\vspace{-0.5cm}
\caption{\small Visual comparisons of non-rigid registration results.  Source, target and deformed surfaces are rendered in \textcolor{red}{red}, \textcolor{green}{green} and \textcolor{orange}{orange}, respectively. \color{magenta}{\faSearch~} Zoom in to see details.} \label{NON:RIGID:REG:FIG}
\vspace{-0.3cm}
\end{figure}

\begin{figure}
    \centering
    \subfloat[handstand]{\includegraphics[width=0.5\linewidth,height=0.49\linewidth]{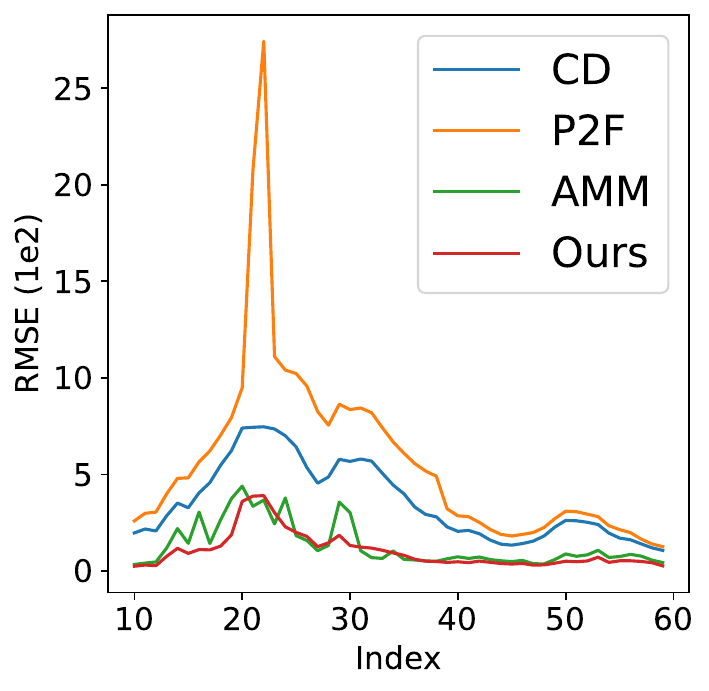}\vspace{-0.3cm}}
    \subfloat[crane]{\includegraphics[width=0.5\linewidth,height=0.5\linewidth]{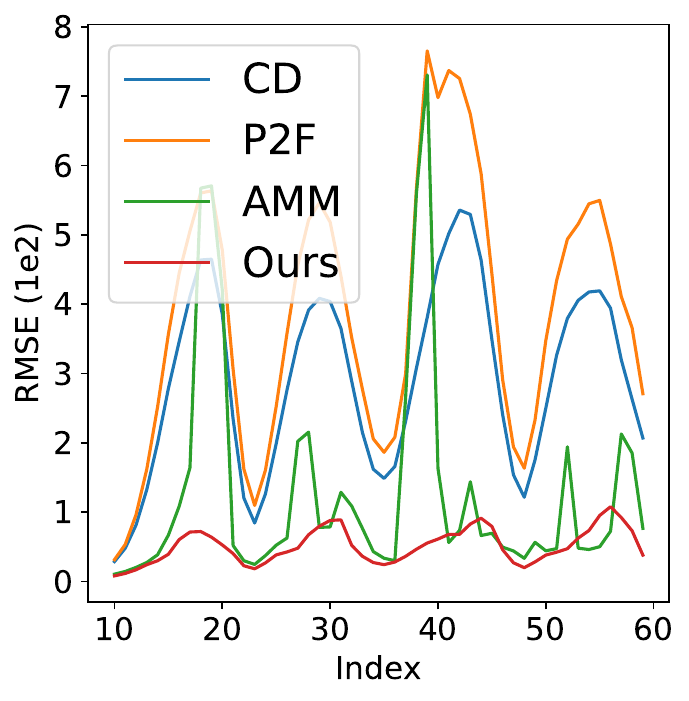}\vspace{-0.3cm}}\\  
    \vspace{0.1cm}
    \subfloat[march]{\includegraphics[width=0.5\linewidth,height=0.5\linewidth]{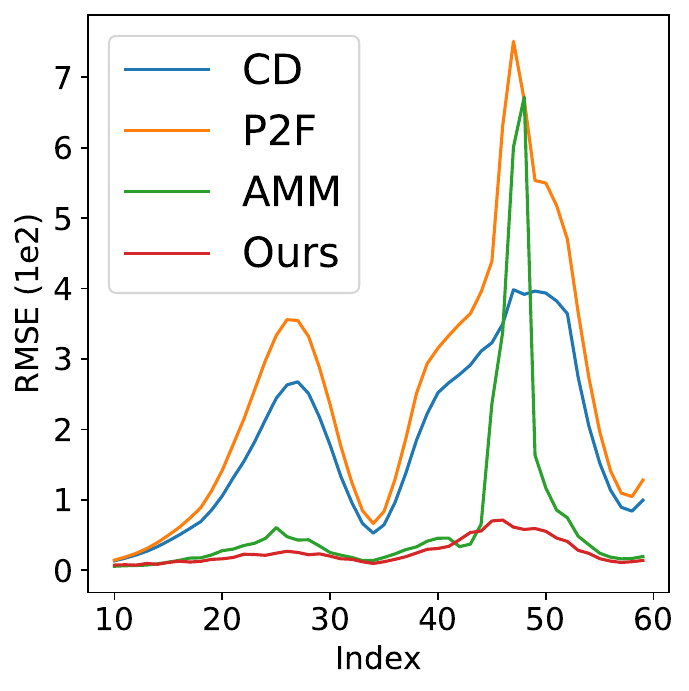}\vspace{-0.3cm}}
    \subfloat[swing]{\includegraphics[width=0.5\linewidth,height=0.5\linewidth]{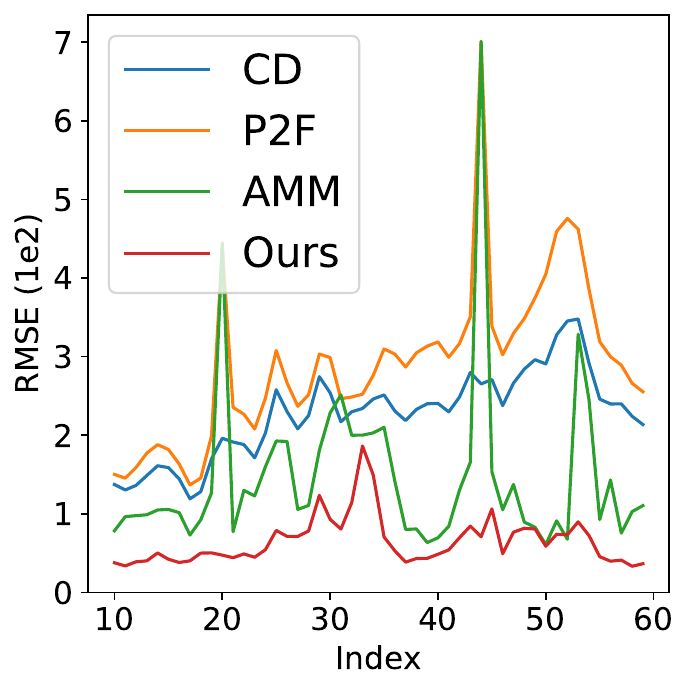}\vspace{-0.3cm}}
    \vspace{-0.3cm}
    \caption{Comparison of the non-rigid registration error of all pairs in each sequence. 
    }
    \label{RMSE:CURVE}
\end{figure}

\begin{table}[h]
    \centering
    \caption{\revise{Running time and GPU memory costs of different distance metrics in the non-rigid registration task.}}
    \renewcommand\arraystretch{1.3}
    \begin{tabularx}{1\linewidth}{>{\hsize=1.5\hsize}X|>{\centering\arraybackslash\hsize=0.6\hsize}X>{\centering\arraybackslash\hsize=0.6\hsize}X>{\centering\arraybackslash\hsize=0.6\hsize}X}
    \toprule
        Method &  CD & P2F &  Ours  \\
    \hline
        Running Time (ms/Iter.) & 42 & 110 & 111\\
        GPU Memory (MB)  & 3245 & 3264 & 3265\\
    \bottomrule
    \end{tabularx}
    \label{TEMPLATE:NONRIGID:TABLE:TIME}
\end{table}

\subsection{Scene Flow Estimation}

\subsubsection{Implementation Details} We used the Flyingthings3D dataset \cite{FLYINGTHINGS3D}, where $N_{\rm src}=N_{\rm tgt}=8192$. The hyperparameters involved in our DDM were set as $K=5$ and $M=81920$. In order to adapt the density at each point's location, the value of $\sigma$ 
was set as 3 times of the distance to its nearest point.  For the optimization-based methods, we optimized the scene flow directly with the Adam optimizer for 500 iterations with a learning rate of 0.01. For unsupervised learning-based methods, we employed two state-of-the-art methods, namely NSFP \cite{NSFP} and SCOOP \cite{SCOOP}, both of which employ CD as the distance metric. We maintained the training settings identical to those specified in their original papers, with the only modification being the replacement of the alignment criterion with our DDM.

\subsubsection{Comparisons} Following \cite{NSFP, SCOOP}, we employed \textit{End Point Error (EPE)}, \textit{Flow Estimation Accuracy (Acc)} with thresholds 0.05 and 0.1 (denoted as \textit{Acc-0.05} and \textit{Acc-0.1}), and \textit{Outliers} as the evaluation metrics. From Table \ref{SCENEFLOW:TABLE} and Fig. \ref{SCENEFLOW:FIG}, it can be seen that our DDM drives much more accurate scene flows than EMD, CD, and BCD under the optimization-based framework, and our DDM further boosts the accuracy of SOTA unsupervised learning-based methods to a significant extent, demonstrating its superiority and the importance of the distance metric in 3D data modeling. \revise{Table \ref{TEMPLATE:SCENEFLOW:TABLE:TIME} illustrates the computational cost of various distance metrics per iteration during the optimization process. Specifically, EMD demands significantly more computing resources due to the necessity of constructing a bi-directional mapping between two point sets. In contrast, CD, BCD, and our proposed DDM exhibit comparable computational requirements. However, our DDM outperforms the others in terms of flow estimation accuracy, highlighting its superior effectiveness and efficiency. }

\begin{table}[h!] 
\centering
\caption{Quantitative comparisons of scene flow estimation on the Flyingthings3D dataset \cite{FLYINGTHINGS3D} \label{SCENEFLOW:TABLE}.}
\renewcommand\arraystretch{1.1}
\begin{tabular}{l|c c c c}
\toprule %
 Method &  EPE3D(m)$\downarrow$ & Acc-0.05 $\uparrow$   & Acc-0.1$\uparrow$ & Outliers $\downarrow$ \\
\hline
{EMD} \cite{EMD} & 0.3681 & 0.1894 & 0.4226 & 0.7838 \\
{CD} \cite{CD} & 0.1557 & 0.3489 & 0.6581 & 0.6799 \\
BCD \cite{BCD} & 0.7045 & 0.0309 & 0.0980 & 0.9965 \\
Ours & \textbf{0.0843} & \textbf{0.6026} & \textbf{0.8749} & \textbf{0.4624} \\                              
\hline

NSFP \cite{NSFP} & 0.0899 & 0.6095 & 0.8496 & 0.4472 \\
Ours (NSFP)& \textbf{0.0657} & \textbf{0.7514} & \textbf{0.9138} & \textbf{0.3234}\\
\hline
SCOOP  \cite{SCOOP}           & 0.0839 & 0.5698 & 0.8516 & 0.4834 \\
Ours (SCOOP)& \textbf{0.0732} & \textbf{0.6307} & \textbf{0.8927} & \textbf{0.4374}\\

\bottomrule %
\end{tabular}
\end{table}

\begin{figure*}[h]
    \subfloat[GT]{\includegraphics[width=0.25\linewidth]{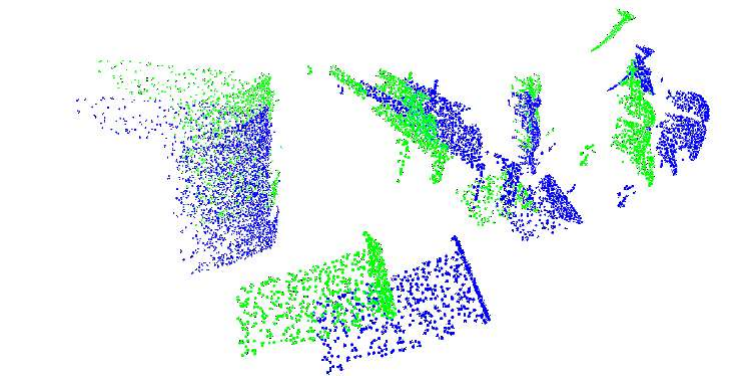}}
    \subfloat[EMD]{\includegraphics[width=0.25\linewidth]{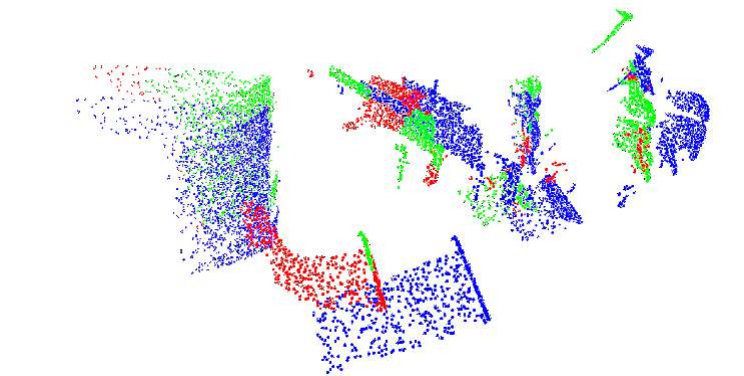}} 
    \subfloat[CD]{\includegraphics[width=0.25\linewidth]{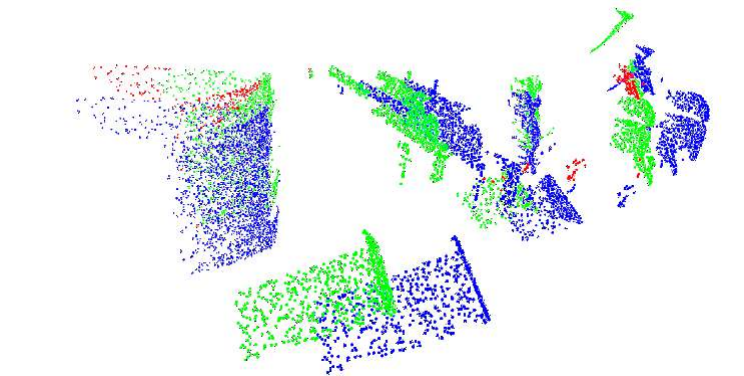}}
    \subfloat[Ours]{\includegraphics[width=0.25\linewidth]{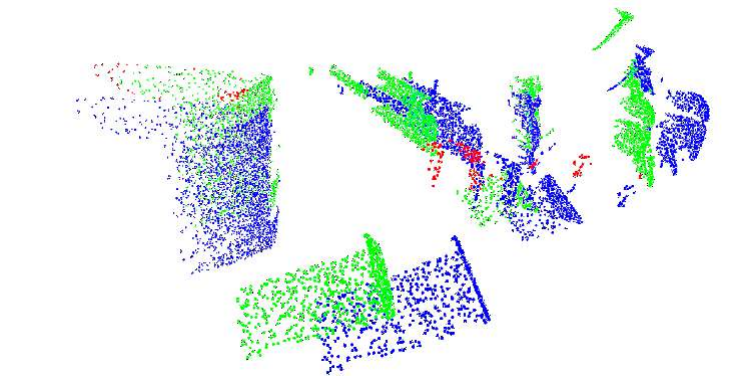}}\\
    \subfloat[NSFP]{\includegraphics[width=0.25\linewidth]{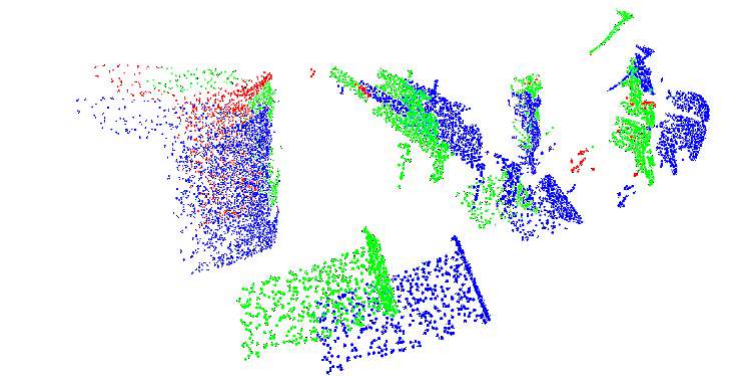}}
    \subfloat[Ours (NSFP)]{\includegraphics[width=0.25\linewidth]{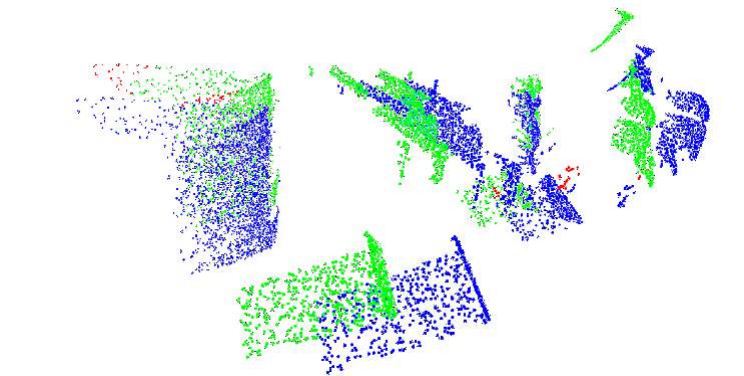}}
    \subfloat[SCOOP]{\includegraphics[width=0.25\linewidth]{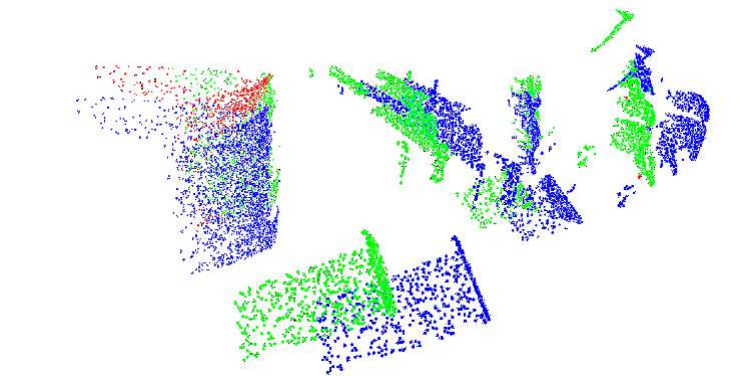}}
    \subfloat[Ours (SCOOP)]{\includegraphics[width=0.25\linewidth]{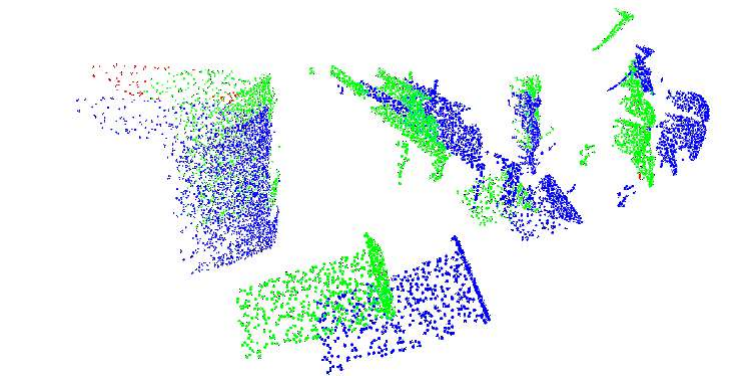}}
    \caption{\small Visual comparisons of scene flow estimation results. 
The source point cloud is represented by \textcolor{blue}{blue} points, while the points translated by the correctly predicted flow are represented by \textcolor{green}{green} points and those translated by the incorrectly predicted flow are represented by \textcolor{red}{red} points.}
\vspace{-0.3cm}
    \label{SCENEFLOW:FIG}
\end{figure*}

\begin{table}[h]
    \centering
    \caption{\revise{Running time and GPU memory costs of different distance metrics under the scene flow estimation task.}}
    \renewcommand\arraystretch{1.0}
    \revise{\begin{tabular}{l|cccc}
    \toprule
        Method & EMD & CD & BCD \cite{BCD} &  Ours  \\
    \hline
        Running Time (ms/Iter.) & 1021 & 13 & 14 & 15 \\
        GPU Memory (MB)  & 2257 & 1745 & 1755 & 1767 \\
    \bottomrule
    \end{tabular}}
    \label{TEMPLATE:SCENEFLOW:TABLE:TIME}
    \vspace{-0.3cm}
\end{table}

\subsection{Parametric Model Estimation from Point Clouds}

\subsubsection{Implementation Details} We used the CAPE dataset \cite{CAPE}, which contains numbers of clothed human models and their corresponding SMPL parameters. We selected two sequences from the whole dataset named `longlong basketball trial2' and `blazerlong volleyball trial2'. In each sequence, we used the pose of the $i$-th frame as the initial parameter and the points sampled from the ($i$+5)-th frame as the simulated scanned point cloud, where the number of the sampled points was $2\times 10^4$. The hyperparameters invovled in our DDM were set as $K=5$, $\beta=1.0$, and $M=6\times 10^4$. 
To bring the experimental setup closer to real-world conditions, where collected point clouds are typically incomplete, we positioned a virtual camera in front of the reference mesh and selectively sampled points exclusively from the visible faces, emulating the data collection process of an RGB-D camera. In such a case, the number of the sampled points was $10^4$ and the hyperparameters about our DDM were kept the same. 
The optimization was conducted through Adam optimizer, running $10^3$ iterations with a learning rate of $5\times 10^{-4}$.  

\subsubsection{Comparisons} We compared our DDM with CD and the P2F distance. 
For the fairness, the number of sampled points of these two baseline methods was set the same as our DDM. We utilized the \textit{Vertex-to-Vertex Error} (\textit{V2V}) between the estimated and ground truth SMPL models for quantitative evaluation. The numerical and visual results are shown in Table \ref{TAB:SMPL:REG} and Fig. \ref{FIG:SMPL:REG}, respectively. Obviously, our DDM is much better than the baseline distance metrics under the scenarios of 
both the whole and partial canned point clouds. \revise{Additionally, Table \ref{SMPL:TABLE:TIME} presents the computational resource requirements of different distance metrics during the optimization process, showing that our DDM achieves comparable running time and GPU memory usage. This further highlights the efficiency of our proposed method in practical applications.}

\begin{table}[h]
\centering
\caption{Quantitative comparison of parametric model estimation on the selected sequences from the CAPE dataset \cite{CAPE}.}
\renewcommand\arraystretch{1.0}
\begin{tabular}{l|l|c c c}
\toprule
\multirow{2}{*}{Scan} & \multirow{2}{*}{Sequence} & \multicolumn{3}{c}{V2V $\downarrow$  $\pm$ STD $\downarrow$  ($\times 10^{-2}$)} \\
\cline{3-5}
    &  & CD & P2F & Ours\\
\hline
\multirow{2}{*}{Whole}  & longlong basketball & 1.46$\pm$0.53 & 1.58$\pm$0.62 & \textbf{0.96$\pm$0.40} \\
                        &blazerlong volleyball & 1.93$\pm$0.92 &2.05$\pm$0.98 &\textbf{1.25$\pm$0.55} \\
\hline
\multirow{2}{*}{Partial}    & longlong basketball & 4.05$\pm$0.63 &2.17$\pm$0.71 & \textbf{1.71$\pm$0.44} \\
                            & blazerlong volleyball & 4.42$\pm$0.76 &3.01$\pm$1.75 & \textbf{1.97$\pm$0.56} \\
\bottomrule
\end{tabular}
\label{TAB:SMPL:REG}
\end{table}

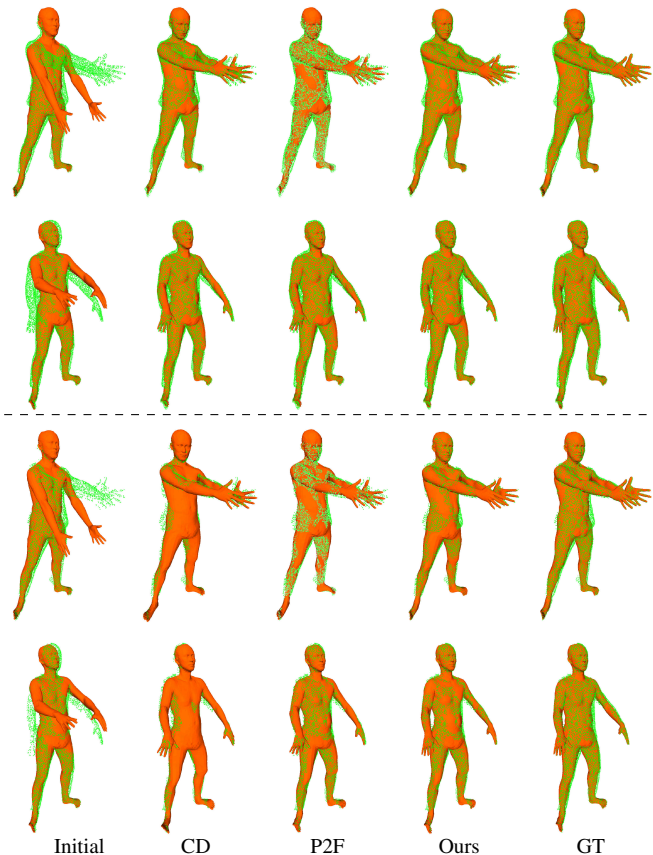
\begin{figure}[h]
\centering
{
\begin{tikzpicture}[]
\node[] (a) at (0,0) {\includegraphics[width=0.2\linewidth]{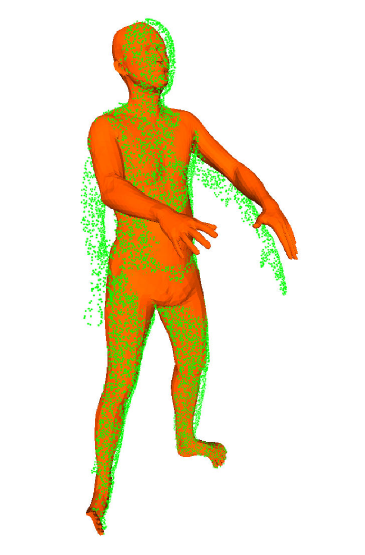} };
\node[] (a) at (2.8/1.6,0) {\includegraphics[width=0.2\linewidth]{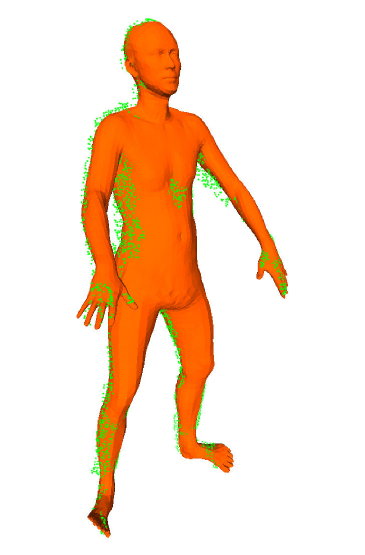} };
\node[] (a) at (5.6/1.6,0) {\includegraphics[width=0.2\linewidth]{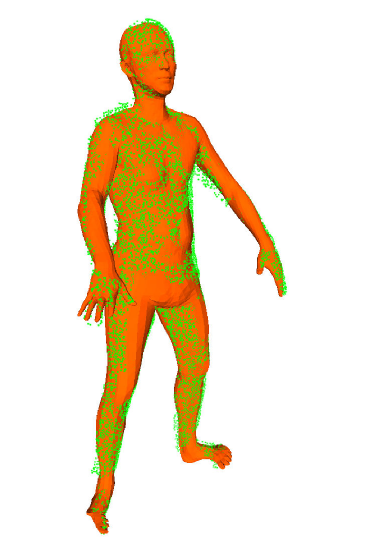} };
\node[] (a) at (8.4/1.6,0) {\includegraphics[width=0.2\linewidth]{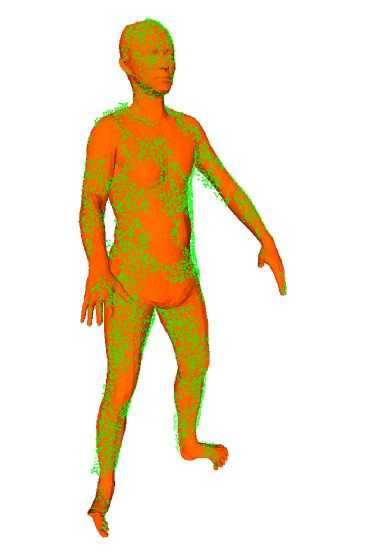} };
\node[] (a) at (11.2/1.6,0) {\includegraphics[width=0.2\linewidth]{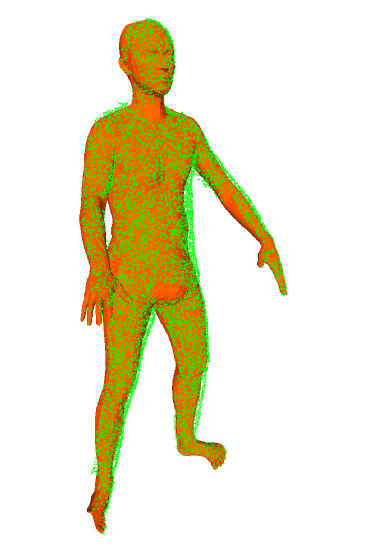} };

\node[] (a) at (0/1.6,4.5/1.6) {\includegraphics[width=0.2\linewidth]{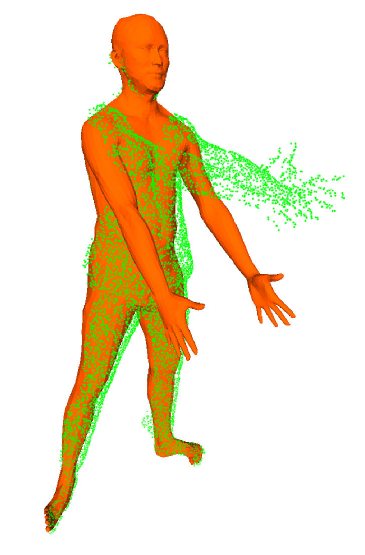} };
\node[] (a) at (2.8/1.6,4.5/1.6) {\includegraphics[width=0.2\linewidth]{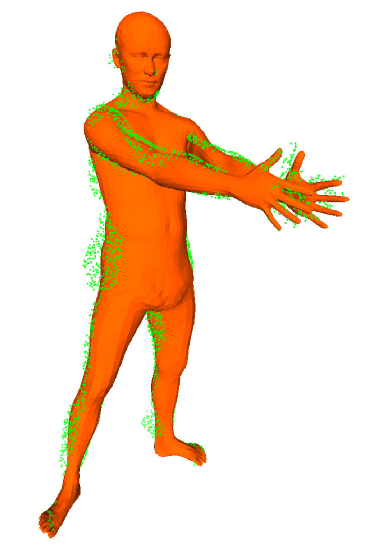} };
\node[] (a) at (5.6/1.6,4.5/1.6) {\includegraphics[width=0.2\linewidth]{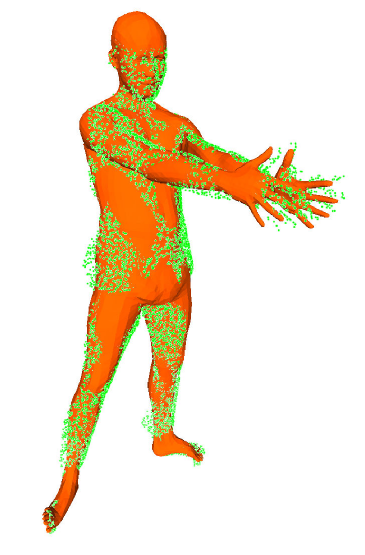} };
\node[] (a) at (8.4/1.6,4.5/1.6) {\includegraphics[width=0.2\linewidth]{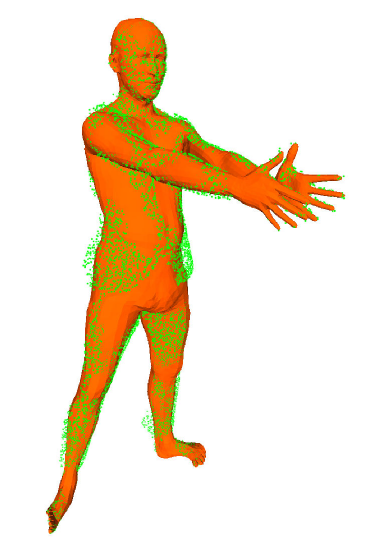} };
\node[] (a) at (11.2/1.6,4.5/1.6) {\includegraphics[width=0.2\linewidth]{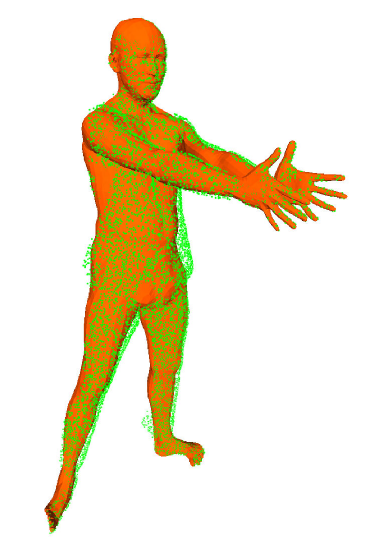} };

\node[] (a) at (0/1.6,9/1.6) {\includegraphics[width=0.2\linewidth]{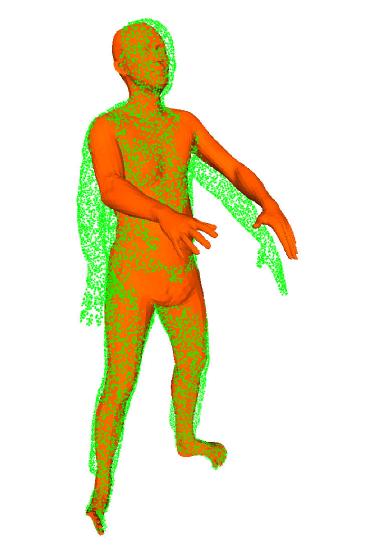} };
\node[] (a) at (2.8/1.6,9/1.6) {\includegraphics[width=0.2\linewidth]{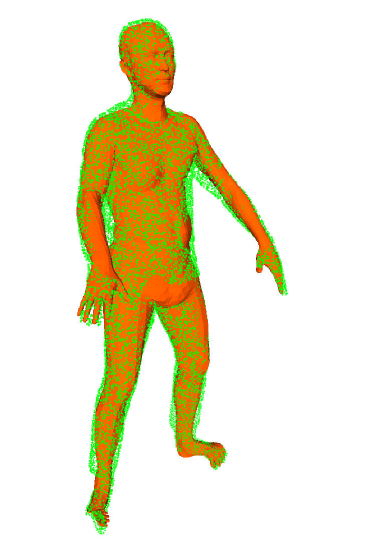} };
\node[] (a) at (5.6/1.6,9/1.6) {\includegraphics[width=0.2\linewidth]{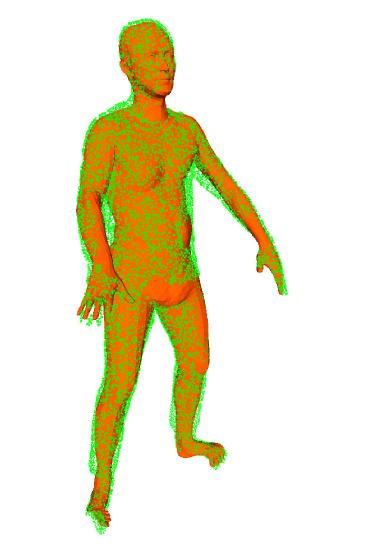} };
\node[] (a) at (8.4/1.6,9/1.6) {\includegraphics[width=0.2\linewidth]{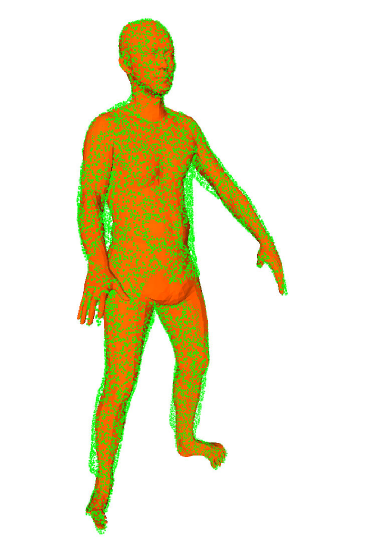} };
\node[] (a) at (11.2/1.6,9/1.6) {\includegraphics[width=0.2\linewidth]{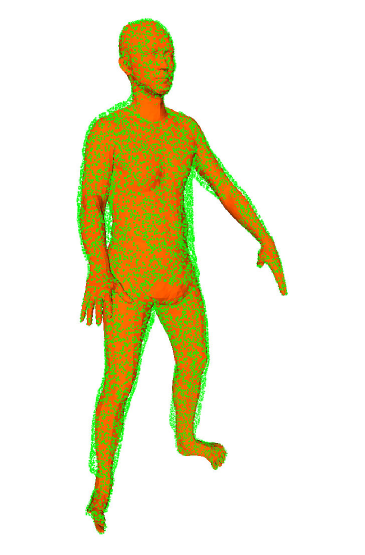} };

\node[] (a) at (0/1.6,13.5/1.6) {\includegraphics[width=0.2\linewidth]{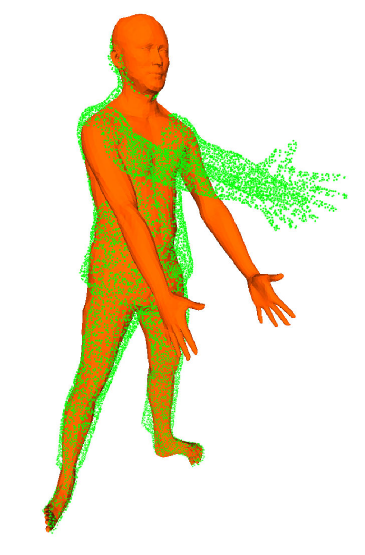} };
\node[] (a) at (2.8/1.6,13.5/1.6) {\includegraphics[width=0.2\linewidth]{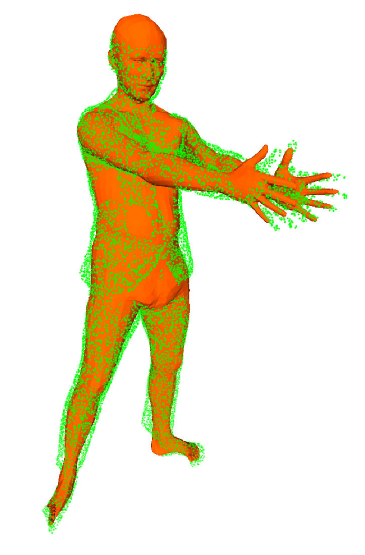} };
\node[] (a) at (5.6/1.6,13.5/1.6) {\includegraphics[width=0.2\linewidth]{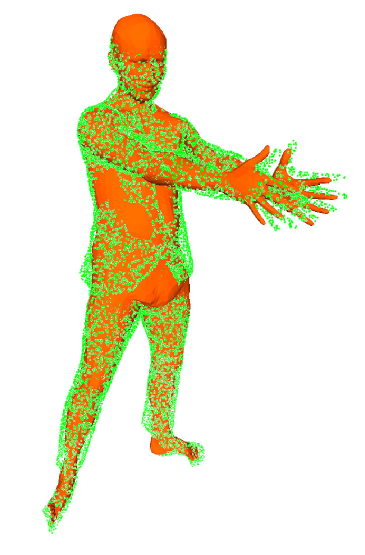} };
\node[] (a) at (8.4/1.6,13.5/1.6) {\includegraphics[width=0.2\linewidth]{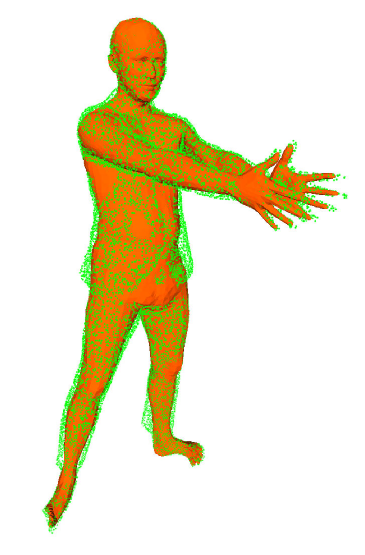} };
\node[] (a) at (11.2/1.6,13.5/1.6) {\includegraphics[width=0.2\linewidth]{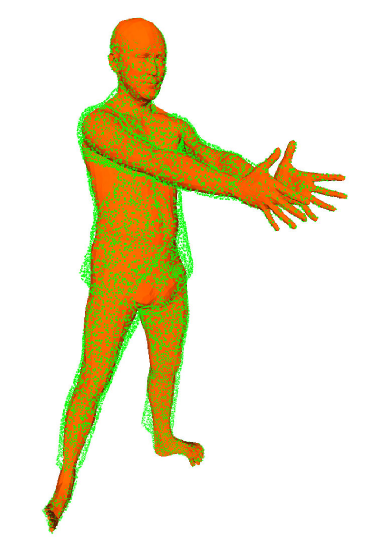} };



\node[] (a) at (0/1.6+0.2,-2.3/1.6) {\footnotesize Initial};
\node[] (a) at (2.8/1.6,-2.3/1.6) {\footnotesize CD};
\node[] (a) at (5.6/1.6,-2.3/1.6) {\footnotesize P2F};
\node[] (a) at (8.4/1.6,-2.3/1.6) {\footnotesize Ours};
\node[] (a) at (11.2/1.6,-2.3/1.6) {\footnotesize GT};

\draw [dashed] (-0.8,4.3) -- (7.8,4.3);

\end{tikzpicture}
}
\vspace{-0.6cm}
\caption{\small Visual comparisons of parametric model estimation results.  
Scanned point clouds and SMPL models are presented in \textcolor{green}{green} and \textcolor{orange}{orange}, respectively. The upper two rows show results by completely scanned point clouds, whereas the bottom two rows by partially scanned point clouds. \color{magenta}{\faSearch~} Zoom in to see details.
} \label{FIG:SMPL:REG}
\vspace{-0.5cm}
\end{figure}

\begin{table}[!h]
    \centering
    \caption{\revise{Running time and GPU memory costs of different distance metrics in the parametric model estimation task.}}
    \renewcommand\arraystretch{1.0}
    \revise{\begin{tabularx}{1\linewidth}{>{\hsize=1.8\hsize}X|>{\centering\arraybackslash\hsize=0.7\hsize}X>{\centering\arraybackslash\hsize=0.7\hsize}X>{\centering\arraybackslash\hsize=0.7\hsize}X}
    \toprule
        Method &  CD & P2F &  Ours  \\
    \hline
        Running Time (ms/Iter.) & 60 & 64 & 65 \\
        GPU Memory (MB)  & 1813 & 1810& 1809\\
    \bottomrule
    \end{tabularx}}
    \label{SMPL:TABLE:TIME}
\end{table}

\subsection{Ablation Study}
We conducted comprehensive ablation studies to help better understand our DDM. According to the different representations of 3D geometric models under evaluation, we divided our ablation studies into three major categories, (1) point cloud to point cloud, (2) triangle mesh to triangle mesh, and (3) triangle mesh to point cloud, with each one conducted on the corresponding task.

\subsubsection{Point Cloud to Point Cloud}
We opted for the rigid registration task to conduct an ablation study on our DDM when both models are represented with point clouds. We assessed registration accuracy under varying hyperparameters, including the number of reference points $M$, $\beta$ in Eq. \eqref{equ:s(q)} that modulates the confidence score, assigning smaller weights to differences in non-overlapping areas, and the standard deviation of Gaussian noise $\sigma$ controlling the distribution of generated reference points. As shown in Table \ref{ABLATION:P2P}, with the increment of the value of $M$, registration accuracy improves slightly; however, an excessive number of reference points escalates time and memory costs, making it impractical.  there is a significant decline in registration accuracy when $\beta=0$, indicating its importance and necessity. Conversely, a large $\beta$ leads to a slight drop in accuracy because all reference points exhibit similar scores, making it challenging to distinguish differences introduced by overlapping and non-overlapping areas.  Finally, the value of $\sigma$ has a subtle impact on the registration accuracy.

\begin{table}
\centering
\caption{Rigid registration accuracy and calculation cost under different settings. The default setting is indicated by \underline{underlining}. 
}  \label{ABLATION:P2P}
\renewcommand\arraystretch{1.0}
    \resizebox{1\linewidth}{!}{
    \begin{tabular}{l|c c c| c c}
        \toprule
        Setting & SR (\%) & RE ($^\circ$) & TE (cm) & Time (ms) & GPU (MB) \\
        \hline
        $M=N_{\rm 2}\times$1 & 83.20 &1.511  & 6.047 & 18.518 & 1721 \\
        $M=N_{\rm 2}\times$5 & 82.71 & 1.477 & 5.980 & 19.876 & 1741 \\
        \underline{$M=N_{\rm 2}\times$10} & 82.77 & 1.485 & 6.098 & 21.872 & 1745 \\ 
        $M=N_{\rm 2}\times$20 & 82.65 & 1.495 & 6.033 & 24.271 & 1779 \\
        \hline
        $\beta=10$ & 82.65 & 1.499 & 6.152 & 22.011 & 1745 \\
        \underline{$\beta=20$} & 82.77 & 1.485 & 6.098 & 21.872 & 1745 \\
        $\beta=30$ & 81.78 &1.586 & 6.264 & 21.882 & 1745\\
        \hline
        $\sigma=0.01$ & 83.45 & 1.461 & 6.024 & 21.643 & 1745 \\
        \underline{$\sigma=0.05$}  & 82.77 & 1.485 & 6.098 & 21.872 & 1745 \\
        $\sigma=0.1$ & 82.21 & 1.488  & 5.968 & 22.023 & 1745 \\
        \bottomrule
    \end{tabular} }
    
\end{table}

\subsubsection{Triangle Mesh to Triangle Mesh}
We utilized the non-rigid registration task to conduct an ablation study on our DDM, where both 3D models are represented with triangle meshes. We varied the values of $M$ and $\sigma$ to examine how these two hyperparameters impact registration accuracy. The results are presented in Table \ref{ABLATION:T2T}. Evidently, employing more reference points leads to higher accuracy, but it also escalates computational costs, particularly in terms of time. Therefore, in practical applications, employing an excessive number of reference points is unnecessary.  When the value of $\sigma$ is extremely small, the generated reference points are distributed very closely to the surface, and thus, fail to effectively capture the differences between the two DDFs, leading to diminished accuracy.

\begin{table}[h]
\centering
\caption{Non-rigid registration accuracy and calculation cost under different setting. The default setting is indicated by \underline{underlining}. 
} \label{ABLATION:T2T}
\renewcommand\arraystretch{1.0}
\begin{tabular}{l|c| c c}
    \toprule
    Setting & RMSE ($\times 10^{-2}$) & Time (ms) & GPU (MB)\\
    \hline
    $M=2\times 10^{4}$  & 0.857 & 99.90& 3257 \\
    \underline{$M=4\times 10^{4}$}  & 0.612 & 110.74 & 3265 \\  
    $M=6\times 10^{4}$  & 0.580 & 151.51 & 3267 \\
    \hline
    $\sigma=0.05$ &  0.646 & 110.65 & 3265 \\
    \underline{$\sigma=0.1$} &   0.612 & 110.74 & 3265 \\  
    $\sigma=0.2$  &  0.618 & 112.64 & 3265 \\
    \bottomrule

\end{tabular}
\end{table}

\subsubsection{Triangle Mesh and Point Cloud}
We opted for the parametric model estimation task to conduct an ablation study of our DDM, employing point cloud and triangle mesh models. We varied the values of $M$, $\beta$, and $\sigma$ to examine their impact on accuracy. The results are presented in Table \ref{ABLATION:T2P}. Evidently, employing a greater number of reference points leads to higher accuracy, but at the cost of increased computation, particularly in terms of time. Therefore, in practical applications, it is unnecessary to use an excessively large number of reference points. Considering the influence of clothing, the SMPL models and the scanned data do not perfectly overlap; consequently, overly large or small values of $\beta$ can diminish accuracy. When the value of $\sigma$ is extremely small, the reference points are  positioned extremely close to the surfaces, and they cannot measure the discrepancy between the two DDFs, resulting in decreased accuracy.

\begin{table}[h]
    \centering
    \caption{Comparison of parametric model estimation accuracy under different settings of DDM. The default setup is indicated by \underline{underlining}. 
    } \label{ABLATION:T2P}
    \renewcommand\arraystretch{1.0}
    \begin{tabular}{l|c | c c}
    \toprule
    Setting & V2V ($\times 10^{-2}$) & Time (ms) & GPU (MB)\\
    \hline
    $M=2\times 10^4$ & 1.855 & 45.248 & 1786 \\ 
    \underline{$M=6\times 10^4$} & 1.074 & 64.766 & 1810 \\  
    $M=12\times 10^4$ & 1.072 & 104.058 & 1846 \\ 
    \hline
    $\beta=0$ &  1.104 & 63.534 & 1810  \\ 
    \underline{$\beta=1$} &   1.074 & 64.766 & 1810  \\  
    $\beta=10$  &  1.362 & 64.137 & 1810  \\ 
    \hline
    $\sigma=0.01$ & 1.126 & 64.667 & 1810 \\
    \underline{$\sigma=0.05$} &  1.074& 64.766 & 1810  \\  
    $\sigma=0.1$  &  1.073& 64.245 & 1810 \\  
    \bottomrule
    \end{tabular}
    \label{tab:my_label}
\end{table}

\section{Conclusion} \label{CONCLUSION}
We have introduced DDM, a robust and versatile metric for efficiently and effectively measuring the discrepancy between 3D geometric models. Unlike existing methods that primarily focus on establishing direct correspondences between two models and subsequently aggregating point-wise distances between corresponding points,  DDM takes a different approach by measuring the discrepancy between the DDFs of the two models, which indirectly establish correspondence between two models while capturing the local surface geometry of 3D models. By integrating DDM into various 3D geometric modeling tasks, such as template surface fitting, rigid and non-rigid registration, scene flow estimation, and human pose optimization, we have demonstrated its substantial superiority in these specific tasks through extensive experiments. We believe that  the introduction of our DDM could significantly advance the progress in the realm of 3D geometric modeling and processing.

\ifCLASSOPTIONcaptionsoff
  \newpage
\fi

{\footnotesize 
\bibliographystyle{IEEEtran}
\bibliography{egbib}}

\begin{thebibliography}{10}
\providecommand{\url}[1]{#1}
\csname url@samestyle\endcsname
\providecommand{\newblock}{\relax}
\providecommand{\bibinfo}[2]{#2}
\providecommand{\BIBentrySTDinterwordspacing}{\spaceskip=0pt\relax}
\providecommand{\BIBentryALTinterwordstretchfactor}{4}
\providecommand{\BIBentryALTinterwordspacing}{\spaceskip=\fontdimen2\font plus
\BIBentryALTinterwordstretchfactor\fontdimen3\font minus \fontdimen4\font\relax}
\providecommand{\BIBforeignlanguage}[2]{{%
\expandafter\ifx\csname l@#1\endcsname\relax
\typeout{** WARNING: IEEEtran.bst: No hyphenation pattern has been}%
\typeout{** loaded for the language `#1'. Using the pattern for}%
\typeout{** the default language instead.}%
\else
\language=\csname l@#1\endcsname
\fi
#2}}
\providecommand{\BIBdecl}{\relax}
\BIBdecl

\bibitem{ICP}
P.~Besl and N.~D. McKay, ``A method for registration of 3-d shapes,'' \emph{IEEE Transactions on Pattern Analysis and Machine Intelligence}, vol.~14, no.~2, pp. 239--256, 1992.

\bibitem{ICPPPP}
P.~J. Besl and N.~D. McKay, ``Method for registration of 3-d shapes,'' in \emph{Sensor fusion IV: control paradigms and data structures}, vol. 1611.\hskip 1em plus 0.5em minus 0.4em\relax Spie, 1992, pp. 586--606.

\bibitem{ARL}
Z.~Deng, Y.~Yao, B.~Deng, and J.~Zhang, ``A robust loss for point cloud registration,'' in \emph{Proceedings of the IEEE/CVF International Conference on Computer Vision}, 2021, pp. 6138--6147.

\bibitem{AMM}
Y.~Yao, B.~Deng, W.~Xu, and J.~Zhang, ``Fast and robust non-rigid registration using accelerated majorization-minimization,'' \emph{IEEE Transactions on Pattern Analysis and Machine Intelligence}, vol.~45, no.~8, pp. 9681--9698, 2023.

\bibitem{FOLDINGNET}
Y.~Yang, C.~Feng, Y.~Shen, and D.~Tian, ``Foldingnet: Point cloud auto-encoder via deep grid deformation,'' in \emph{Proceedings of the IEEE conference on computer vision and pattern recognition}, 2018, pp. 206--215.

\bibitem{AE1}
H.~Wang, Q.~Liu, X.~Yue, J.~Lasenby, and M.~J. Kusner, ``Unsupervised point cloud pre-training via occlusion completion,'' in \emph{Proceedings of the IEEE/CVF international conference on computer vision}, 2021, pp. 9782--9792.

\bibitem{AE2}
Y.~Pang, W.~Wang, F.~E. Tay, W.~Liu, Y.~Tian, and L.~Yuan, ``Masked autoencoders for point cloud self-supervised learning,'' in \emph{European conference on computer vision}.\hskip 1em plus 0.5em minus 0.4em\relax Springer, 2022, pp. 604--621.

\bibitem{AE3}
R.~Zhang, Z.~Guo, P.~Gao, R.~Fang, B.~Zhao, D.~Wang, Y.~Qiao, and H.~Li, ``Point-m2ae: multi-scale masked autoencoders for hierarchical point cloud pre-training,'' \emph{Advances in neural information processing systems}, vol.~35, pp. 27\,061--27\,074, 2022.

\bibitem{TEXT2MESH}
O.~Michel, R.~Bar-On, R.~Liu, S.~Benaim, and R.~Hanocka, ``Text2mesh: Text-driven neural stylization for meshes,'' in \emph{Proceedings of the IEEE/CVF Conference on Computer Vision and Pattern Recognition}, 2022, pp. 13\,492--13\,502.

\bibitem{CLIPMESH}
N.~Mohammad~Khalid, T.~Xie, E.~Belilovsky, and T.~Popa, ``Clip-mesh: Generating textured meshes from text using pretrained image-text models,'' in \emph{SIGGRAPH Asia 2022 conference papers}, 2022, pp. 1--8.

\bibitem{POINTPWC}
W.~Wu, Z.~Y. Wang, Z.~Li, W.~Liu, and L.~Fuxin, ``Pointpwc-net: Cost volume on point clouds for (self-) supervised scene flow estimation,'' in \emph{Computer Vision--ECCV 2020: 16th European Conference, Glasgow, UK, August 23--28, 2020, Proceedings, Part V 16}.\hskip 1em plus 0.5em minus 0.4em\relax Springer, 2020, pp. 88--107.

\bibitem{NSFP}
X.~Li, J.~Kaesemodel~Pontes, and S.~Lucey, ``Neural scene flow prior,'' in \emph{Advances in Neural Information Processing Systems}, vol.~34, 2021, pp. 7838--7851.

\bibitem{SCOOP}
I.~Lang, D.~Aiger, F.~Cole, S.~Avidan, and M.~Rubinstein, ``Scoop: Self-supervised correspondence and optimization-based scene flow,'' in \emph{Proceedings of the IEEE/CVF Conference on Computer Vision and Pattern Recognition}, 2023, pp. 5281--5290.

\bibitem{EMD}
Y.~Rubner, C.~Tomasi, and L.~J. Guibas, ``The earth mover's distance as a metric for image retrieval,'' \emph{International Journal of Computer Vision}, vol.~40, no.~2, p.~99, 2000.

\bibitem{CD}
H.~G. Barrow, J.~M. Tenenbaum, R.~C. Bolles, and H.~C. Wolf, ``Parametric correspondence and chamfer matching: Two new techniques for image matching,'' in \emph{Proceedings: Image Understanding Workshop}.\hskip 1em plus 0.5em minus 0.4em\relax Science Applications, Inc, 1977, pp. 21--27.

\bibitem{METRO}
P.~Cignoni, C.~Rocchini, and R.~Scopigno, ``Metro: measuring error on simplified surfaces,'' in \emph{Computer graphics forum}, vol.~17, no.~2.\hskip 1em plus 0.5em minus 0.4em\relax Wiley Online Library, 1998, pp. 167--174.

\bibitem{BCD}
T.~Wu, L.~Pan, J.~Zhang, T.~Wang, Z.~Liu, and D.~Lin, ``Balanced chamfer distance as a comprehensive metric for point cloud completion,'' \emph{Advances in Neural Information Processing Systems}, vol.~34, pp. 29\,088--29\,100, 2021.

\bibitem{SWD}
T.~Nguyen, Q.-H. Pham, T.~Le, T.~Pham, N.~Ho, and B.-S. Hua, ``Point-set distances for learning representations of 3d point clouds,'' in \emph{Proceedings of the IEEE/CVF International Conference on Computer Vision}, 2021, pp. 10\,478--10\,487.

\bibitem{DPDIST}
D.~Urbach, Y.~Ben-Shabat, and M.~Lindenbaum, ``Dpdist: Comparing point clouds using deep point cloud distance,'' in \emph{Proceedings of the European Conference on Computer Vision}.\hskip 1em plus 0.5em minus 0.4em\relax Springer, 2020, pp. 545--560.

\bibitem{P2F_improved}
H.~Pottmann, S.~Leopoldseder, and M.~Hofer, ``Registration without icp,'' \emph{Computer Vision and Image Understanding}, vol.~95, no.~1, pp. 54--71, 2004.

\bibitem{P2F_improved2}
------, ``Approximation with active b-spline curves and surfaces,'' in \emph{10th Pacific Conference on Computer Graphics and Applications, 2002. Proceedings.}\hskip 1em plus 0.5em minus 0.4em\relax IEEE, 2002, pp. 8--25.

\bibitem{IFNET}
J.~Chibane, T.~Alldieck, and G.~Pons-Moll, ``Implicit functions in feature space for 3d shape reconstruction and completion,'' in \emph{Proceedings of the IEEE/CVF Conference on Computer Vision and Pattern Recognition}, June 2020, pp. 6970--6981.

\bibitem{POINTNET}
C.~R. Qi, H.~Su, K.~Mo, and L.~J. Guibas, ``Pointnet: Deep learning on point sets for 3d classification and segmentation,'' in \emph{Proceedings of the IEEE conference on computer vision and pattern recognition}, 2017, pp. 652--660.

\bibitem{POINTNET2}
C.~R. Qi, L.~Yi, H.~Su, and L.~J. Guibas, ``Pointnet++: Deep hierarchical feature learning on point sets in a metric space,'' \emph{Advances in neural information processing systems}, vol.~30, pp. 1--xxx, 2017.

\bibitem{DIFFUSIONNET}
N.~Sharp, S.~Attaiki, K.~Crane, and M.~Ovsjanikov, ``Diffusionnet: Discretization agnostic learning on surfaces,'' \emph{ACM Transactions on Graphics (TOG)}, vol.~41, no.~3, pp. 1--16, 2022.

\bibitem{PSR}
M.~Kazhdan, M.~Bolitho, and H.~Hoppe, ``Poisson surface reconstruction,'' in \emph{Proceedings of the fourth Eurographics symposium on Geometry processing}, 2006, pp. 61--70.

\bibitem{SPSR}
M.~Kazhdan and H.~Hoppe, ``Screened poisson surface reconstruction,'' \emph{ACM Transactions on Graphics}, vol.~32, no.~3, pp. 1--13, 2013.

\bibitem{OCCNET}
L.~Mescheder, M.~Oechsle, M.~Niemeyer, S.~Nowozin, and A.~Geiger, ``Occupancy networks: Learning 3d reconstruction in function space,'' in \emph{Proceedings of the IEEE/CVF Conference on Computer Vision and Pattern Recognition}, June 2019, pp. 4460--4470.

\bibitem{CONVOCCNET}
S.~Peng, M.~Niemeyer, L.~Mescheder, M.~Pollefeys, and A.~Geiger, ``Convolutional occupancy networks,'' in \emph{European Conference on Computer Vision}.\hskip 1em plus 0.5em minus 0.4em\relax Springer, 2020, pp. 523--540.

\bibitem{SAP}
S.~Peng, C.~Jiang, Y.~Liao, M.~Niemeyer, M.~Pollefeys, and A.~Geiger, ``Shape as points: A differentiable poisson solver,'' \emph{Advances in Neural Information Processing Systems}, vol.~34, pp. 13\,032--13\,044, 2021.

\bibitem{POCO}
A.~Boulch and R.~Marlet, ``Poco: Point convolution for surface reconstruction,'' in \emph{Proceedings of the IEEE/CVF Conference on Computer Vision and Pattern Recognition}, June 2022, pp. 6302--6314.

\bibitem{SDFORIGINAL}
H.~Hoppe, T.~DeRose, T.~Duchamp, J.~McDonald, and W.~Stuetzle, ``Surface reconstruction from unorganized points,'' in \emph{Proceedings of the 19th annual conference on computer graphics and interactive techniques}, 1992, pp. 71--78.

\bibitem{IMLS}
R.~Kolluri, ``Provably good moving least squares,'' \emph{ACM Transactions on Algorithms}, vol.~4, no.~2, pp. 1--25, 2008.

\bibitem{IMLS2}
Z.-Q. Cheng, Y.-Z. Wang, B.~Li, K.~Xu, G.~Dang, and S.-Y. Jin, ``A survey of methods for moving least squares surfaces,'' in \emph{Proceedings of the Fifth Eurographics/IEEE VGTC conference on Point-Based Graphics}, 2008, pp. 9--23.

\bibitem{DEEPIMLS}
S.-L. Liu, H.-X. Guo, H.~Pan, P.-S. Wang, X.~Tong, and Y.~Liu, ``Deep implicit moving least-squares functions for 3d reconstruction,'' in \emph{Proceedings of the IEEE/CVF Conference on Computer Vision and Pattern Recognition}, June 2021, pp. 1788--1797.

\bibitem{DEEPSDF}
J.~J. Park, P.~Florence, J.~Straub, R.~Newcombe, and S.~Lovegrove, ``Deepsdf: Learning continuous signed distance functions for shape representation,'' in \emph{Proceedings of the IEEE/CVF Conference on Computer Vision and Pattern Recognition}, June 2019, pp. 165--174.

\bibitem{NDF}
J.~Chibane, G.~Pons-Moll \emph{et~al.}, ``Neural unsigned distance fields for implicit function learning,'' \emph{Advances in Neural Information Processing Systems}, vol.~33, pp. 21\,638--21\,652, 2020.

\bibitem{MESHUDF}
B.~Guillard, F.~Stella, and P.~Fua, ``Meshudf: Fast and differentiable meshing of unsigned distance field networks,'' in \emph{European Conference on Computer Vision}, 2022, pp. 576--592.

\bibitem{GIFS}
J.~Ye, Y.~Chen, N.~Wang, and X.~Wang, ``Gifs: Neural implicit function for general shape representation,'' in \emph{Proceedings of the IEEE/CVF Conference on Computer Vision and Pattern Recognition}, June 2022, pp. 12\,829--12\,839.

\bibitem{DOG}
P.-S. Wang, Y.~Liu, and X.~Tong, ``Dual octree graph networks for learning adaptive volumetric shape representations,'' \emph{ACM Transactions on Graphics}, vol.~41, no.~4, pp. 1--15, 2022.

\bibitem{GEOUDF}
S.~Ren, J.~Hou, X.~Chen, Y.~He, and W.~Wang, ``Geoudf: Surface reconstruction from 3d point clouds via geometry-guided distance representation,'' in \emph{Proceedings of the IEEE/CVF Internation Conference on Computer Vision}, 2023, pp. 14\,214--14\,224.

\bibitem{MESH2SDF}
J.~A. B{\ae}rentzen and H.~Aanaes, ``Signed distance computation using the angle weighted pseudonormal,'' \emph{IEEE Transactions on Visualization and Computer Graphics}, vol.~11, no.~3, pp. 243--253, 2005.

\bibitem{MESH2SDF2}
G.~Barill, N.~G. Dickson, R.~Schmidt, D.~I. Levin, and A.~Jacobson, ``Fast winding numbers for soups and clouds,'' \emph{ACM Transactions on Graphics (TOG)}, vol.~37, no.~4, pp. 1--12, 2018.

\bibitem{MARCHINGCUBE}
W.~E. Lorensen and H.~E. Cline, ``Marching cubes: A high resolution 3d surface construction algorithm,'' \emph{ACM siggraph computer graphics}, vol.~21, no.~4, pp. 163--169, 1987.

\bibitem{HD}
D.~P. Huttenlocher, G.~A. Klanderman, and W.~J. Rucklidge, ``Comparing images using the hausdorff distance,'' \emph{IEEE Transactions on Pattern Analysis and Machine Intelligence}, vol.~15, no.~9, pp. 850--863, 1993.

\bibitem{P2FORIGINAL}
M.~W. Jones, ``3d distance from a point to a triangle,'' \emph{Department of Computer Science, University of Wales Swansea Technical Report CSR-5}, p.~5, 1995.

\bibitem{POINT2PLANE}
Y.~Chen and G.~Medioni, ``Object modelling by registration of multiple range images,'' \emph{Image and vision computing}, vol.~10, no.~3, pp. 145--155, 1992.

\bibitem{SICP}
S.~Rusinkiewicz, ``A symmetric objective function for icp,'' \emph{ACM Transactions on Graphics (TOG)}, vol.~38, no.~4, pp. 1--7, 2019.

\bibitem{MESHHD}
N.~Aspert, D.~Santa-Cruz, and T.~Ebrahimi, ``Mesh: Measuring errors between surfaces using the hausdorff distance,'' in \emph{Proceedings. IEEE International Conference on Multimedia and Expo}, vol.~1.\hskip 1em plus 0.5em minus 0.4em\relax IEEE, 2002, pp. 705--708.

\bibitem{RMA}
W.~Feng, J.~Zhang, H.~Cai, H.~Xu, J.~Hou, and H.~Bao, ``Recurrent multi-view alignment network for unsupervised surface registration,'' in \emph{Proceedings of the IEEE/CVF Conference on Computer Vision and Pattern Recognition}, 2021, pp. 10\,297--10\,307.

\bibitem{POINT2MESH}
R.~Hanocka, G.~Metzer, R.~Giryes, and D.~Cohen-Or, ``Point2mesh: a self-prior for deformable meshes,'' \emph{ACM Transactions on Graphics (TOG)}, vol.~39, no.~4, pp. 126--1, 2020.

\bibitem{LARGESTEP}
B.~Nicolet, A.~Jacobson, and W.~Jakob, ``Large steps in inverse rendering of geometry,'' \emph{ACM Transactions on Graphics}, vol.~40, no.~6, pp. 1--13, 2021.

\bibitem{MDA}
Y.~Jung, H.~Kim, G.~Hwang, S.-H. Baek, and S.~Lee, ``Mesh density adaptation for template-based shape reconstruction,'' in \emph{ACM SIGGRAPH 2023 Conference Proceedings}, 2023, pp. 1--10.

\bibitem{SMPL_REG}
B.~L. Bhatnagar, C.~Sminchisescu, C.~Theobalt, and G.~Pons-Moll, ``Loopreg: Self-supervised learning of implicit surface correspondences, pose and shape for 3d human mesh registration,'' \emph{Advances in Neural Information Processing Systems}, vol.~33, pp. 12\,909--12\,922, 2020.

\bibitem{SMPL}
M.~Loper, N.~Mahmood, J.~Romero, G.~Pons-Moll, and M.~J. Black, ``Smpl: A skinned multi-person linear model,'' \emph{ACM Trans. Graph.}, vol.~34, no.~6, oct 2015.

\bibitem{CLOSESTPOINT}
M.~Smid, ``Closest-point problems in computational geometry,'' in \emph{Handbook of computational geometry}.\hskip 1em plus 0.5em minus 0.4em\relax Elsevier, 2000, pp. 877--935.

\bibitem{DEFORMATIONNODE}
R.~W. Sumner, J.~Schmid, and M.~Pauly, ``Embedded deformation for shape manipulation,'' in \emph{ACM SIGGRAPH 2007 Papers}, 2007, p. 80–es.

\bibitem{3DCARICSHOP}
Y.~Qiu, X.~Xu, L.~Qiu, Y.~Pan, Y.~Wu, W.~Chen, and X.~Han, ``3dcaricshop: A dataset and a baseline method for single-view 3d caricature face reconstruction,'' in \emph{Proceedings of the IEEE/CVF Conference on Computer Vision and Pattern Recognition}, 2021, pp. 10\,236--10\,245.

\bibitem{3DMATCH}
A.~Zeng, S.~Song, M.~Nie{\ss}ner, M.~Fisher, J.~Xiao, and T.~Funkhouser, ``3dmatch: Learning local geometric descriptors from rgb-d reconstructions,'' in \emph{Proceedings of the IEEE conference on computer vision and pattern recognition}, 2017, pp. 1802--1811.

\bibitem{MAC}
X.~Zhang, J.~Yang, S.~Zhang, and Y.~Zhang, ``3d registration with maximal cliques,'' in \emph{Proceedings of the IEEE/CVF Conference on Computer Vision and Pattern Recognition}, 2023, pp. 17\,745--17\,754.

\bibitem{POINTDSC}
X.~Bai, Z.~Luo, L.~Zhou, H.~Chen, L.~Li, Z.~Hu, H.~Fu, and C.-L. Tai, ``Pointdsc: Robust point cloud registration using deep spatial consistency,'' in \emph{Proceedings of the IEEE/CVF Conference on Computer Vision and Pattern Recognition}, 2021, pp. 15\,859--15\,869.

\bibitem{DGR}
C.~Choy, W.~Dong, and V.~Koltun, ``Deep global registration,'' in \emph{Proceedings of the IEEE/CVF conference on computer vision and pattern recognition}, 2020, pp. 2514--2523.

\bibitem{AMA}
D.~Vlasic, I.~Baran, W.~Matusik, and J.~Popovi{\'c}, ``Articulated mesh animation from multi-view silhouettes,'' \emph{ACM Transactions on Graphics}, vol.~27, no.~3, pp. 1--9, 2008.

\bibitem{FLYINGTHINGS3D}
N.~Mayer, E.~Ilg, P.~Hausser, P.~Fischer, D.~Cremers, A.~Dosovitskiy, and T.~Brox, ``A large dataset to train convolutional networks for disparity, optical flow, and scene flow estimation,'' in \emph{Proceedings of the IEEE Conference on Computer Vision and Pattern Recognition}, 2016, pp. 4040--4048.

\bibitem{CAPE}
Q.~Ma, J.~Yang, A.~Ranjan, S.~Pujades, G.~Pons-Moll, S.~Tang, and M.~J. Black, ``Learning to dress 3d people in generative clothing,'' in \emph{Proceedings of the IEEE/CVF Conference on Computer Vision and Pattern Recognition}, 2020, pp. 6469--6478.

\end{thebibliography}




\end{document}